\documentclass{article}

\usepackage{microtype}
\usepackage{graphicx}
\usepackage{booktabs} 

\usepackage{tikz}
\usetikzlibrary{decorations.markings}
\usetikzlibrary{automata, positioning, arrows}
\tikzset{
	->, 
	every state/.style={thick, fill=gray!10}, 
	initial text=$ $, 
    strike through/.append style={
        decoration={markings, mark=at position 0.5 with {
        \draw[-] ++ (0pt,-5pt) -- (0pt,5pt);}
      },postaction={decorate}}
}
\usetikzlibrary{arrows.meta}
\usepackage{subcaption}
\usetikzlibrary{shapes.misc}

\usepackage{hyperref}

\usepackage[accepted]{icml2024}

\usepackage{amsmath}
\usepackage{amssymb}
\usepackage{mathtools}
\usepackage{amsthm}
\usepackage{dsfont}

\usepackage[capitalize,noabbrev]{cleveref}

\theoremstyle{plain}
\newtheorem{theorem}{Theorem}[section]
\newtheorem{proposition}[theorem]{Proposition}
\newtheorem{lemma}[theorem]{Lemma}

\theoremstyle{definition}
\newtheorem{definition}[theorem]{Definition}

\theoremstyle{remark}
\newtheorem{remark}[theorem]{Remark}

\usepackage[textsize=tiny]{todonotes}

\icmltitlerunning{Agent-Specific Effects: A Causal Effect Propagation Analysis in Multi-Agent MDPs}

\begin{document}

\twocolumn[
\icmltitle{Agent-Specific Effects:\\ A Causal Effect Propagation Analysis in Multi-Agent MDPs}



\icmlsetsymbol{equal}{*}
\icmlsetsymbol{cross}{$^{\dagger}$}

\begin{icmlauthorlist}
\icmlauthor{Stelios Triantafyllou}{mpi}
\icmlauthor{Aleksa Sukovic}{mpi,uds}
\icmlauthor{Debmalya Mandal}{cross,ww}
\icmlauthor{Goran Radanovic}{mpi}
\end{icmlauthorlist}

\icmlaffiliation{mpi}{Max Planck Institute for Software Systems, Saarbrücken, Germany}
\icmlaffiliation{uds}{Saarland University, Saarbrücken, Germany}
\icmlaffiliation{ww}{University of Warwick, Department of Computer Science, UK}

\icmlcorrespondingauthor{Stelios Triantafyllou}{strianta@mpi-sws.org}

\icmlkeywords{Causal Effects, Counterfactual Reasoning, Multi-Agent Markov Decision Processes, Accountability}

\vskip 0.3in
]



\printAffiliationsAndNotice{\icmlPrevAffiliation}  

\begin{abstract}
Establishing causal relationships between actions and outcomes is fundamental for accountable multi-agent decision-making. However, interpreting and quantifying agents' contributions to such relationships pose significant challenges. These challenges are particularly prominent in the context of multi-agent sequential decision-making, where the causal effect of an agent's action on the outcome depends on how other agents respond to that action. In this paper, our objective is to present a systematic approach for attributing the causal effects of agents' actions to the influence they exert on other agents. Focusing on multi-agent Markov decision processes, we introduce agent-specific effects (ASE), a novel causal quantity that measures the effect of an agent's action on the outcome that propagates through other agents. We then turn to the counterfactual counterpart of ASE (cf-ASE), provide a sufficient set of conditions for identifying cf-ASE, and propose a practical sampling-based algorithm for estimating it. Finally, we experimentally evaluate the utility of cf-ASE through a simulation-based testbed, which includes a sepsis management environment.
\end{abstract}



\section{Introduction}\label{sec.intro}

Conducting post-hoc analysis on decisions and identifying their causal contributions to the realized outcome represents a cornerstone of accountable decision making. 
Such analysis can be particularly crucial in safety-critical applications, where accountability is paramount.
Often, these applications involve multiple decision makers or agents, who interact over a period of time, and make decisions or actions that are interdependent. 
This, in turn, implies that the impact of an agent's action on the outcome depends on how other agents respond to that action in subsequent periods.
Thus, the interdependence between the agents' actions poses challenges in quantifying their causal contributions to the outcome, since their effects are interleaved rather than isolated.

To illustrate these challenges, consider a canonical problem setting where accountability is paramount: the sepsis treatment problem~\cite{Komorowski2018}. In a multi-agent variant of this problem, a clinician and an AI agent decide on a series of treatments for an ICU patient over a period of time. For example, the AI can function as the primary decision maker, suggesting treatments that can be overridden by the clinician, who assumes a supervisory role. A simplified, one-shot decision making instance of this problem setting 
is depicted in Fig. \ref{fig: worlds} (left) through a causal graph.
In this setting, the AI agent starts by observing the patient state $S_0$ and suggesting treatment $A_0$. Subsequently, the clinician observes both $S_0$ and $A_0$, and decides whether to override the AI’s treatment or not. 
If the clinician overrides the AI, the patient outcome $Y$ is determined by $S_0$ and the alternative treatment $H_0$ suggested by the clinician. Otherwise, $Y$ is determined by state $S_0$ and the AI's treatment $A_0$.

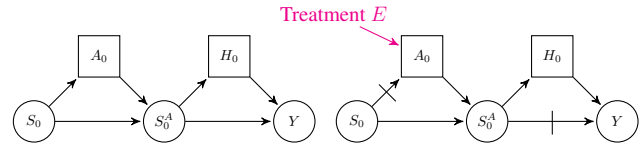
\begin{figure}[H]
\centering
\begin{tikzpicture}[
    >=stealth', 
    shorten >=1pt, 
    auto,
    node distance=0.45cm, 
    scale=.55, 
    transform shape, 
    align=center, 
    state/.style={circle, draw, minimum size=1cm}]


\node[state] (S-0) at (0,0) {$S_0$};

\node[state] (A-0) [rectangle, above right=1cm of S-0] {$A_0$};
\path (S-0) edge (A-0);

\node[state] (SA-0) [below right=1cm of A-0] {$S^A_0$};
\path (S-0) edge (SA-0);
\path (A-0) edge (SA-0);

\node[state] (H-0) [above right=1cm of SA-0, rectangle] {$H_0$};
\path (SA-0) edge (H-0);

\node[state] (Y) [below right=1cm of H-0] {$Y$};
\path (SA-0) edge (Y);
\path (H-0) edge (Y);


\node[state] (S'-0) [right=0.5cm of Y] {$S_0$};

\node[state] (A'-0) [rectangle, above right=1cm of S'-0] {$A_0$};
\path (S'-0) edge [strike through] (A'-0);

\node[] (int') [above left=0.3cm of A'-0] {\Large \textcolor{magenta}{Treatment} $\color{magenta} E$};
\path (int') edge[magenta] (A'-0);

\node[state] (SA'-0) [below right=1cm of A'-0] {$S^A_0$};
\path (S'-0) edge (SA'-0);
\path (A'-0) edge (SA'-0);

\node[state] (H'-0) [above right=1cm of SA'-0, rectangle] {$H_0$};
\path (SA'-0) edge (H'-0);

\node[state] (Y') [below right=1cm of H'-0] {$Y$};
\path (SA'-0) edge [strike through] (Y');
\path (H'-0) edge (Y');

\end{tikzpicture}

\captionsetup{type=figure}
\caption{
The figure illustrates the sepsis treatment problem setting over $1$ 
time-step. 
Squares denote agents' actions, $A$ for AI and $H$ for clinician. 
Circles $S$ are patient states, while $S^A$ include both $S$ and $A$, i.e., $S^A = (S, A)$. 
$Y$ is the outcome. 
Edges that are striked through represent deactivated edges.
Exogenous arrows (magenta) represent interventions on $A_0$ that fix its value to the action indicated.
The graph on the right depicts a path-specific effect of treatment $E$ on $Y$.
}
\label{fig: worlds}
\end{figure}

Given the causal graph, we can utilize standard causal concepts to quantify the impact of the agents' actions on the outcome $Y$.
Arguably, the most straightforward way of doing so would be through the notion of \textit{total counterfactual effects}. Given, for example, a scenario where the AI suggests treatment $C$ and an undesirable outcome is realized, the total counterfactual effect can be used to measure the impact that an alternative treatment $E$ would have on the outcome. 
Evaluating this effect involves intervening on the AI's action, setting it to treatment $E$, and measuring the probability that the undesirable outcome would have been avoided.
However, this analysis may be misleading for assessing the accountability of the AI, as it is unclear to what degree its action affects the outcome through the clinician. 
It is possible that treatment $E$ yields a high total counterfactual effect on the outcome, because it triggers a different response from the clinician, rather than being genuinely better for the patient than treatment $C$.

To quantify the effect of the AI's action that propagates through the clinician, we can utilize the concept of \textit{path-specific effects} \cite{pearl2001direct}.
In particular, we can measure the total counterfactual effect of treatment $E$ on outcome $Y$ in a modified model where the edge between $S_0^A$ and $Y$ in Fig. \ref{fig: worlds} (right) is deactivated.
The latter can be interpreted as if the outcome will be realized based on the factual action of the AI, i.e., treatment $C$, and the action taken by the clinician, who sees the alternative action, i.e., treatment $E$.

A natural generalization to multiple time steps is depicted in Fig. \ref{fig: pse}.
To ensure that action $A_1$ of the AI is the same as the factual one, and hence that the
AI's response in the next time-step does not affect the counterfactual outcome,
we could additionally deactivate the edge between $S_1$ and $A_1$.
While intuitively appealing, this effect does not capture higher-order dependencies between the agents' actions. 
Namely, the clinician's policy depends on the AI's policy, so by fixing the AI's subsequent actions to the ones originally taken, we inadvertently remove such a dependency from the analysis. 

\begin{figure}[h]
\centering
\begin{tikzpicture}[
    >=stealth', 
    shorten >=1pt, 
    auto,
    node distance=0.5cm, 
    scale=.55, 
    transform shape, 
    align=center, 
    state/.style={circle, draw, minimum size=1cm}]


\node[state] (S-0) at (0,0) {$S_0$};

\node[state] (A-0) [rectangle, above right=1cm of S-0] {$A_0$};
\path (S-0) edge [strike through] (A-0);

\node[] (int) [above left=0.3cm of A-0] {\Large \textcolor{magenta}{Treatment} $\color{magenta} E$};
\path (int) edge[magenta] (A-0);

\node[state] (SA-0) [below right=1cm of A-0] {$S^A_0$};
\path (S-0) edge (SA-0);
\path (A-0) edge (SA-0);

\node[state] (H-0) [above right=1cm of SA-0, rectangle] {$H_0$};
\path (SA-0) edge (H-0);

\node[state] (S-1) [below right=1cm of H-0] {$S_1$};
\path (SA-0) edge [strike through] (S-1);
\path (H-0) edge (S-1);

\node[state] (A-1) [thick, cyan, rectangle, above right=1cm of S-1] {$A_1$};
\path (S-1) edge [strike through] (A-1);

\node[state] (SA-1) [below right=1cm of A-1] {$S^A_1$};
\path (S-1) edge (SA-1);
\path (A-1) edge (SA-1);

\node[state] (H-1) [above right=1cm of SA-1, rectangle] {$H_1$};
\path (SA-1) edge (H-1);

\node[state] (Y) [below right=1cm of H-1] {$Y$};
\path (SA-1) edge (Y);
\path (H-1) edge (Y);
\end{tikzpicture}

\captionsetup{type=figure}
\caption{
The causal graph depicts a path-specific effect of treatment $E$ on outcome $Y$ in the sepsis treatment problem setting over 2 time-steps. It aims to capture the counterfactual effect of $E$ that propagates through the clinician.  A cyan colored node signifies that the node is set to the action that the agent took in the factual scenario, i.e., under treatment $C$.}
\label{fig: pse}
\end{figure}
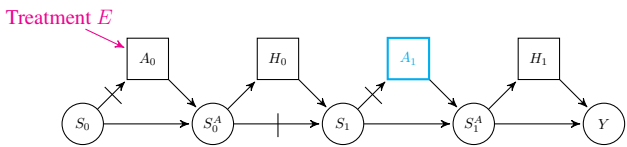

Our objective is to capture these dependencies through a novel causal notion termed \textit{agent-specific effects}.
At a high level, 
we contend that the clinician should behave as if the AI's actions are not fixed, but rather responsive to the clinician's decisions.
However, the effect should be measured using the factual actions of the AI. 
We show that analyzing this effect necessitates counterfactual reasoning across three distinct scenarios, and hence it cannot be expressed through path-specific effects. Our contributions are as follows.

\textbf{Agent-Specific Effects.} (Section \ref{sec.def}) 
Focusing on multi-agent Markov decision processes (MMDPs), we introduce \textit{(counterfactual) agent-specific effects}, a novel causal quantity that measures the effect of an agent's action on the MMDP outcome through a selected subset of agents.
To enable counterfactual reasoning we represent MMDPs by structural causal models (Section \ref{sec.framework}).

\textbf{Identifiability.} (Section \ref{sec.id})  We delve into the identifiability of agent-specific effects, acknowledging their general non-identifiability. To overcome this challenge, we consider an existing causal property, which we refer to as \textit{noise monotonicity}. We show that assuming noise monotonicity suffices for the identification of counterfactual agent-specific effects. 
Notably, compared to closely related identifiability results \cite{lu2020sample, nasr2023counterfactual}, ours does not rely on the assumption of model bijectiveness, and hence is more general.
Lastly, we show that noise monotonicity does not reduce the expressivity of distributions, and that it strictly implies \textit{monotonicity} \cite{pearl1999probabilities} in binary models.

\textbf{Algorithm.} (Section \ref{sec.alg}) We propose a practical sampling-based algorithm for the estimation of counterfactual agent-specific effects. 
We show that this algorithm outputs unbiased estimates provided that noise monotonicity holds.

\textbf{Connections to PSE.} (Section \ref{sec.def} and \ref{sec.connections})
We discuss in detail the importance of agent-specific effects and their conceptual benefits relative to path-specific effects.
We introduce a generalized notion of the latter, which allows us to establish a formal connection between agent- and path-specific effects.

\textbf{Experiments.} (Section \ref{sec.exp})
We conduct extensive experiments on two test-beds, \textit{Graph} and \textit{Sepsis}, evaluating the robustness and practicality of our approach.\footnote{Code to reproduce our experiments is available at \href{https://github.com/stelios30/agent-specific-effects.git}{https://github.com/stelios30/agent-specific-effects.git}.}
Our results indicate that the effect of an agent's action 
on the outcome can be frequently and to a great extent attributed to the behavior of other agents in the system.
This deduction is supported by our analysis of agent-specific effects.
Remarkably, our analysis yields valuable insights, even in scenarios where our theoretical assumptions do not hold, such as revealing aspects of the clinician's trust in the AI.


\subsection{Related Work}\label{sec.related_work}

This paper is related to works on \textit{mediation} or \textit{path analysis} with a single 
\cite{pearl2001direct, zhang2018non, zhang2018fairness}
or multiple
\cite{daniel2015causal, steen2017flexible, avin2005identifiability, singal2021flow, chiappa2019path} mediators.
Unlike this line of work, we do not study the effect of an exposure variable (in our case action) on a response variable (in our case outcome) that propagates through some variable or path in a general causal model. Instead, we focus on MMDPs, and aim to quantify and reason about effects that propagate through a subset of decision-making agents, who control multiple sequentially dependent actions.

Our work also relates to the area of \textit{causal contributions}, which studies the problem of attributing to multiple causes a contribution to a target effect \cite{jung2022measuring}, e.g., to explain the effect of different model features on predictions \cite{heskes2020causal}. Even though we do not develop here such attribution methods, we view this as a natural next step.

Close to our work are also papers that look into the problem of measuring the \textit{influence} between actions in multi-agent decision making systems \cite{jaques2019social, pieroth2022detecting}. The main distinction between this type of works and ours is that we care about the influence of an action on the outcome through the actions of other agents, rather than merely the influence of one action on an other. 

From a technical point of view, this paper is close to works which combine SCMs with sequential decision making frameworks similar to MMDPs, e.g., MDPs \cite{tsirtsis2024finding}, POMDPs \cite{buesing2018woulda} and Dec-POMDPs \cite{triantafyllou2023towards}. 

Furthermore, the multi-agent variant of the sepsis treatment setting we consider in this paper, presented in Sections \ref{sec.intro} and \ref{sec.exp}, is aligned with the extensive literature on \textit{human-in-the-loop} (HITL) decision making. HITL calls for human experts to monitor the AI's decisions during training \cite{saunders2018trial} and/or execution \cite{lynn2019artificial}, in order to ensure safety \cite{xu2022look} and public trust \cite{beil2019ethical}. Within healthcare, clinicians are tasked with overseeing the decision-making process and intervene by overriding the AI's decisions, whenever necessary \cite{liu2022medical}. Support for human oversight has been also recognized by the European Union as one of the main requirements for trustworthy AI \cite{ethicsEU}.

Finally, this paper is related to works that utilize causal tools to study accountability or related concepts in decision making.
We consider works on \textit{responsibility} in multi-agent one-shot
\cite{chockler2004responsibility}
or sequential
\cite{triantafyllou2022actual}
decision making,
as well as works that focus on
\textit{explainability}
\cite{madumal2020explainable},
\textit{incentives}
\cite{everitt2021agent, farquhar2022path},
\textit{harm}
\cite{richens2022counterfactual}, 
\textit{manipulation} 
\cite{carroll2023characterizing} 
and \textit{intention} 
\cite{halpern2018towards}.
Our view is that agent-specific effects can be integrated into all aforementioned works enhancing their causal reasoning, e.g., by modifying current responsibility attribution methods to account for \textit{actual causes} whose effects propagate through other agents, or by defining agent-specific counterfactual harm to measure the harm caused by an action through a set of agents.

More related work can be found in Appendix \ref{app.related_work}.


\section{Framework and Preliminaries}\label{sec.framework}

In this section, we introduce and establish a connection between
the two main parts of our formal framework, 
Structural Causal Models (SCMs) 
\cite{pearl2009causality} 
and Multi-Agent Markov Decision Processes (MMDPs) 
\cite{boutilier1996planning}.

\subsection{Graphs}

We begin by introducing some necessary graph notation.
For a directed graph $G$, we denote with $\mathbf{V}(G)$ and $\mathbf{E}(G)$ the sets of vertices and edges in $G$. With $\mathbf{E}_+^{V^i}(G)$ or $\mathbf{E}_+^i(G)$ for short, we denote the set of outgoing edges of node $V^i$ in $G$.
Furthermore, we use $Pa^i(G)$ to denote the parent set of $V^i$ in $G$.
We call graph $g$ an \textit{edge-subgraph} of $G$ if $\mathbf{V}(g) = \mathbf{V}(G)$ and $\mathbf{E}(g) \subseteq \mathbf{E}(G)$. We also denote with $\overline{g}$ the edge-subgraph of $G$ with $\mathbf{E}(\overline{g}) = \mathbf{E}(G) \backslash \mathbf{E}(g)$.

\subsection{Structural Causal Models}

An SCM $M$ is a tuple $\langle\mathbf{U}, \mathbf{V}, \mathcal{F}, P(\mathbf{u})\rangle$, where:
$\mathbf{U}$ is a set of unobservable noise variables with joint probability distribution $P(\mathbf{u})$;
$\mathbf{V}$ is a set of observable random variables;
$\mathcal{F}$ is a collection of functions. Each random variable $V^i \in \mathbf{V}$ corresponds to one noise variable $U^i \in \mathbf{U}$ and one function $f^i \in \mathcal{F}$.
The value of $V^i$ is determined by the \textit{structural equation} $V^i := f^i(\mathbf{PA}^i, U^i)$, where $\mathbf{PA}^i \subseteq \mathbf{V} \backslash V^i$.
Moreover, SCM $M$ induces a \textit{causal graph} $G$, where $\mathbf{V}$ constitutes the set of nodes $\mathbf{V}(G)$, and for every node $V^i \in \mathbf{V}(G)$ set $\mathbf{PA}^i$ constitutes its parent set $Pa^i(G)$.

\textbf{Interventions.}
An \textit{intervention} on variable $X \in \mathbf{V}$ in SCM $M$ describes the process of modifying the structural equation which determines the value of $X$ in $M$. 
We consider two types of interventions, \textit{hard interventions} and \textit{natural interventions}.
A \textbf{hard intervention} $do(X := x)$ fixes the value of $X$ in $M$ to a constant $x \in \mathrm{dom}\{X\}$.\footnote{$\mathrm{dom}\{X\}$ here denotes the domain of variable $X$.}
The new SCM that is derived from performing $do(X := x)$ in $M$ is denoted by $M^{do(X := x)}$, and is identical to $M$ except that function $f^X$ is replaced by the constant function.

Now, let $Z \in \mathbf{V}$. We denote with $Z_{do(X := x)}$ or $Z_x$ for short, the random variable which, for some noise $\mathbf{u} \sim P(\mathbf{u})$, takes the value that $Z$ would take in $M^{do(X := x)}$.
A \textbf{natural intervention} $do(Z := Z_x)$ replaces the function $f^Z$ of $M$ with $Z_x$.
The derived SCM is denoted by $M^{do(Z := Z_x)}$.

Hard and natural interventions on a set of variables $\mathbf{W} \subseteq \mathbf{V}$, instead of a singleton, are defined in a similar way.

\textbf{Distributions.}
Joint distribution $P(\mathbf{V})_M$ is referred to as the \textit{observational distribution} of $M$.
Furthermore, distributions of the form $P(Y_x)_M$ and of the form $P(Y_x | \mathbf{v})_M$, where $\mathbf{v} \in \mathrm{dom}\{\mathbf{V}\}$ and $\mathbf{v}(X) \neq x$,\footnote{$\mathbf{v}(X)$ here denotes the value of variable $X$ in $\mathbf{v}$.}
are referred to as \textit{interventional} and \textit{counterfactual distributions}, respectively. 
Subscript $M$ here implies that the probability distribution is defined over SCM $M$. 
Note that we may drop this subscript when the underlying SCM is assumed or obvious. 

\textbf{Identifiability.}
Identifying whether or not an interventional or counterfactual distribution over SCM $M$ can be uniquely estimated from the causal graph and observational distribution of $M$ is a fundamental challenge in causality.
Formally, a distribution $P(\cdot)_M$ is \textit{identifiable}
if $P(\cdot)_{M_1} = P(\cdot)_{M_2}$ for any two SCMs $M_1$ and $M_2$, s.t.~(a) the causal graphs induced by $M_1$ and $M_2$ are identical to that of $M$, and (b) the observational distributions of $M_1$ and $M_2$ agree with $M$'s. Otherwise, we say that $P(\cdot)_M$ is \textit{non-identifiable}.

\textbf{Effects.} The \textit{total causal effect} of intervention $do(X := x)$ on $Y = y$, relative to reference value $x^*$, is defined as
\begin{align*}
    \mathrm{TCE}_{x, x^*}(y)_M = P(y_{x})_M - P(y_{x^*})_M,
\end{align*}
where $P(y_{x})_M$ is short for $P(Y_{x} = y)_M$.

Given a realization $\mathbf{v}$ of $\mathbf{V}$, the \textit{total counterfactual effect} of intervention $do(X := x)$ on $Y = y$ is defined as\footnote{
We consider $x^*$ here to be $\mathbf{v}(X)$, but it could also be anything.
}
\begin{align*}
    \mathrm{TCFE}_{x}(y|\mathbf{v})_M = P(y_{x}|\mathbf{v})_M - P(y|\mathbf{v})_M.
\end{align*}

\textbf{Assumptions.} Throughout this paper, we assume SCMs to be \textit{recursive}, i.e., their causal graphs are acyclic. We further make the \textit{exogeneity} assumption: noise variables are mutually independent, $P(\mathbf{u}) = \prod_{u^i \in \mathbf{u}} P(u^i)$. We also assume that the domains of observable variables are finite and discrete, and that noise variables take numerical values.

\subsection{Multi-Agent Markov Decision Processes}\label{sec.framework_mmdp}

An MMDP is a tuple $\langle \mathcal{S}, \{1, ..., n\}, \mathcal{A}, T, h, \sigma \rangle$, where: $\mathcal{S}$ is the state space; $\{1, ..., n\}$ is the set of agents; $\mathcal{A} = \times_{i = 1}^n \mathcal{A}_i$ is the joint action space, with $\mathcal{A}_i$ being the action space of agent $i$; $T : \mathcal{S} \times \mathcal{A} \times \mathcal{S} \rightarrow [0, 1]$ is the transitions probability function; $h$ is the finite time horizon; $\sigma$ is the initial state distribution.\footnote{For ease of notation, rewards are considered part of the state.}
For example, if the agents take joint action $\mathbf{a}$ in state $s$, then the environment transitions to state $s'$ with probability equal to $T(s'|s, \mathbf{a})$.
We denote with $\pi_i$ the policy of agent $i$ and with $\pi$ the agents' joint policy. As such, the probability of agent $i$ taking action $a_i$ in state $s$ is given by $\pi_i(a_i|s)$.
Furthermore, we call trajectory a realization $\tau$ of the MMDP, i.e., $\tau$ consists of a sequence of state-joint action pairs $\{(s_t, \mathbf{a}_t)\}_{t \in \{0, ..., h - 1\}}$ and a final state $s_h$.

Next, we describe how to model an MMDP coupled with a joint policy $\pi$ as an SCM $M$, henceforth called MMDP-SCM.
First, we define the set of observable variables $\mathbf{V}$ to include all state and action variables of the MMDP.
To each such variable, we assign a noise variable $U$ and a function $f$. 
We then formulate the structural equations of $M$ as
\begin{align}\label{eq.struct_eq}
    & S_0 = f^{S_0}(U^{S_0});
    \quad S_t = f^{S_t}(S_{t-1}, \mathbf{A}_{t-1}, U^{S_t});\nonumber\\
    & A_{i,t} = f^{A_{i, t}}(S^t, U^{A_{i,t}}).
\end{align}
The causal graph of $M$ is included in Appendix \ref{app.causal_graph}.\footnote{
The causal graph in Appendix \ref{app.causal_graph} corresponds to the SCM of a general MMDP with $n$ agents and horizon $h$. 
For demonstration purposes, Fig.~\ref{fig: worlds}-\ref{fig: ase} show causal graphs that correspond to SCMs of Turn-Based MMDPs -- a special case of (simultaneous-moves) MMDPs \cite{jia2019feature, sidford2020solving}.
}
Note that \cite{buesing2018woulda} shows that such parameterization of a partially observable MDP to an SCM is always possible. Their result can be directly extended to MMDPs.


\section{Formalizing Agent-Specific Effects}\label{sec.def}

In this section, we introduce the notion of \textit{agent-specific effects} (ASE) and its counterfactual counterpart (cf-ASE).  For a given MMDP-SCM $M$, the purpose of ASE and cf-ASE is to quantify the total causal and counterfactual effect, respectively, of an action $A_{i,t} = a_{i, t}$ on a response variable $Y$, only through a selected set of agents, also called the \textit{effect agents}.
To that end, our definitions consider \textbf{three alternative worlds}. In the first world, action $a_{i, t}$ is fixed. In the second world, a reference $a^*_{i, t}$ (for ASE) or realized $\tau(A_{i, t})$ (for cf-ASE) action is fixed instead. In the third world, the reference or realized action is fixed, and additionally for every time-step $t' > t$ the effect agents are forced to take the actions that they would naturally take in the first world, while all other agents, also called the \textit{non-effect agents}, are forced to take the actions that they would naturally take in the second world.
Note that response variable $Y$ can be any state variable in $M$ that chronologically succeeds $A_{i,t}$.
We formally define ASE and cf-ASE based on the language of natural interventions as follows.\footnote{We note that our main theoretical and experimental results are w.r.t.~cf-ASE. However, we also provide the definition of ASE for completeness and in order to enable a direct comparison with related causal quantities (see Section \ref{sec.connections}).}
\begin{definition}[ASE]\label{def.ase}
Given an MMDP-SCM $M$ and a non-empty subset of agents $\mathbf{N}$ in $M$, the $\mathbf{N}$-specific effect of intervention $do(A_{i,t} := a_{i,t})$ on $Y = y$, relative to reference action $a^*_{i,t}$, is defined as
\begin{align*}
    \mathrm{ASE}^{\mathbf{N}}_{a_{i,t}, a^*_{i,t}}(y)_M = P(y_{a^*_{i,t}})_{M^{do(I)}} - P(y_{a^*_{i,t}})_M,
\end{align*}
where $I = \{Z := Z_{a^*_{i,t}} | Z \in \{A_{i',t'}\}_{i' \notin \mathbf{N}, t' > t}\} \cup \{Z := Z_{a_{i,t}} | Z \in \{A_{i',t'}\}_{i' \in \mathbf{N}, t' > t}\}$.
\end{definition}
\begin{definition}[cf-ASE]\label{def.cf-ase}
Given an MMDP-SCM $M$, a non-empty subset of agents $\mathbf{N}$ in $M$ and a trajectory $\tau$ of $M$, 
the counterfactual $\mathbf{N}$-specific effect of intervention ${do(A_{i,t}:= a_{i,t})}$ on $Y = y$ is defined as
\begin{align*}
    \mathrm{cf\text{-}ASE}^{\mathbf{N}}_{a_{i,t}}(y|\tau)_M = &P(y_{\tau(A_{i,t})}|\tau;M)_{M^{do(I)}} - P(y|\tau)_M,
\end{align*}
where $I = \{A_{i',t'} := \tau(A_{i',t'})\}_{i' \notin \mathbf{N}, t' > t}
\cup \{Z := Z_{a_{i,t}} | Z \in \{A_{i',t'}\}_{i' \in \mathbf{N}, t' > t}\}$. 
\end{definition}

Going back to the sepsis problem from Section \ref{sec.intro}, the clinician here takes the role of the effect agent ($i' \in \mathbf{N}$) and the AI that of the non-effect agent ($i' \notin \mathbf{N}$). The  (counterfactual) clinician-specific effect of intervention $do(A_0 := E)$ on $Y$ is illustrated in Fig. \ref{fig: ase}.
In this figure, we see that the AI takes its factual actions, while the clinician takes the actions that it would naturally take under intervention $do(A_0 := E)$. This should be contrasted with the path-specific effect in Fig.~\ref{fig: pse}, where the clinician behaves as if the AI agent takes its factual action in $A_1$. 
In Appendix \ref{app.example}, we provide the formal expression of Definition \ref{def.cf-ase} to this example.

To demonstrate the importance of this distinction, consider a clinician who does not account for AI's actions when deciding whether or not to override them, but instead bases its decisions solely on the patient's state, e.g., assuming full control whenever the patient is in critical condition. 
Under this behavior, the path-specifics effect in Fig.~\ref{fig: pse} always evaluates to $0$.
Namely, the intervention on $A_0$ does not impact the realization of $S_1$: the clinician does not respond to the AI's action, while the edge between $S_0^A$ and $S_1$ is deactivated. Hence, the realized states and actions in later periods are also the same. 
We deem this to be problematic since we block the influence that the intervention on $A_0$ may have on the clinician's actions in later periods (e.g., $H_1$ in Fig. \ref{fig: pse}) through state $S_1$.
On the other hand, for the agent-specific effect in Fig.~\ref{fig: ase}, the clinician acts according to the states that naturally realize under intervention $do(A_0 := E)$, so this influence is not blocked. 
For example, if the initial state is such that the clinician does not decide to override, the realization of state $S_1$ under intervention $do(A_0 := E)$ could be possibly different than the one in the factual world.  

While it suffices to reactivate the edge between $S_0^A$ and $S_1$ in Fig.~\ref{fig: pse} to obtain a path-specific effect that alleviates the aforementioned issue, this modified version can lead to a different type of counter-intuitive results. Namely, in a scenario where the clinician never overrides the AI's actions, this measure might have a positive evaluation.
Our experiments in Section \ref{sec.exp} demonstrate this tendency. 
In contrast, the agent-specific effect in Fig. \ref{fig: ase} does not have this issue since we intervene on $A_0$ with $do(A_0 := C)$.

Despite these distinctions, we can formalize both ASE and PSE using a generalized notion of path-specific effects, which we introduce in Section \ref{sec.connections}.

\begin{figure}[h]
\centering
\begin{tikzpicture}[
    >=stealth', 
    shorten >=1pt, 
    auto,
    node distance=0.5cm, 
    scale=.55, 
    transform shape, 
    align=center, 
    state/.style={circle, draw, minimum size=1cm}]


\node[state] (S'-0) [right=1cm of Y] {$S_0$};

\node[state] (A'-0) [rectangle, above right=1cm of S'-0] {$A_0$};
\path (S'-0) edge [strike through] (A'-0);

\node[] (int') [above left=0.3cm of A'-0] {\Large \textcolor{cyan}{Treatment} $\color{cyan} C$};
\path (int') edge[cyan] (A'-0);

\node[state] (SA'-0) [below right=1cm of A'-0] {$S^A_0$};
\path (S'-0) edge (SA'-0);
\path (A'-0) edge (SA'-0);

\node[state] (H'-0) [thick, magenta, above right=1cm of SA'-0, rectangle] {$H_0$};
\path (SA'-0) edge [strike through] (H'-0);

\node[state] (S'-1) [below right=1cm of H'-0] {$S_1$};
\path (SA'-0) edge (S'-1);
\path (H'-0) edge (S'-1);

\node[state] (A'-1) [thick, cyan, rectangle, above right=1cm of S'-1] {$A_1$};
\path (S'-1) edge [strike through] (A'-1);

\node[state] (SA'-1) [below right=1cm of A'-1] {$S^A_1$};
\path (S'-1) edge (SA'-1);
\path (A'-1) edge (SA'-1);

\node[state] (H'-1) [thick, magenta, above right=1cm of SA'-1, rectangle] {$H_1$};
\path (SA'-1) edge [strike through] (H'-1);

\node[state] (Y') [below right=1cm of H'-1] {$Y$};
\path (SA'-1) edge (Y');
\path (H'-1) edge (Y');
\end{tikzpicture}

\captionsetup{type=figure}
\caption{
The figure depicts the agent-specific effect of treatment $E$ on outcome $Y$ in the sepsis problem from the introduction. A magenta (resp. cyan) colored node signifies that the node is set to the action that the agent would naturally take under intervention $E$ (resp. the action that the agent took in the factual scenario, i.e., under treatment $C$). }
\label{fig: ase}
\end{figure}
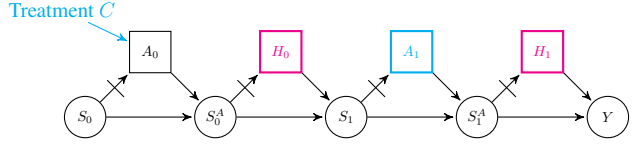


\section{Identifiability Results}\label{sec.id}

Note that $P(y|\tau)_M$ from Definition \ref{def.cf-ase} is $1$ if $y = \tau(Y)$, and $0$ otherwise. Thus, $P(y|\tau)_M$ is identifiable.
In contrast, $P(y_{\tau(A_{i,t})}|\tau;M)_{M^{do(I)}}$ is non-identifiable, since evidence $\mathbf{V} = \tau$ is contradictory to intervention $do(I)$ \cite{shpitser2008complete}. Subsequently, $\mathrm{cf\text{-}ASE}^{\mathbf{N}}_{a_{i,t}}(y|\tau)_M$ is also non-identifiable under our current set of assumptions.

Theorem \ref{thrm.ase_nonid} states that $\mathrm{ASE}^{\mathbf{N}}_{a_{i,t}, a^*_{i,t}}(y)_M$, which is defined over interventional probabilities instead of counterfactual ones, is also non-identifiable.\footnote{Note that the same non-identifiability results would also hold even if we had access to interventional data.} The proof is based on the \textit{recanting witness criterion} \cite{avin2005identifiability}.
\begin{theorem}\label{thrm.ase_nonid}
    Let $M$ be an MMDP-SCM, and $\mathbf{N}$ be a non-empty subset of agents in $M$. It holds that the $\mathbf{N}$-specific effect $\mathrm{ASE}^{\mathbf{N}}_{a_{i,t}, a^*_{i,t}}(y)_M$ is non-identifiable if $t_Y > t + 1$, where $t_Y$ denotes the time-step of variable $Y$.   
\end{theorem}
To overcome the aforementioned non-identifiability challenges, we consider the following property, which was first introduced in \cite{pearl2009causality} [p.~295] for binary SCMs.
\begin{definition}[Noise Monotonicity]
\label{def.nm}
Given an SCM $M$ with causal graph $G$, we say that variable $V^i  \in \mathbf{V}$ is noise-monotonic in $M$ w.r.t.~a total ordering $\leq_i$ on $\mathrm{dom}\{V^i\}$, if for any $pa^i \in \mathrm{dom}\{Pa^i(G)\}$ and $u^i_1, u^i_2 \sim P(U^i)$ s.t. $u^i_1 < u^i_2$, it holds that
$
f^i(pa^i, u^i_1) \leq_i f^i(pa^i, u^i_2)
$.
\end{definition}
Assuming \textit{noise monotonicity} means that functions $f^i$ of the underlying SCM, or MMDP-SCM in our case, are monotonic w.r.t.~noise variables $U^i$. The next theorem states that noise monotonicity suffices to render cf-ASE identifiable.\footnote{Even though we do not explicitly prove it in this paper, it follows that the same result also holds for ASE.}
\begin{theorem}\label{thrm.id_ase_cf}
Let $M$ be an MMDP-SCM, $\mathbf{N}$ be a non-empty subset of agents in $M$, and $\tau$ be a trajectory of $M$. The counterfactual $\mathbf{N}$-specific effect $\mathrm{cf\text{-}ASE}^{\mathbf{N}}_{a_{i,t}}(y|\tau)_M$ is identifiable, if for every variable $Z \in \mathbf{V}$ there is a total ordering $\leq_Z$ on $\mathrm{dom}\{Z\}$ w.r.t.~which $Z$ is noise-monotonic in $M$.
\end{theorem}
Theorem \ref{thrm.id_ase_cf} is closely related to Theorem $1$ from \cite{lu2020sample} and Theorem $5.1$ from \cite{nasr2023counterfactual}, where they show that under \textbf{strong} noise monotonicity the counterfactual effect of an action on its subsequent state is identifiable in the class of bijective causal models.
Note that in contrast to their work, here we assume \textbf{weak} noise monotonicity, which does not imply model bijectiveness, i.e., functions $f^i(pa^i, \cdot)$ are not necessarily invertible. Hence, our result is more general.

Our proof leverages a novel identification formula, Eq.~\eqref{eq.ctf_formula}, that allows us to express \textit{counterfactual factors} \cite{correa2021nested} in terms of the observational distribution.
As such, Lemma \ref{lem.pns_nm} can be utilized for deriving the identification formula of any counterfactual probability (under exogeneity and noise monotonicity), 
and thus is of independent interest. Appendix \ref{app.ctf_remark} provides a more thorough explanation.
\begin{lemma}\label{lem.pns_nm}
Let $M$ be an SCM with causal graph $G$, and $\leq_X$ be a total ordering on $\mathrm{dom}\{X\}$ for some $X \in \mathbf{V}$.
If $X$ is noise-monotonic in $M$ w.r.t.~$\leq_X$, then for any $x^1, ..., x^k \in \mbox{dom}\{X\}$ and $pa^1, ..., pa^k \in \mbox{dom}\{Pa^X(G)\}$ it holds that
\begin{align}\label{eq.ctf_formula}
    &P(\land_{i \in \{1, ..., k\}} x^i_{pa^i}) = 
    \max\big(0, \min_{i \in \{1, ..., k\}} P(X \leq_X x^i | pa^i)\nonumber \\&- \max_{i \in \{1, ..., k\}}P(X <_X x^i | pa^i) \big).
\end{align}
\end{lemma}
Next, we present novel findings to enhance the characterization of noise monotonicity. 
\cite{buesing2018woulda} show how we can represent any partially observable MDP by an SCM. In that direction, we show next how to represent any MMDP by an MMDP-SCM that satisfies noise monotonicity.

\begin{lemma}\label{lemma.nm_mmdp}
    Let $\mathbf{X} = \{X^1, ..., X^N\}$ represent the set of state and action variables of an MMDP, and $P(\mathbf{X})$ be the joint distribution induced by the MMDP under a joint policy $\mathbf{\pi}$. Additionally, let $\leq_1, ..., \leq_N$ be total orderings on $\mathrm{dom}\{X^1\}, ..., \mathrm{dom}\{X^N\}$, respectively. We can construct an MMDP-SCM $M = \langle\mathbf{U}, \mathbf{V}, \mathcal{F}, P(\mathbf{u})\rangle$ with causal graph $G$, where $\mathbf{V} = \{V^1, ..., V^N\}$, $\mathbf{U} = \{U^1, ..., U^N\}$, $P(U^i) = \mathrm{Uniform}[0,1]$, and functions in $\mathcal{F}$ are defined as
\begin{align*}
    f^i(Pa^i(G), U^i) = \inf_{v^i}\{P(X^i \leq_i v^i | Pa^i(G))_M \geq U^i\},
\end{align*}
such that $P(\mathbf{X}) = P(\mathbf{V})$ and every variable $V^i \in \mathbf{V}$ is noise-monotonic in $M$ w.r.t.~$\leq_i$.
\end{lemma}

\cite{oberst2019counterfactual} recently proposed another property for categorical SCMs, \textit{counterfactual stability}. 
Interestingly, they show that their property implies the well-known property of
\textit{monotonicity} \cite{pearl1999probabilities} in binary SCMs. Note, that counterfactual stability does not suffice for identifiability. Proposition \ref{prop.nm_m} states that noise monotonicity also implies monotonicity in binary SCMs.
\begin{proposition}\label{prop.nm_m}
Let $M$ be an SCM consisting of two binary variables $X$ and $Y$ with $Y = f^Y(X, U^Y)$. If $Y$ is noise-monotonic in $M$ w.r.t.~the numerical ordering, then $Y$ is monotonic relative to $X$ in $M$. However, if $Y$ is monotonic relative to $X$ in $M$, then $Y$ is not necessarily noise-monotonic in $M$ w.r.t.~any total ordering on $\mathrm{dom}\{X\}$.
\end{proposition}


\section{Algorithm}\label{sec.alg}

We now present our algorithm for the estimation of cf-ASE, which follows a standard methodology for counterfactual inference \cite{pearl2009causality}. 
Algorithm \ref{alg.cf-ase} samples noise variables from the posterior $P(\mathbf{u} | \tau)$ (abduction), 
in order to compute cf-ASE (prediction) under interventions (action) specified in Definition \ref{def.cf-ase}, and finally outputs the average value.

\begin{algorithm}
\caption{Estimates $\mathrm{cf\text{-}ASE}^{\mathbf{N}}_{a_{i,t}}(y|\tau)_M$}
\label{alg.cf-ase}
\textbf{Input}: MMDP-SCM $M$, trajectory $\tau$, effect agents $\mathbf{N}$, action variable $A_{i,t}$, action $a_{i,t}$, outcome variable $Y$, outcome $y$, number of posterior samples $H$
\begin{algorithmic}[1] 
\STATE $h \gets 0$
\STATE $c \gets 0$
\WHILE{$h < H$}
\STATE $\mathbf{u}_h \sim P(\mathbf{u} | \tau)$ \COMMENT{Sample noise from the posterior}
\STATE $h \gets h + 1$
\STATE $\tau^h \sim P(\mathbf{V}|\mathbf{u}_h)_{M^{do(A_{i,t} := a_{i,t})}}$ \COMMENT{Sample cf trajectory}
\STATE $I(\mathbf{u}_h) \gets \{A_{i',t'} := \tau(A_{i',t'})\}_{i' \notin \mathbf{N}, t' > t} \cup \{A_{i',t'} := \tau^h(A_{i',t'})\}_{i' \in \mathbf{N}, t' > t}$\;
\STATE $y^h \sim P(Y|\mathbf{u}_h)_{M^{do(I(\mathbf{u}_h))}}$ \COMMENT{Sample cf outcome}
\IF {$y^h = y$}
\STATE $c \gets c + 1$
\ENDIF
\ENDWHILE
\STATE \textbf{return} $\frac{c}{H} - \mathds{1}(\tau(Y) = y)$
\end{algorithmic}
\end{algorithm}
\begin{theorem}\label{thrm.algo_unbias}
Let MMDP-SCMs $M = \langle \mathbf{U}, \mathbf{V}, \mathcal{F}, P(\mathbf{u})\rangle$ and $\hat{M} = \langle \mathbf{U}, \mathbf{V}, \hat{\mathcal{F}}, \hat{P}(\mathbf{u})\rangle$, s.t.~$P(\mathbf{V})_M = P(\mathbf{V})_{\hat{M}}$ and every variable $Z \in \mathbf{V}$ is noise-monotonic in both $M$ and $\hat{M}$ w.r.t.~the same total ordering $\leq_Z$ on $\mathrm{dom}\{Z\}$. It holds that the output of Algorithm \ref{alg.cf-ase}, when it takes as input model $\hat{M}$, is an unbiased estimator of $\mathrm{cf\text{-}ASE}^{\mathbf{N}}_{a_{i,t}}(y|\tau)_M$.
\end{theorem}
If noise-monotonicity holds, then according to Theorem \ref{thrm.algo_unbias}, access to a model $\hat{M}$ suffices to get an unbiased estimate of $\mathrm{cf\text{-}ASE}^{\mathbf{N}}_{a_{i,t}}(y|\tau)_M$. 
Based on Lemma \ref{lemma.nm_mmdp}, such an MMDP-SCM always exists, and it can be derived solely from the observational distribution $P(\mathbf{V})_M$ and the total orderings w.r.t.~which variables in $M$ are noise-monotonic.
%


\section{Connection to Path-Specific Effects}\label{sec.connections}

\textit{Path-specific effects} are a well-known causal concept used to quantify the causal effect of one variable on another along a specified set of paths in the causal graph \cite{pearl2001direct}.

\begin{definition}[PSE]
\label{def.pse}
Given an SCM $M$ with causal graph $G$, and an edge-subgraph $g$ of $G$ (also called the \textit{effect subgraph}), the $g$-specific effect of intervention $do(X := x)$ on $Y = y$, relative to reference value $x^*$, is defined as
\begin{align*}
    \mathrm{PSE}^g_{x, x^*}(y)_M = \mathrm{TCE}_{x, x^*}(y)_{M_g},
\end{align*}
where $M_{g} = \langle\mathbf{U}, \mathbf{V}, \mathcal{F}_g, P(\mathbf{u})\rangle$ is a modified SCM, which induces $g$ and is formed as follows.
Each function $f^i$ in $\mathcal{F}$ is replaced with a new function $f^i_g$ in $\mathcal{F}_g$, s.t.~for every $pa^i(g) \in \mathrm{dom}\{Pa^i(g)\}$ and $\mathbf{u} \sim P(\mathbf{u})$ it holds $f^i_g(pa^i(g), u^i) = f^i(pa^i(g), pa^i(\overline{g})^*, u^i)$.
Here, $pa^i(\overline{g})^*$ denotes the value of $Pa^i(\overline{g})_{x^*}$ in $M$ for noise $\mathbf{u}$.
\end{definition}
As argued in Section \ref{sec.intro}, despite also considering natural interventions, path-specific effects cannot be used to express agent-specific effects. Intuitively, this is because analyzing the latter effects requires reasoning across three alternative worlds, while the definition of the former reasons across only two.
To establish a deeper connection between the two concepts, we define the notion of \textit{fixed path-specific effects} (FPSE) utilizing a similar path-deactivation process as in Definition~\ref{def.pse}.\footnote{We emphasize here that PSE and FPSE are defined over general SCMs, while ASE is defined explicitly over MMDP-SCMs.}
Appendix \ref{app.fpse} includes identifiability results for FPSE and its counterfactual counterpart.
\begin{definition}[FPSE]
\label{def.fpse}
Given an SCM $M$ with causal graph $G$, and two edge-subgraphs of $G$, \textit{effect subgraph} $g$ and \textit{reference subgraph} $g^*$ such that $\mathbf{E}(g) \cap \mathbf{E}(g^*) = \emptyset$, the fixed $g$-specific effect of intervention $do(X := x)$ on $Y = y$, relative to reference value $x^*$ and $g^*$, is defined as
\begin{align*}
    \mathrm{FPSE}^{g, g^*}_{x, x^*}(y)_M = P(y_{x^*})_{M_q} - P(y_{x^*})_M,
\end{align*}
where $q$ is the edge-subgraph of $G$ with 
$\mathbf{E}(q) = \mathbf{E}(G) \backslash \{\mathbf{E}(g), \mathbf{E}(g^*)\}$, and $M_q = \langle\mathbf{U}, \mathbf{V}, \mathcal{F}_q, P(\mathbf{u})\rangle$ is a modified SCM, which induces $q$ and is formed as follows.
Each function $f^i$ in $\mathcal{F}$ is replaced with a new function $f^i_q$ in $\mathcal{F}_q$, s.t.~for every $pa^i(q) \in \mathrm{dom}\{Pa^i(q)\}$ and $\mathbf{u} \sim P(\mathbf{u})$ it holds $f^i_q(pa^i(q), u^i) = f^i(pa^i(q), pa^i(g)^e, pa^i(g^*)^*, u^i)$. Here, $pa^i(g^*)^*$ (resp. $pa^i(g)^e$) denotes the value of $Pa^i(g^*)_{x^*}$ (resp. $Pa^i(g)_x$) in $M$ for noise $\mathbf{u}$.
\end{definition}

Definitions of PSE and FPSE have two main distinctions: FPSE considers one additional edge-subgraph compared to PSE, and it computes the counterfactual probability over its modified model under intervention $do(X := x^*)$ instead of $do(X := x)$.
This additional subgraph allows FPSE to reason across three alternative words, i.e., one more than PSE, which is why it can express agent-specific effects.

\begin{lemma}\label{lemma.fpse_ase}
Let $M$ be an MMDP-SCM with causal graph $G$ and $\mathbf{N}$ be a non-empty subset of agents in $M$. It holds that
\begin{align*}
    &\mathrm{ASE}^{\mathbf{N}}_{a_{i,t}, a^*_{i,t}}(y)_M = \mathrm{FPSE}^{g, g^*}_{a_{i,t}, a^*_{i,t}}(y)_M 
\end{align*}
where $g$ is the edge-subgraph of $G$ with $\mathbf{E}(g) = \bigcup_{A_{i',t'}: i' \in \mathbf{N}, t' > t} \mathbf{E}^{A_{i',t'}}_+(G)$ and 
$g^*$ is the edge-subgraph of $G$ with
$\mathbf{E}(g^*) = \bigcup_{A_{i',t'}: i' \notin \mathbf{N}, t' > t} \mathbf{E}^{A_{i',t'}}_+(G)$.
\end{lemma}
Furthermore, by ``not using'' the extra edge-subgraph, FPSE can express path-specific effects.
\begin{proposition}\label{prop.fpse_pse}
Given an SCM $M$ with causal graph $G$, and an edge-subgraph $g$ of $G$, it holds that
\begin{align*}
    \mathrm{PSE}^g_{x, x^*}(y)_M = \mathrm{FPSE}^{\overline{g}, g^\emptyset}_{x^*, x}(y)_M + \mathrm{TCE}_{x, x^*}(y)_M,
\end{align*}
where $g^\emptyset$ is the edge-subgraph of $G$ with $\mathbf{E}(g^\emptyset) = \emptyset$.
\end{proposition}
%


\section{Experiments}\label{sec.exp}
\begin{figure*}
    \centering
    \begin{subfigure}[c]{0.8\textwidth}
        \centering
        \includegraphics[width=\textwidth,keepaspectratio]{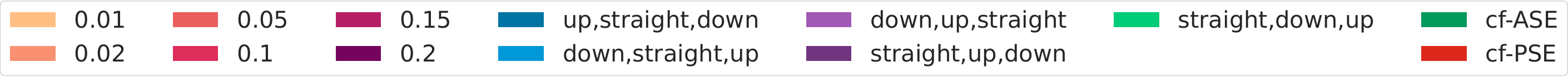}
    \end{subfigure}%
    \hfill%
    \\[1.5ex]
    \begin{subfigure}[b]{0.24\textwidth}
         \centering
         \includegraphics[width=\textwidth]{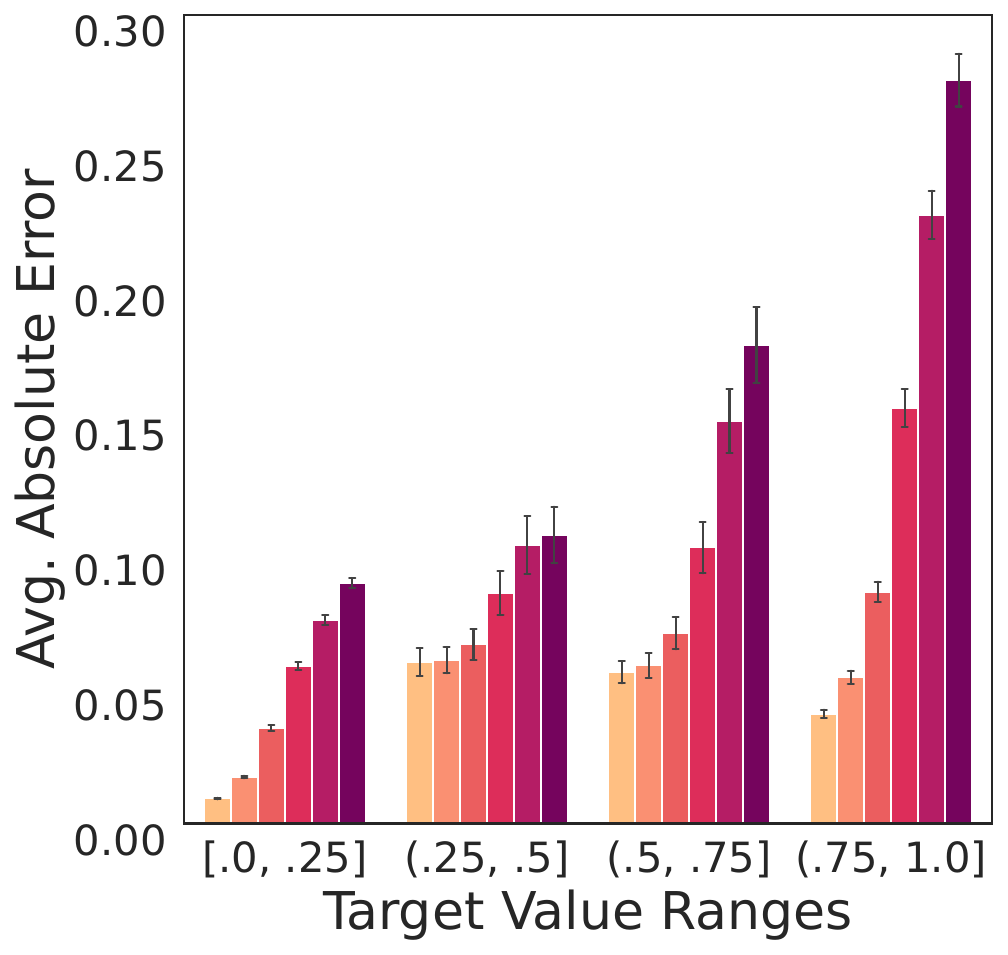}
         \caption{Graph: Distribution Error}
         \label{plot:graph_prob_error}
     \end{subfigure}
     \hfill
     \begin{subfigure}[b]{0.24\textwidth}
         \centering
         \includegraphics[width=\textwidth]{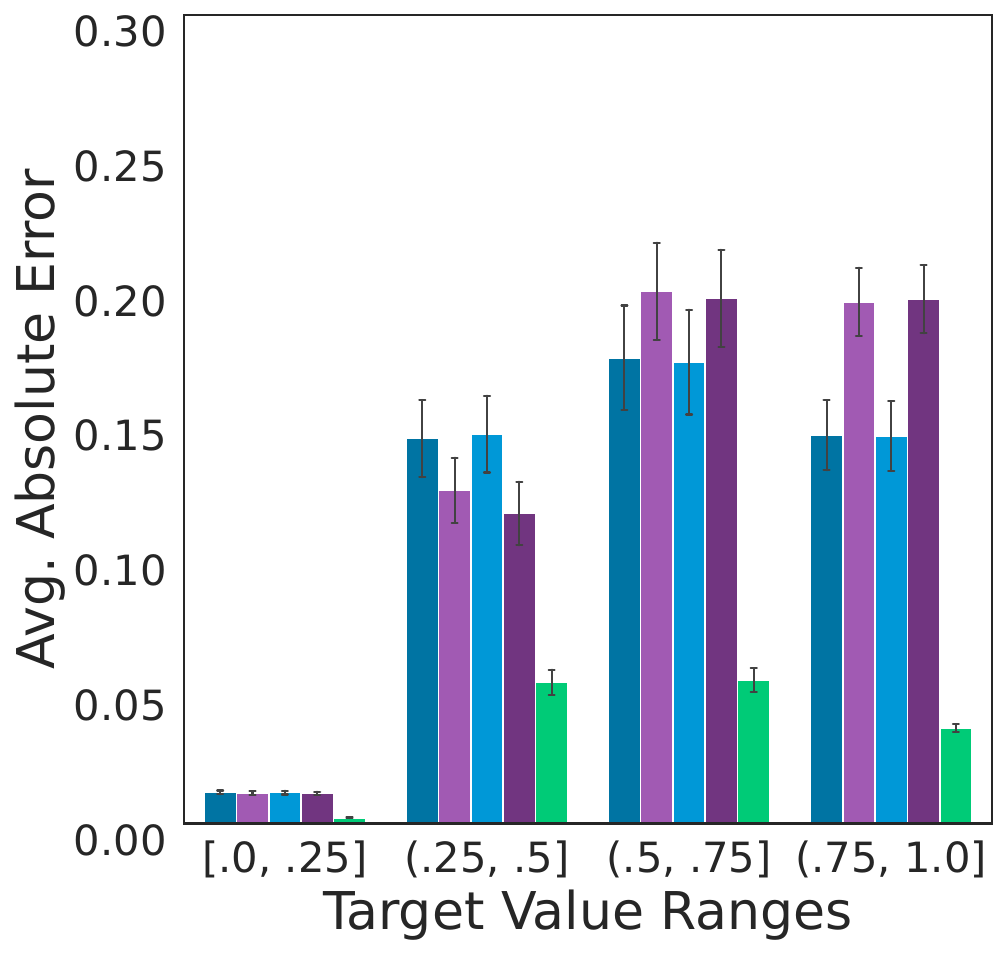}
         \caption{Graph: NM Violation}
         \label{plot:graph_ord_error}
     \end{subfigure}
     \hfill
     \begin{subfigure}[b]{0.24\textwidth}
         \centering
         \includegraphics[width=\textwidth]{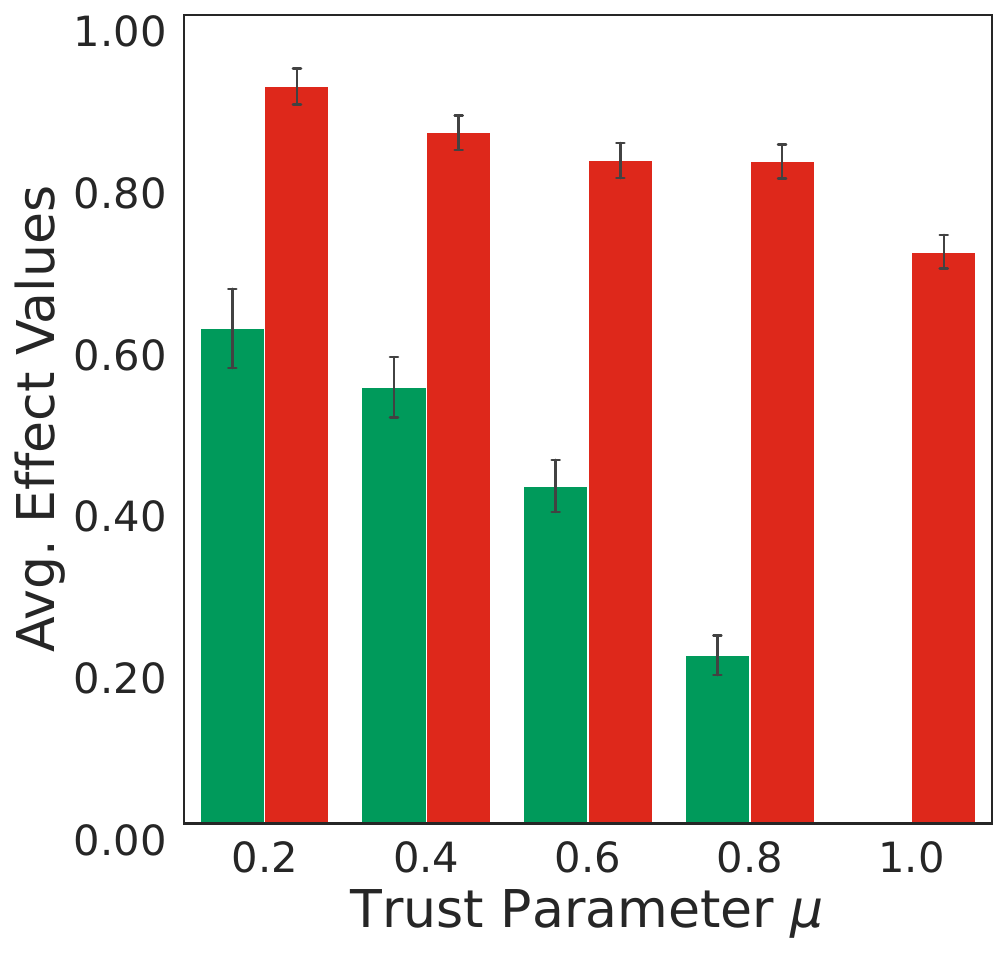}
         \caption{Sepsis: Through Clinician}
         \label{plot: clinician_se}
     \end{subfigure}
     \hfill
     \begin{subfigure}[b]{0.24\textwidth}
         \centering
         \includegraphics[width=\textwidth]{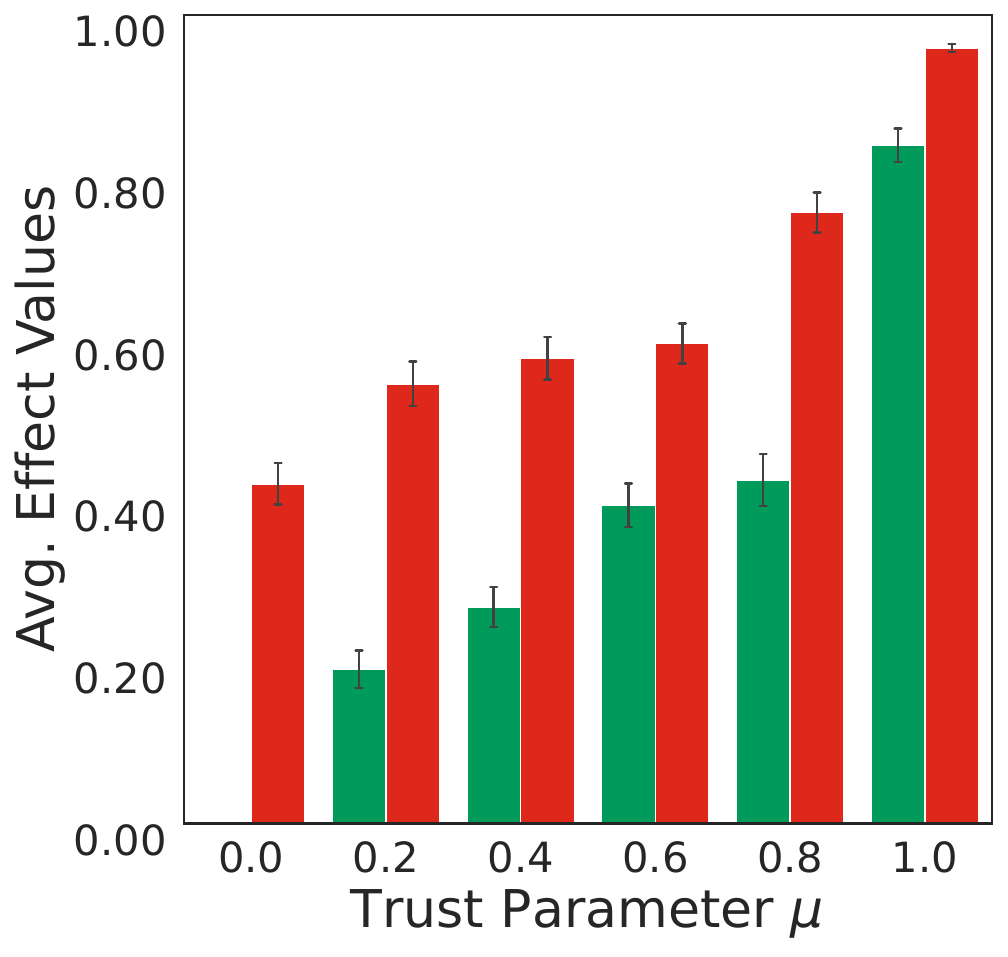}
         \caption{Sepsis: Through AI}
         \label{plot: ai_se}
     \end{subfigure}
    \caption{
    (Graph) Plots \ref{plot:graph_prob_error} and \ref{plot:graph_ord_error} show the average absolute error of cf-ASE when computed for wrongly estimated agents' policies and under violation of noise monotonicity, respectively.
    (Sepsis) Plots \ref{plot: clinician_se} and \ref{plot: ai_se} show average values of cf-ASE and cf-PSE when computing effects that propagate through the clinician and AI, respectively, while varying trust parameter $\mu$.
    }
    \label{fig:enter-label}
\end{figure*}

In this section, we empirically evaluate  
our approach 
using two environments from prior work, \textit{Graph}~\cite{triantafyllou2021blame} and \textit{Sepsis}~\cite{oberst2019counterfactual}, modified to align with our objectives. 
We refer the reader to Appendix \ref{app.exp_setup} for more details on the experimental setup, and to Appendix \ref{app.exp_results} for additional results related to the Sepsis environment.

\subsection{Environments and Experimental Setup}

{\bf Graph.}
In this environment, $6$ agents move in a graph with $1$ initial, $3$ terminal and $9$ intermediate nodes split into $3$ layers. All nodes of a column are connected (with directed edges) to all the next column's nodes. Thus, agents can move on the graph going \textit{up}, \textit{straight} or \textit{down}.
The agents' goal is to be equally distributed when reaching the last column.
The policy of each agent $i \in \{1, ..., 6\}$ is parameterized with a $p_i = 0.05 \cdot i$, which denotes their probability of taking a random action at any node. With $1 - p_i$ probability, agent $i$ follows a policy that is common among all agents, and depends on their current positions on the graph. 
We implement this environment, such that each action variable of the underlying MMDP-SCM is noise-monotonic w.r.t.~a predefined total ordering, which is \textit{up}, \textit{down}, \textit{straight}.

{\bf Sepsis.}
In our multi-agent version of Sepsis, AI and clinician have roles similar to those described in Section~\ref{sec.intro}. 
At each round, the AI recommends one of eight possible treatments, which is then reviewed and potentially overridden by the clinician.
If the clinician does not override AI's action, i.e., takes a \textit{no-op} action, then the treatment selected by the AI is applied. Otherwise, a new treatment selected by the clinician is applied.
The agents' goal is to keep the patient alive for $20$ rounds or achieve an earlier discharge. 
We learn the policies of both agents using Policy Iteration \cite{Sutton2018}, 
and we train them on slightly different transition probabilities.
For the clinician, this gives us their policy for selecting treatments, so we additionally need to specify 
their probability of overriding an action of the AI agent. Across all states, this is specified by a parameter $\mu$, which intuitively models the level of the clinician's trust in the AI: greater values of $\mu$ correspond to higher levels of trust. 

\textbf{Setup.} 
For each environment, we generate a set of trajectories in which agents fail to reach their goal, $500$ for Graph and $100$ (per value of $\mu$) for Sepsis. For each trajectory $\tau$, we compute the TCFE of all the potential alternative actions that agents could have taken.
Among these actions $a_{i,t}$, we retain those that exhibit a TCFE greater than or equal to a predefined threshold $\theta$ ($0.75$ for Graph and $0.8$ for Sepsis). Formally, we require $\mathrm{TCFE}_{a_{i,t}}(Y = \text{success}|\tau) \geq \theta$.
This process yields a total of $854$ selected alternative actions for Graph and $9197$ for Sepsis, which serve as the basis for evaluating cf-ASE. 
Counterfactual effects in our experiments are computed for $100$ counterfactual samples.

\subsection{Results}\label{sec.exp_graph}

\textbf{Robustness.} We first assess the robustness of cf-ASE in the presence of uncertainty, using the Graph environment, for which we know the underlying causal model. 
Given this knowledge, we can compute the correct values of cf-ASE, which we refer to as {\em target values}, and we do so for all $854$ alternative actions that were selected using the previously outlined process, covering all possible combinations of effect-agents ($31$ combinations in total). With this information, we perform two robustness tests. 

First, we test the robustness to uncertainty over the observational distribution of the model, specifically the agents' policies. 
Uncertainty is incorporated in the underlying model by randomly sampling a probability $\hat{p_i} \in [p_i - \epsilon_{max}, p_i + \epsilon_{max}]$, which defines $i$'s (wrongly) estimated policy. Here, $\epsilon_{max}$ is a parameter that we vary. Plot \ref{plot:graph_prob_error} shows the average absolute difference between the target values and the counterfactual agent-specific effects estimated for different levels of uncertainty (i.e., $\epsilon_{max}$). 
Results are grouped by target value range.
We observe that for small estimation errors ($\epsilon_{max} \leq 0.05$) we can still obtain good enough estimates of cf-ASE, especially when the target value is low.

Second, we assess the impact of violating the noise monotonicity assumption on the estimation of cf-ASE, by running our estimation algorithm with an incorrect total ordering. 
Plot \ref{plot:graph_ord_error} shows the average absolute difference between the target values and the counterfactual agent-specific effects estimated for different misspecified total orderings.
The results indicate that assuming the wrong total ordering can have a significant impact on the estimation accuracy of cf-ASE.
Interestingly, when the reverse of the correct total ordering is assumed, the estimation error is relatively low. 

\begin{figure*}
    \centering
     \begin{subfigure}[b]{0.24\textwidth}
         \centering
         \includegraphics[width=\textwidth]{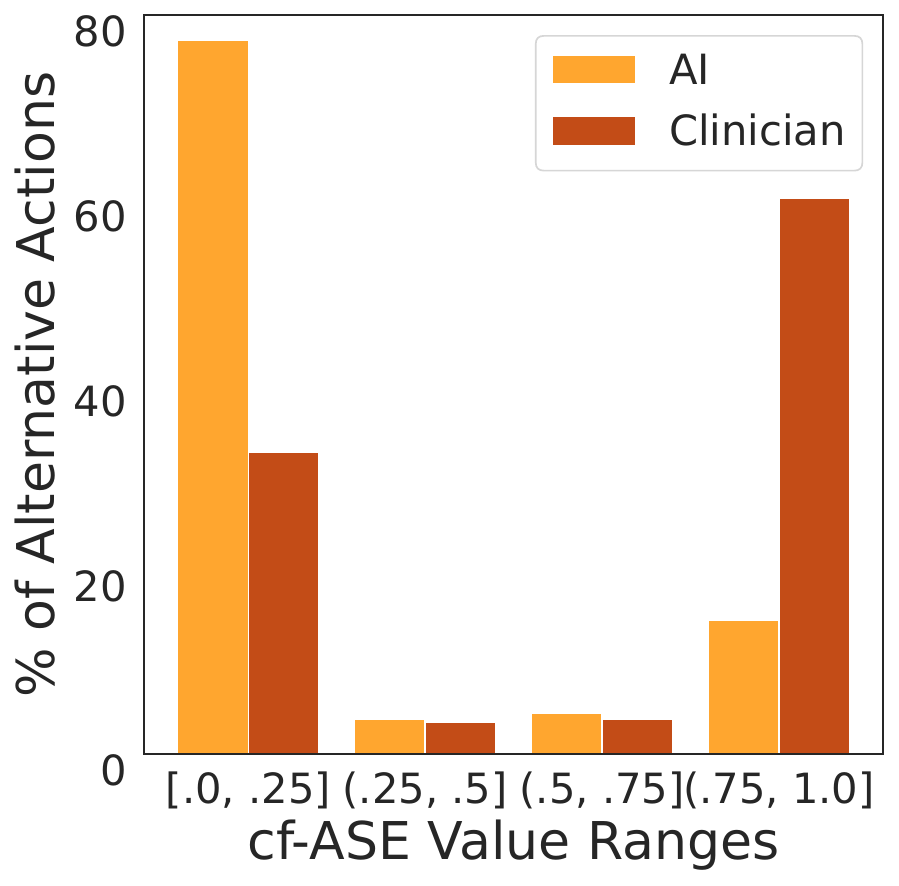}
         \caption{Sepsis: $\mu = 0.2$}
         \label{plot: sepsis_ase_perc02}
     \end{subfigure}
     \hfill
     \begin{subfigure}[b]{0.24\textwidth}
         \centering
         \includegraphics[width=\textwidth]{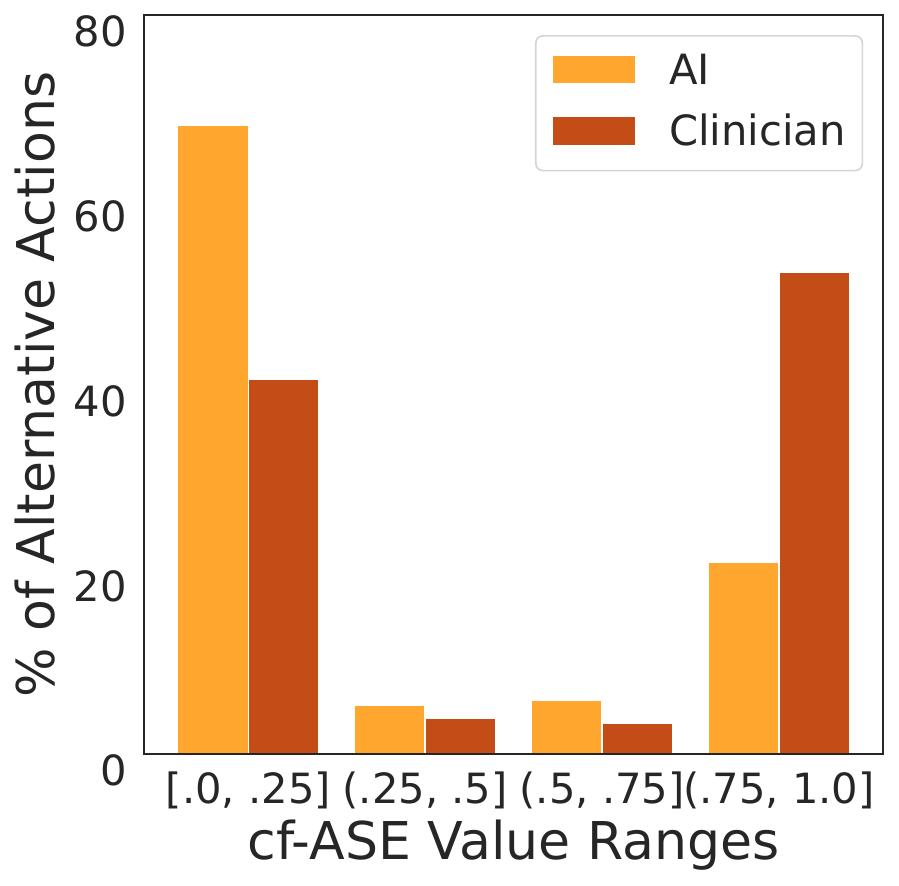}
         \caption{Sepsis: $\mu = 0.4$}
         \label{plot: sepsis_ase_perc04}
     \end{subfigure}
     \hfill
     \begin{subfigure}[b]{0.24\textwidth}
         \centering
         \includegraphics[width=\textwidth]{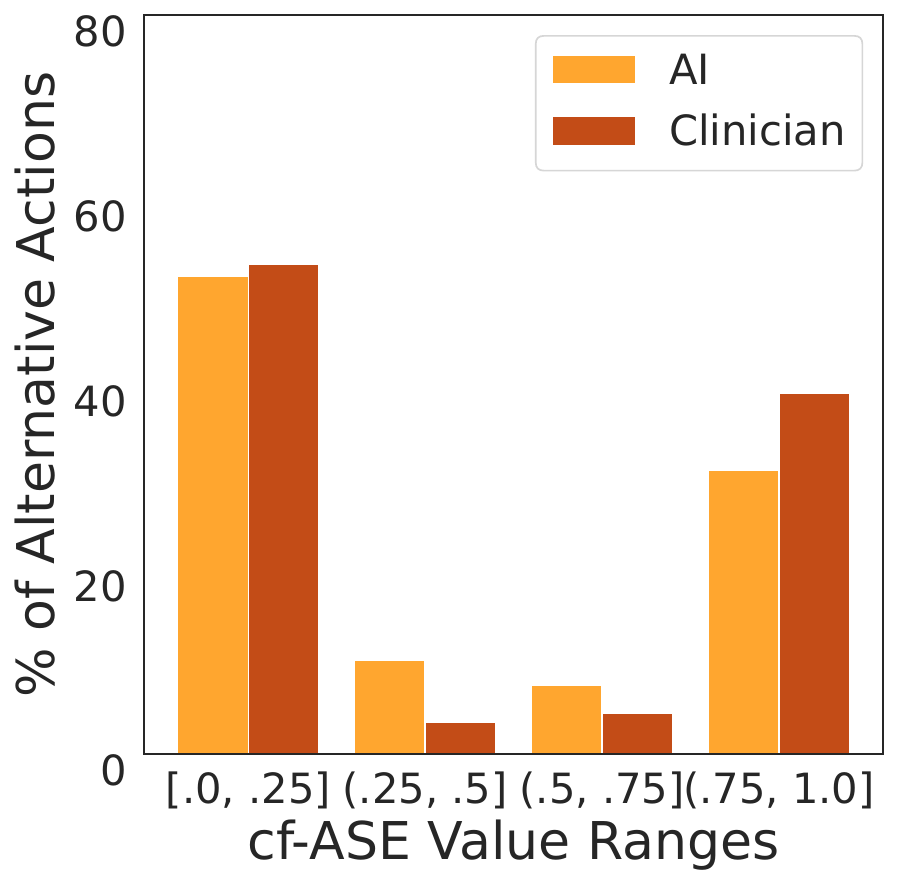}
         \caption{Sepsis: $\mu = 0.6$}
         \label{plot: sepsis_ase_perc06}
     \end{subfigure}
     \hfill
     \begin{subfigure}[b]{0.24\textwidth}
         \centering
         \includegraphics[width=\textwidth]{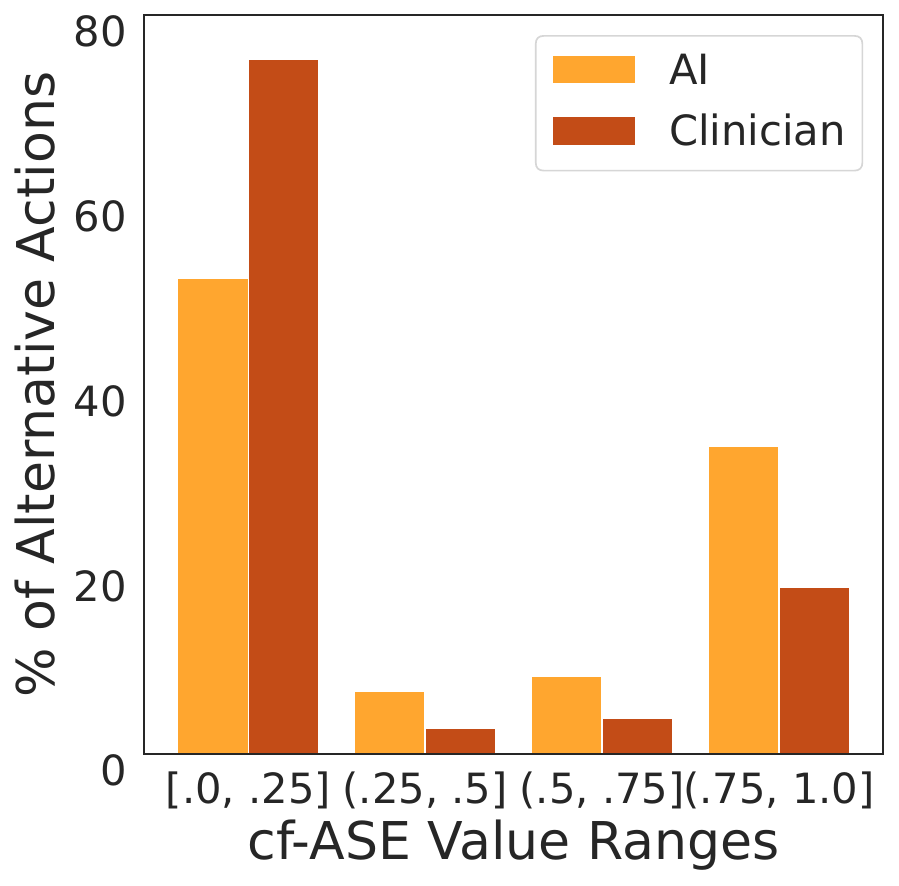}
         \caption{Sepsis: $\mu = 0.8$}
         \label{plot: sepsis_ase_perc08}
     \end{subfigure}
    \caption{
    (Sepsis) Plots in this figure show the distribution of selected alternative actions (of both the AI and the clinician) across different ranges of cf-ASE values and for different values of trust parameter $\mu$.
    }
\end{figure*}

\textbf{Practicality.} Next, we assess the conceptual importance and practicality of agent-specific effects using the Sepsis environment. For all selected actions of the AI (resp. clinician) we compute their counterfactual effect which propagates through the clinician (resp. AI).
In Plots \ref{plot: clinician_se} and \ref{plot: ai_se}, we present the average \textit{clinician-} and \textit{AI-specific effects}, as captured by cf-ASE and the counterfactual counterpart of PSE (cf-PSE), for different levels of trust.
Plots \ref{plot: sepsis_ase_perc02}-\ref{plot: sepsis_ase_perc08} show, for different values of trust parameter $\mu$, the distribution of selected alternative actions across value ranges of cf-ASE.

Plots \ref{plot: clinician_se}, \ref{plot: ai_se} and \ref{plot: sepsis_ase_perc02}-\ref{plot: sepsis_ase_perc08} suggest that the causal interdependencies between agents' policies can have a significant impact on the total counterfactual effect of an action. 
For instance, in Plot \ref{plot: sepsis_ase_perc06} we can see that even for a large trust level $\mu = 0.6$, close to $40\%$ of the AI's actions that admit TCFE greater than $0.8$, also admit a clinician-specific effect (as measured by cf-ASE) greater than $0.75$.  
This means that the total counterfactual effect of an action can rather frequently be attributed to the behavior of other agents in the system.

Plot \ref{plot: clinician_se} illustrates that our approach uncovers the following intuitive pattern: the clinician-specific effect decreases as the level of trust increases, ultimately dropping to $0$ when the human fully trusts the choices of the AI.
This is an important finding as it implies that given a set of trajectories, cf-ASE could provide us with valuable insight about the clinician's trust towards the AI.
A similar conclusion can be drawn from Plot \ref{plot: ai_se}, in which cf-ASE demonstrates the following trend: the AI-specific effect increases with $\mu$. This is expected as it implies that the effect of the clinician's actions on the outcome through the AI decreases, as the latter assumes less agency. 
Compared to cf-ASE, we observe that cf-PSE suggests a less prominent pattern for the clinician-specific effect. Most notably, it can assign positive value to the effect in cases where the clinician blindly follows AI's recommendations. A similar counter-intuitive result can also be observed for the AI-specific effect when $\mu=0$.
Given these results, we conclude that cf-ASE is a more suitable approach than cf-PSE for measuring the clinician- and AI-specific effects in this setting, as it aligns better with standard intuition.

Plots \ref{plot: sepsis_ase_perc02}-\ref{plot: sepsis_ase_perc08} showcase the following natural trend for cf-ASE: the fraction of actions with clinician-specific effect greater than $0.75$ decreases with the level of trust, while the fraction of actions with clinician-specific effect less than $0.25$ increases. This is expected as the less the clinician trusts the AI's actions, the more often they override them, and hence the more they tend to affect the outcome.
Similarly, the fraction of actions with AI-specific effect greater than $0.75$ increases with the level of trust, while the fraction of actions with AI-specific effect less than $0.25$ decreases. This is also expected as the less the clinician trusts the AI's actions the less agency the AI assumes.

\begin{remark}
In the Sepsis experiments, we do not have access to the underlying causal model and instead we only have access to the simulator's transition probabilities and the agents' policies.
Thus, for our analysis we assume that trajectories are generated by an MMDP-SCM with variables that are noise-monotonic w.r.t.~a chosen total ordering.
While this assumption is restrictive, we found that the results of this section are rather robust to this type of uncertainty, which is important in practice, where the assumptions we made in this paper are often violated.
Plots for $5$ additional randomly selected total ordering sets, which back up this claim, can be found in Appendix \ref{app.exp_results}.
\end{remark}


\section{Conclusion}

To summarize, in this paper we introduce agent-specific effects, a novel causal quantity that measures the effect of an action on the outcome that propagates through a set of agents.
To the best of our knowledge, this is the first work that looks into this problem in the context of multi-agent sequential decision making.
Our theoretical contributions include results on the identifiability of this new concept, a practical algorithm to estimate it, and establishing a formal connection with prior work. In the experiments, we demonstrate the practical effectiveness of our approach in a scenario where our theoretical assumptions do not hold.

\textbf{Future work.} Looking forward, we recognize three compelling future directions. 
First, a natural next step would be to derive a \textit{causal explanation formula} akin to \cite{zhang2018fairness}, which will decompose the causal effect of an agent's action by attributing to each agent and subsequent environment state a score reflecting its contribution to the effect.
Second, it would be of practical importance to explore the identifiability of ASE in the presence of unobserved confounding.
Finally, we plan to investigate the utility of ASE in other practical settings, where causal effect propagation analysis is important, e.g., in multi-agent RL.

\section*{Impact Statement}

\textbf{Impact on accountable decision making.} We believe that our approach can meaningfully extend auditing methods for decision-making outcomes.
For example, current approaches on responsibility attribution, counterfactual harm, etc.~focus solely on computing the counterfactual effect of an agent’s action on the outcome (see Section \ref{sec.related_work} for references). 
In this work, we aim to further analyze this effect, by measuring the extent to which it should be attributed to how other agents would respond to this action, providing a more granular insight. Thus, modifying existing causal tools to reason about agent-specific effects could lead to more informed judgements about accountability in multi-agent decision making.

\textbf{Ethical impact.} 
The approach we propose in this paper is situated in the post-processing phase of decision-making scenarios. While it does not directly impact human subjects, we recognize the imperative of conducting more empirical studies to validate its functionality and integration into accountability measures. We advocate for targeted research, including a thorough sensitivity analysis of our theoretical assumptions, in order to ensure the reliability of our approach. 
Finally, prior to utilizing ASE in practice, it is important to ensure that ASE-based explanations or measures are aligned with human intuitions and norms, which is especially important for high-stake domains, such as healthcare. 


\section*{Acknowledgements}
This research was, in part, funded by the Deutsche Forschungsgemeinschaft (DFG, German Research Foundation) – project number $467367360$.

\bibliography{main}
\bibliographystyle{icml2024}

\newpage
\appendix
\onecolumn


f

\section{List of Appendices}

In this section, we provide a brief description of the content provided in the appendices of the paper.

\begin{itemize}
    \item Appendix \ref{app.causal_graph} provides the causal graph of the MMDP-SCM from Section \ref{sec.framework_mmdp}.
    \item Appendix \ref{app.example} provides the formal expression of Definition \ref{def.cf-ase} to the example shown in Fig.~\ref{fig: ase}.
    \item Appendix \ref{app.exp_setup} provides additional details on the experimental setup and implementation.
    \item Appendix \ref{app.exp_results} provides additional experimental results for the Sepsis environment.
    \item Appendix \ref{app.related_work} provides additional related work.
    \item Appendix \ref{app.exp_background} provides additional background needed for the proofs of our theoretical results.
    \item Appendix \ref{app.ctf_remark} contains a remark about Lemma \ref{lem.pns_nm} and counterfactual identification.
    \item Appendix \ref{app.fpse} provides additional details and theoretical results related to fixed path-specific effects.
    \begin{itemize}
        \item Appendix \ref{proof.pns_nm} contains the proof of Lemma \ref{lem.pns_nm}.
    \end{itemize}
    \item Appendix \ref{proof.ase_id} contains the proofs of Theorems \ref{thrm.ase_nonid} and \ref{thrm.id_ase_cf}.
    \item Appendix \ref{proof.nm_mmdp} contains the proof of Lemma \ref{lemma.nm_mmdp}.
    \item  Appendix \ref{proof.nm_m} contains the proof of Proposition \ref{prop.nm_m}.
    \item Appendix \ref{proof.algo_unbias} contains the proof of Theorem \ref{thrm.algo_unbias}.
    \item Appendix \ref{proof.fpse_ase} contains the proof of Lemma \ref{lemma.fpse_ase}.
    \item Appendix \ref{proof.fpse_pse} contains the proof of Proposition \ref{prop.fpse_pse}.
\end{itemize}

\section{Causal Graph of MMDP-SCM}\label{app.causal_graph}

This section includes the causal graph of the MMDP-SCM described in Section \ref{sec.framework_mmdp}. The causal graph is shown in Fig.~\ref{fig: causal_graph}.

\begin{figure*}
\centering
\begin{tikzpicture}[
    >=stealth', 
    shorten >=1pt, 
    auto,
    node distance=0.5cm, 
    scale=.7, 
    transform shape, 
    align=center, 
    state/.style={circle, draw, minimum size=1.5cm}]


\node[state] (S-0) at (0,0) {$S_0$};

\node[state] (A-10) [rectangle, above right=1cm of S-0] {$A_{1,0}$};
\path (S-0) edge (A-10);

\node[state] (A-n0) [rectangle, right=1.5cm of A-10] {$A_{n,0}$};
\path (S-0) edge (A-n0);
\path (A-10) -- node[auto=false]{\ldots} (A-n0);

\node[state] (S-1) [below right=1cm of A-n0] {$S_1$};
\path (S-0) edge (S-1);
\path (A-10) edge (S-1);
\path (A-n0) edge (S-1);

\node[state] (A-11) [rectangle, above right=1cm of S-1] {$A_{1,1}$};
\path (S-1) edge (A-11);

\node[state] (A-n1) [rectangle, right=1.5cm of A-11] {$A_{n,1}$};
\path (S-1) edge (A-n1);
\path (A-11) -- node[auto=false]{\ldots} (A-n1);

\node[state] (S-2) [below right=1cm of A-n1] {$S_2$};
\path (S-1) edge (S-2);
\path (A-11) edge (S-2);
\path (A-n1) edge (S-2);


\node[state] (S-h1) [right=1cm of S-2] {$S_{h-1}$};
\path (S-2) -- node[auto=false]{\ldots} (S-h1);

\node[state] (A-1h1) [rectangle, above right=1cm of S-h1] {$A_{1,h-1}$};
\path (S-h1) edge (A-1h1);

\node[state] (A-nh1) [rectangle, right=1.5cm of A-1h1] {$A_{n,h-1}$};
\path (S-h1) edge (A-nh1);
\path (A-1h1) -- node[auto=false]{\ldots} (A-nh1);

\node[state] (S-h) [below right=1cm of A-nh1] {$S_h$};
\path (S-h1) edge (S-h);
\path (A-1h1) edge (S-h);
\path (A-nh1) edge (S-h);

\end{tikzpicture}

\captionsetup{type=figure}
\caption{The causal graph of an MMDP-SCM with $n$ agents and horizon $h$. Exogenous variables are omitted.}
\label{fig: causal_graph}
\end{figure*}
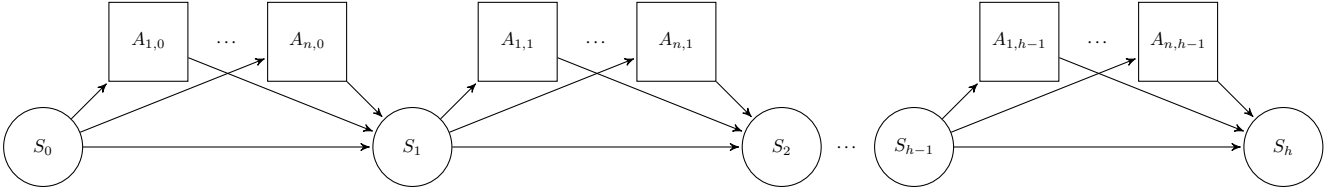
\section{Formal Expression of Definition \ref{def.cf-ase} to the Example shown in Fig.~\ref{fig: ase}}\label{app.example}

In the example illustrated in Fig.~\ref{fig: ase}, the clinician's actions are set to the values that they would naturally take under intervention $do(A_0 := E)$, i.e., $H_0 := H_{0, do(A_0 := E)}$ and $H_1 := H_{1, do(A_0 := E)}$. In the same example, the second action of the AI agent is set to its factual action, i.e., $A_1 := \tau(A_1)$, where $\tau$ here denotes the factual trajectory.
These together form set $I$ from Definition \ref{def.cf-ase}, i.e., $I = \{A_1 := \tau(A_1)\} \cup \{H_0 := H_{0, do(A_0 := E)}, H_1 := H_{1, do(A_0 := E)}\}$.
We can then formally express the counterfactual clinician-specific effect of treatment $E$ on outcome $y$, in this example, as follows
$$\mathrm{cf\text{-}ASE}^{CL}_{A_0 := E}(y|\tau)_M = P(y_{do(A_0 := C)}|\tau;M)_{M^{do(I)}} - P(y|\tau)_M,$$
where $M$ denotes the underlying MMDP-SCM.

\section{Additional Information on Experimental Setup and Implementation}\label{app.exp_setup}

In this section, we provide additional information on the  experimental setup and implementation details.

\subsection{Experimental Setup}

We first provide a more detailed description of the environments and the agents' policies used in the experiments.

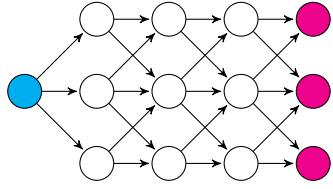
\begin{figure}
\centering
\begin{tikzpicture}[>=stealth', shorten >=1pt, auto,
    node distance=1cm, scale=.45, 
    transform shape, align=center, 
    state/.style={circle, draw, minimum size=1cm}]


\node[state, fill=cyan] (N-0) at (0,0) {};


\node[state] (N-1) [below right=2cm of N-0] {};
\node[state] (N-2) at (N-0 -| N-1) {};
\node[state] (N-3) [above right=2cm of N-0] {};

\path (N-0) edge (N-1);
\path (N-0) edge (N-2);
\path (N-0) edge (N-3);


\node[state] (N-4) [below right=2cm of N-2] {};
\node[state] (N-5) at (N-2 -| N-4) {};
\node[state] (N-6) [above right=2cm of N-2] {};

\path (N-1) edge (N-4);
\path (N-1) edge (N-5);

\path (N-2) edge (N-4);
\path (N-2) edge (N-5);
\path (N-2) edge (N-6);

\path (N-3) edge (N-5);
\path (N-3) edge (N-6);


\node[state] (N-7) [below right=2cm of N-5] {};
\node[state] (N-8) at (N-5 -| N-7) {};
\node[state] (N-9) [above right=2cm of N-5] {};

\path (N-4) edge (N-7);
\path (N-4) edge (N-8);

\path (N-5) edge (N-7);
\path (N-5) edge (N-8);
\path (N-5) edge (N-9);

\path (N-6) edge (N-8);
\path (N-6) edge (N-9);


\node[state, fill=magenta] (N-10) [below right=2cm of N-8] {};
\node[state, fill=magenta] (N-11) at (N-8 -| N-10) {};
\node[state, fill=magenta] (N-12) [above right=2cm of N-8] {};

\path (N-7) edge (N-10);
\path (N-7) edge (N-11);

\path (N-8) edge (N-10);
\path (N-8) edge (N-11);
\path (N-8) edge (N-12);

\path (N-9) edge (N-11);
\path (N-9) edge (N-12);

\end{tikzpicture}

\captionsetup{type=figure}
\caption{Graph environment. The cyan node is the initial node and magenta nodes are the terminal nodes of the graph.}
\label{fig: env_graph}
\end{figure}

\textbf{Graph.} We consider the graph environment depicted in Fig. \ref{fig: env_graph}, in which $6$ agents take simultaneous actions.
All agents begin at the left most node, and at each time-step move to nodes of the next column following the directed edges.
Thus, each agent can select to either move \textit{up}, \textit{straight} or \textit{down}, with out-of-bounds moves leading to moving straight.
The agents' goal in this environment is to be equally distributed when reaching the last column of the graph, i.e., have $2$ agents per node. Agents reach the last column of the graph in $4$ time-steps, when also the environment terminates.

At the initial node, all agents take a random action. At all other nodes (except the terminal ones), each agent $i \in \{1, 2, 3, 4, 5, 6\}$ takes a random action with an agent-specific probability $p_i = 0.05 \cdot i$. At a node $k$ that is occupied by a total number of agents $N_k \leq 2$, agent $i$ moves straight with probability $1 - p_i$. Otherwise, $i$ moves straight with probability $(1 - p_i) \cdot \frac{2}{N_k}$, and towards a row that has less than $2$ agents with probability $(1 - p_i) \cdot \frac{N_k - 2}{N_k}$.\footnote{If there are two rows with less than $2$ agents where $i$ could move to, then each of these moves has probability $(1 - p_i) \cdot \frac{N_k - 2}{2 \cdot N_k}$.}

\textbf{Sepsis.}
The patient's state in this environment can be described by $4$ vital signs 
(heart rate, systolic blood pressure, percent of oxygen saturation and glucose level) 
as well as an additional diabetes variable 
present with $20\%$ probability, influencing the fluctuations of the glucose level.
Both agents can select among 
$8$ distinct 
actions, each of which represents a combination of applying antibiotics (A), vasopressors (V) and mechanical ventilation (E) treatment options. As mentioned in Section \ref{sec.exp}, the clinician can additionally take a no-op action, which means that they do not override the AI's action at that step.
For more details regarding the simulator we consider in these experiments, we refer the reader to \cite{oberst2019counterfactual}.

\begin{figure}\label{fig: sepsis_policies}
    \centering
     \hspace*{\fill}
     \begin{subfigure}[b]{0.23\textwidth}
         \centering
         \includegraphics[width=\textwidth]{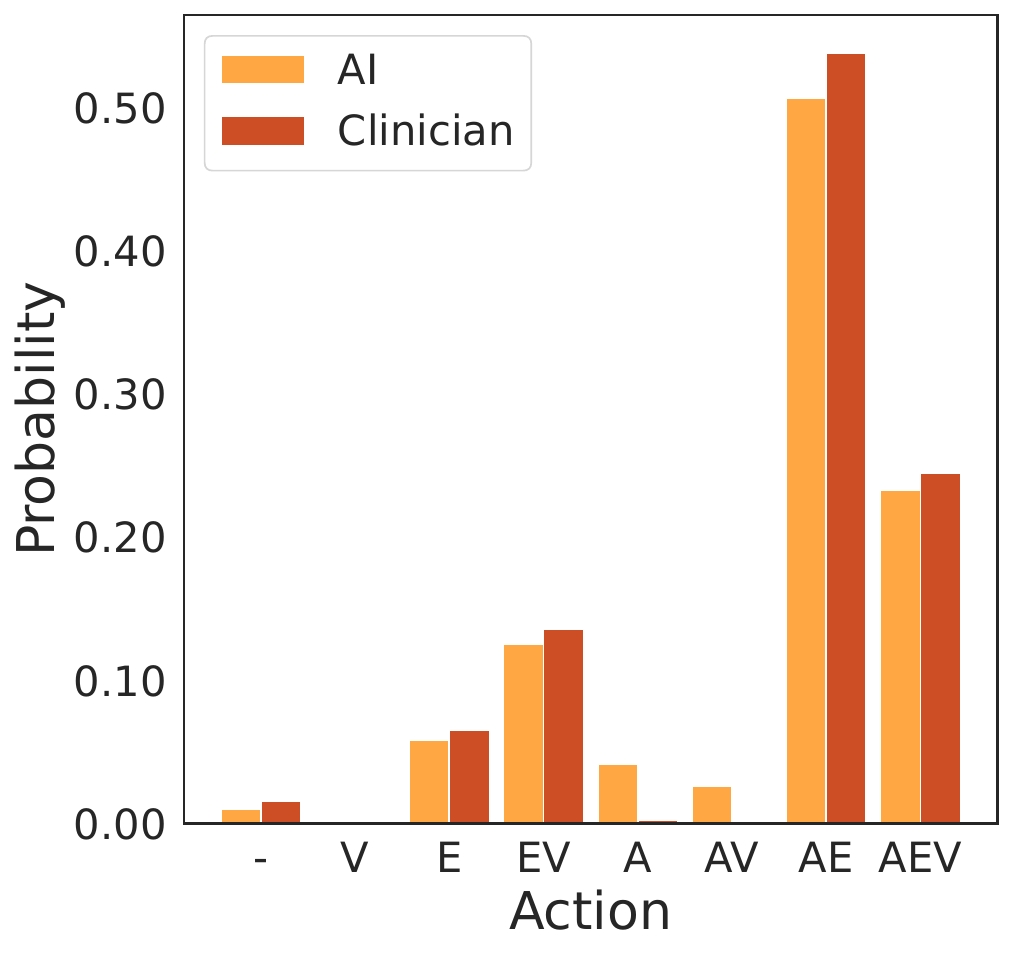}
         \caption{Sepsis: Policy Discrepancy}
         \label{plot: sepsis_discrepancy}
     \end{subfigure}
     \hfill
     \begin{subfigure}[b]{0.23\textwidth}
         \centering
         \includegraphics[width=\textwidth]{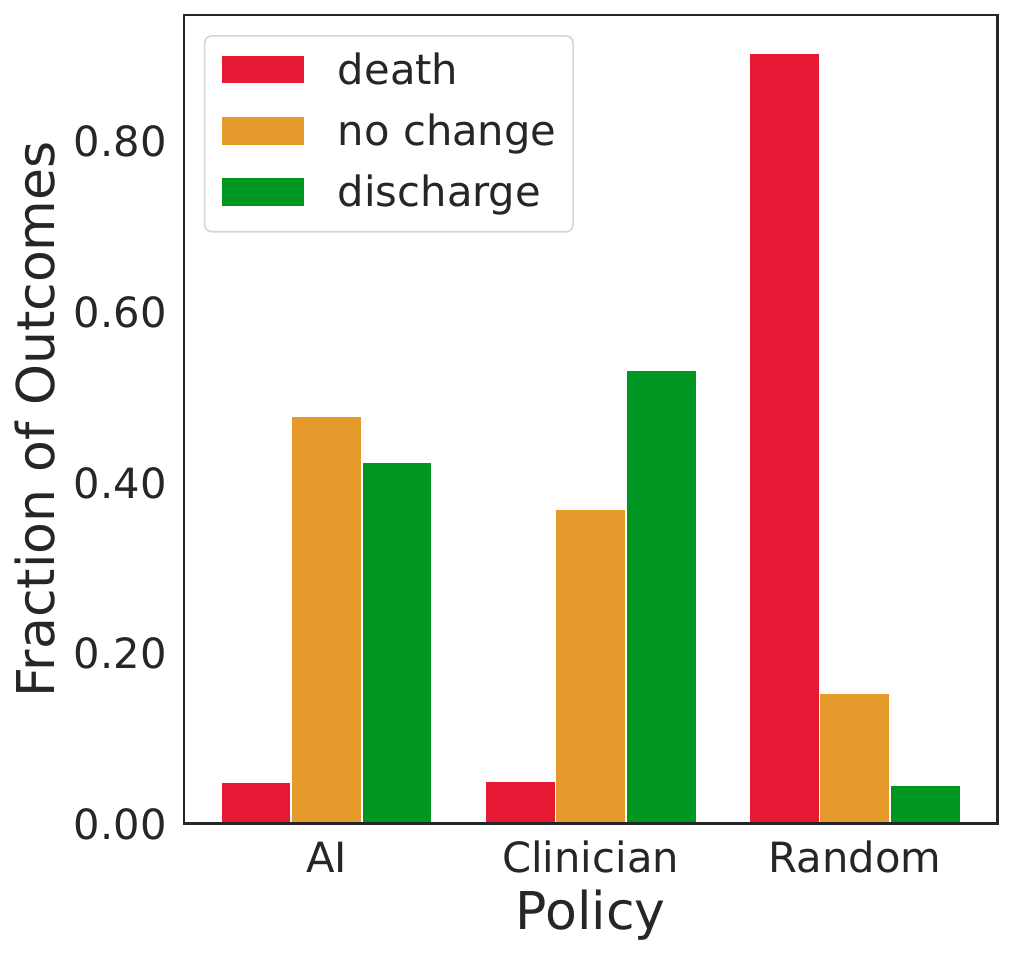}
         \caption{Sepsis: Policy Evaluation}
         \label{plot: sepsis_evaluation}
     \end{subfigure}
     \hspace*{\fill}
    \caption{(a) Probabilities that AI and clinician policies assign to different actions. (b) Observed outcomes across 1000 trajectories sampled from the simulator for AI, clinician and random policies.}
\end{figure}

The clinician's policy is trained using the simulator's true transition probabilities. To introduce grounds for action overriding by the clinician, the AI's policy is trained using slightly modified transition probabilities.
Namely, we increase the probability of antibiotics effectiveness and decrease the probability of vasopressors effectiveness (compared to the original simulator). This results in a policy for the AI which recommends treatments with vasopressors more scarcely and those with antibiotics more abundantly (compared to the clinician's policy), visualized in Plot 7a. To qualitatively evaluate learned policies, we sample 1000 trajectories from the simulator and depict the observed rewards in Plot 7b. For both policies we observe similar levels of undesirable outcomes, but note that treatments prescribed by the clinician tend to lead to an earlier patient discharge, compared to those prescribed by the AI, which tend to keep the patient alive, but do not completely cure them until the end of trajectory. 

We note that for a trajectory in which an undesirable outcome is realized, i.e., the patient dies at a time-step $t$ before the completion of $20$ rounds, the outcome $Y$ is the final state of this trajectory. In other words, in our analysis we focus on the counterfactual probability that the patient would not have died at time $t$.

\subsection{Algorithms for TCFE and cf-PSE}

We now present the algorithms we use in our experiments to estimate total counterfactual effects (TCFE) and the counterfactual counterpart of path-specific effects (cf-PSE). Similar to Algorithm \ref{alg.cf-ase}, Algorithms \ref{alg.tcfe} and \ref{alg.cf-pse} follow the standard \textit{abduction-action-prediction} methodology for causal inference \cite{pearl2009causality}.
Furthermore, similar results to Theorem \ref{thrm.algo_unbias} can be derived for the unbiasedness of the two algorithms.

\begin{algorithm}
\caption{Estimates $\mathrm{TCFE}_{a_{i,t}}(y|\tau)_M$}
\label{alg.tcfe}
\textbf{Input}: MMDP-SCM $M$, trajectory $\tau$, action variable $A_{i,t}$, action $a_{i,t}$, outcome variable $Y$, outcome $y$, number of posterior samples $H$
\begin{algorithmic}[1]
\STATE $h \gets 0$
\STATE $c \gets 0$
\WHILE{$h < H$}
    \STATE $\mathbf{u}_h \sim P(\mathbf{u} | \tau)$ \COMMENT{Sample noise from the posterior}
    \STATE $h \gets h + 1$\;
    \STATE $y^h \sim P(Y|\mathbf{u}_h)_{M^{do(A_{i,t} := a_{i,t})}}$ \COMMENT{Sample cf outcome}
    \IF{$y^h = y$}
        \STATE $c \gets c + 1$
    \ENDIF
\ENDWHILE
\STATE \textbf{return} $\frac{c}{H} - \mathds{1}(\tau(Y) = y)$
\end{algorithmic}
\end{algorithm}

\begin{algorithm}
\caption{Estimates $\mathrm{cf\text{-}PSE}^g_{a_{i,t}}(y|\tau)_M$}
\label{alg.cf-pse}
\textbf{Input}: MMDP-SCM $M$, trajectory $\tau$, effect agents $\mathbf{N}$, action variable $A_{i,t}$, action $a_{i,t}$, outcome variable $Y$, outcome $y$, number of posterior samples $H$
\begin{algorithmic}[1]
\STATE $h \gets 0$
\STATE $c \gets 0$
\WHILE{$h < H$}
    \STATE $\mathbf{u}_h \sim P(\mathbf{u} | \tau)$ \COMMENT{Sample noise from the posterior}
    \STATE $h \gets h + 1$
    \STATE $I \gets \{A_{i,t} := a_{i,t}\} \cup \{A_{i',t'} := \tau(A_{i',t'})\}_{i' \notin \mathbf{N}, t' > t}$
    \STATE $y^h \sim P(Y|\mathbf{u}_h)_{M^{do(I)}}$ \COMMENT{Sample cf outcome}
    \IF{$y^h = y$}
        \STATE $c \gets c + 1$
    \ENDIF
\ENDWHILE
\STATE \textbf{return} $\frac{c}{H} - \mathds{1}(\tau(Y) = y)$
\end{algorithmic}
\end{algorithm}

Note that for $\mathrm{cf\text{-}PSE}^g_{a_{i,t}}(y|\tau)_M$ in Algorithm \ref{alg.cf-pse}, $g$ represents the edge-subgraph of the causal graph of $M$, $G$, with 
$
\mathbf{E}(g) = \mathbf{E}(G) \backslash \bigcup_{A_{i',t'}: i' \notin \mathbf{N}, t' > t} \mathbf{E}^{A_{i',t'}}_-(G),
$
where $\mathbf{E}^{A_{i',t'}}_-(G)$ denotes the set of incoming edges of $A_{i',t'}$ in $G$.

\subsection{Compute Architecture and Software Stack}

All experiments were run on a 64bit Debian-based machine having 4x12 CPU cores clocked at 3GHz with access to 1.5TB of DDR3 1600MHz RAM. The software stack relied on Python 3.9.13, with installed standard scientific packages for numeric calculations and visualization: NumPy (1.24.3), Pandas (2.0.3), Joblib (1.3.1), Matplotlib (3.7.1) and Seaborn (0.12.2). For learning agent's policies in sepsis experiments, we relied on PyMDPToolbox (4.0-b3) library. The experiments are fully reproducible using the random seed $8854$. 
The total running time of the experiments on the Graph environment is $\sim 2.5$ hours and of the experiments on the Sepsis environment is $\sim 7.5$ hours. 
\section{Additional Experimental Results}\label{app.exp_results}

In this section, we provide additional experimental results for the Sepsis environment.
In Plots \ref{plot: sepsis_orderings_cl0}-\ref{plot: sepsis_orderings_ai4}, we present the average \textit{clinician-} and \textit{AI-specific effects}, as captured by cf-ASE and the counterfactual counterpart of PSE (cf-PSE), for different levels of trust and across $5$ randomly selected total ordering sets.
From the plots corresponding to any of these sets of total orderings, we can draw similar conclusions to the ones we drew from Plots \ref{plot: clinician_se} and \ref{plot: ai_se}.
Therefore, as also mentioned in Section \ref{sec.exp}, these plots support our claim that the results of that section are robust to violations of our theoretical assumptions, which often happen in practice.

\begin{figure*}
\label{fig: sepsis_orderings}
    \centering
    \begin{subfigure}[c]{0.22\textwidth}
        \centering
        \includegraphics[width=\textwidth,keepaspectratio]{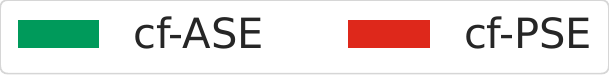}
    \end{subfigure}%
    \hfill%
    \\[1.5ex]
     \begin{subfigure}[b]{0.24\textwidth}
         \centering
         \includegraphics[width=\textwidth]{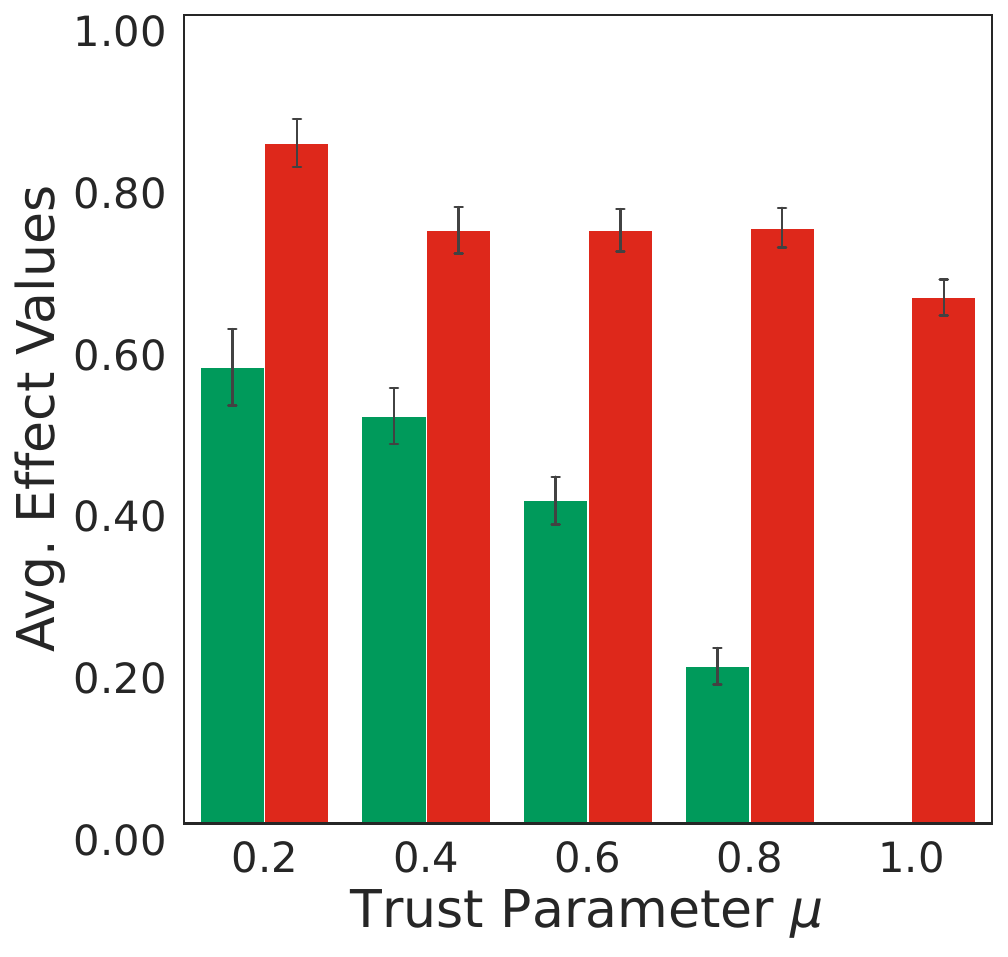}
         \caption{Sepsis: Through CL Ord1}
         \label{plot: sepsis_orderings_cl0}
     \end{subfigure}
     \hfill
     \begin{subfigure}[b]{0.24\textwidth}
         \centering
         \includegraphics[width=\textwidth]{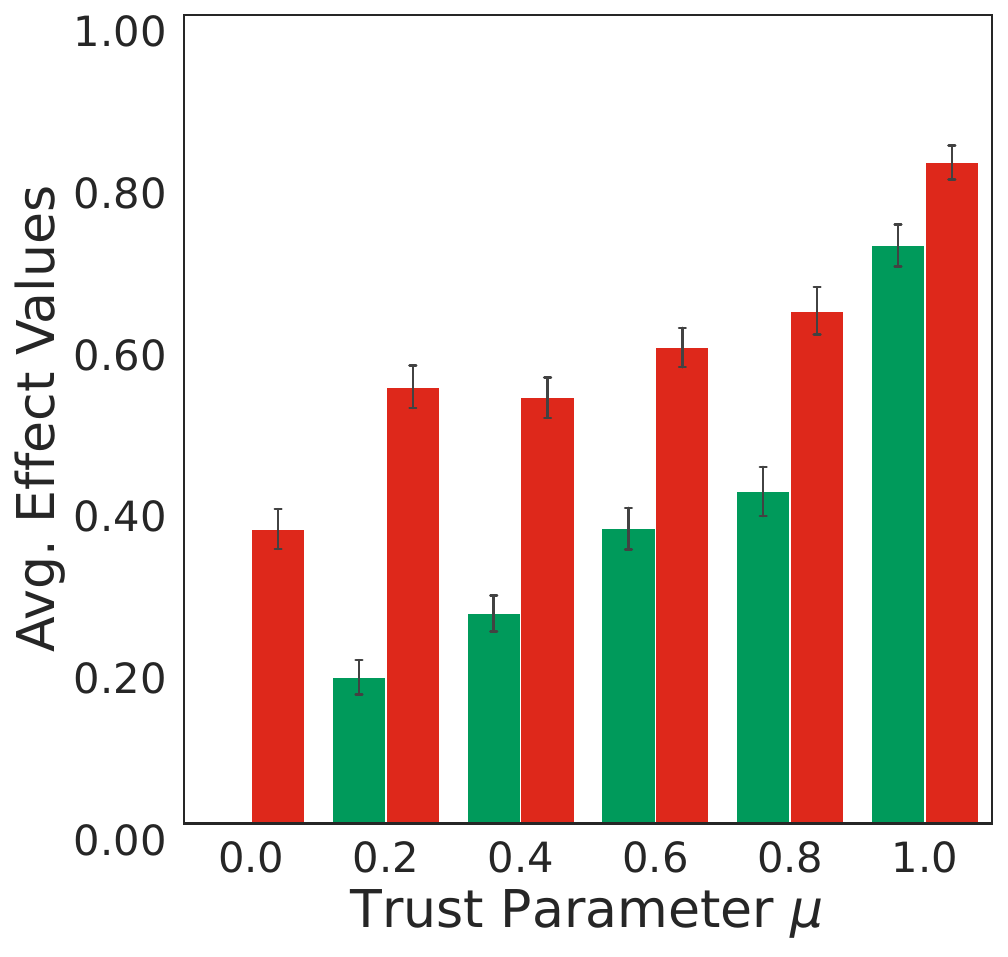}
         \caption{Sepsis: Through AI Ord1}
         \label{plot: sepsis_orderings_ai0}
     \end{subfigure}
     \hfill
     \begin{subfigure}[b]{0.24\textwidth}
         \centering
         \includegraphics[width=\textwidth]{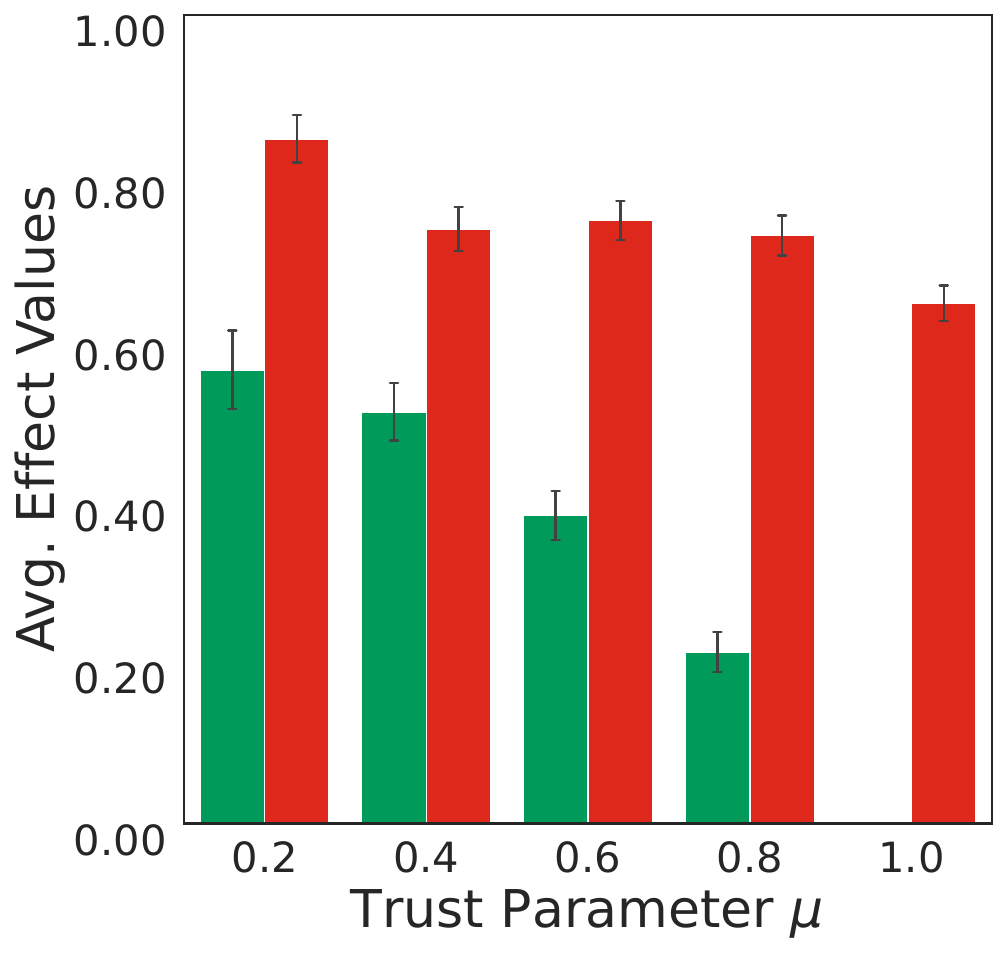}
         \caption{Sepsis: Through CL Ord2}
         \label{plot: sepsis_orderings_cl1}
     \end{subfigure}
     \hfill
     \begin{subfigure}[b]{0.24\textwidth}
         \centering
         \includegraphics[width=\textwidth]{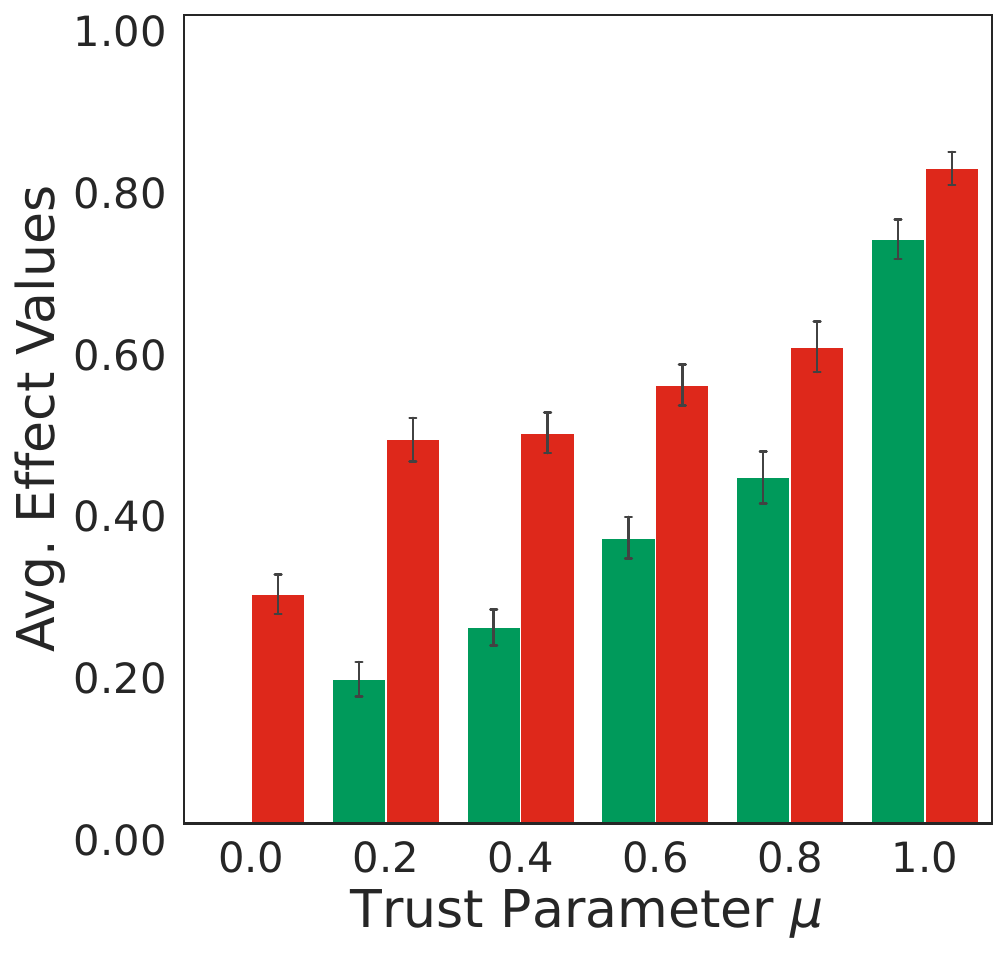}
         \caption{Sepsis: Through AI Ord2}
         \label{plot: sepsis_orderings_ai1}
     \end{subfigure}
     \\
     \begin{subfigure}[b]{0.24\textwidth}
         \centering
         \includegraphics[width=\textwidth]{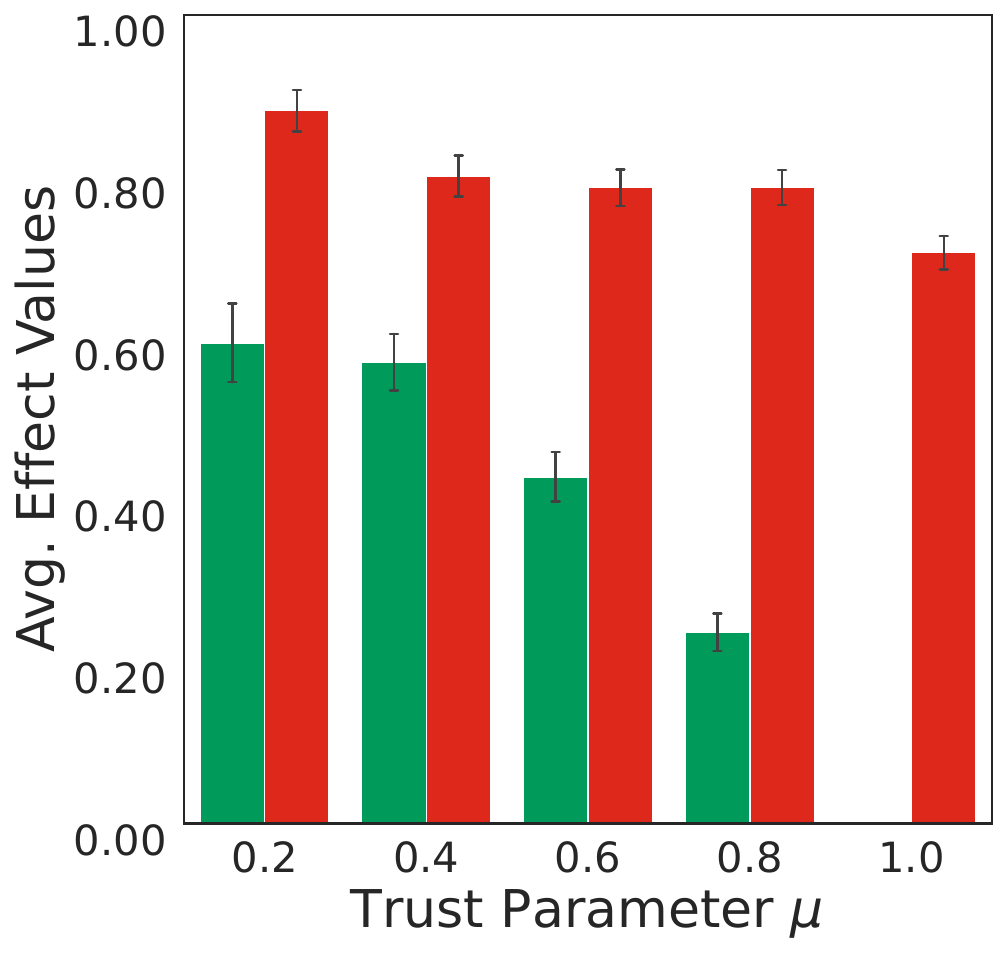}
         \caption{Sepsis: Through CL Ord3}
         \label{plot: sepsis_orderings_cl2}
     \end{subfigure}
     \hfill
     \begin{subfigure}[b]{0.24\textwidth}
         \centering
         \includegraphics[width=\textwidth]{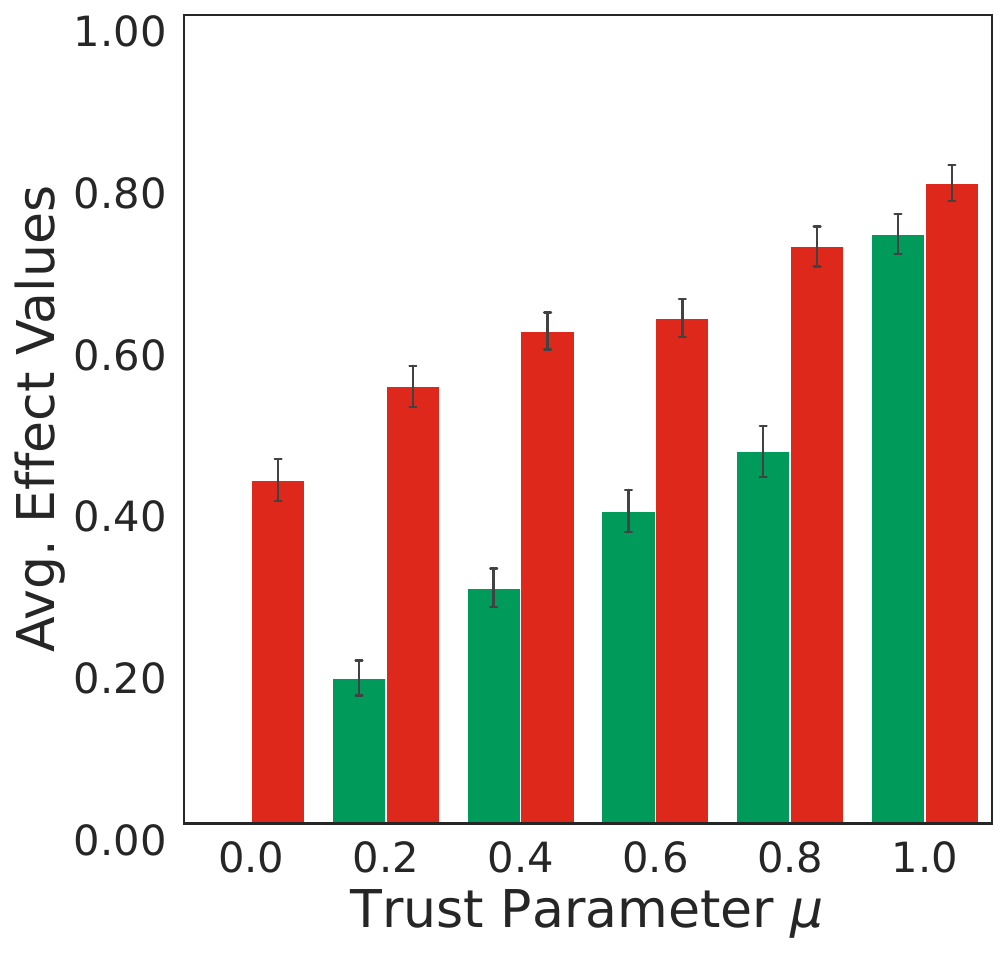}
         \caption{Sepsis: Through AI Ord3}
         \label{plot: sepsis_orderings_ai2}
     \end{subfigure}
     \hfill
     \begin{subfigure}[b]{0.24\textwidth}
         \centering
         \includegraphics[width=\textwidth]{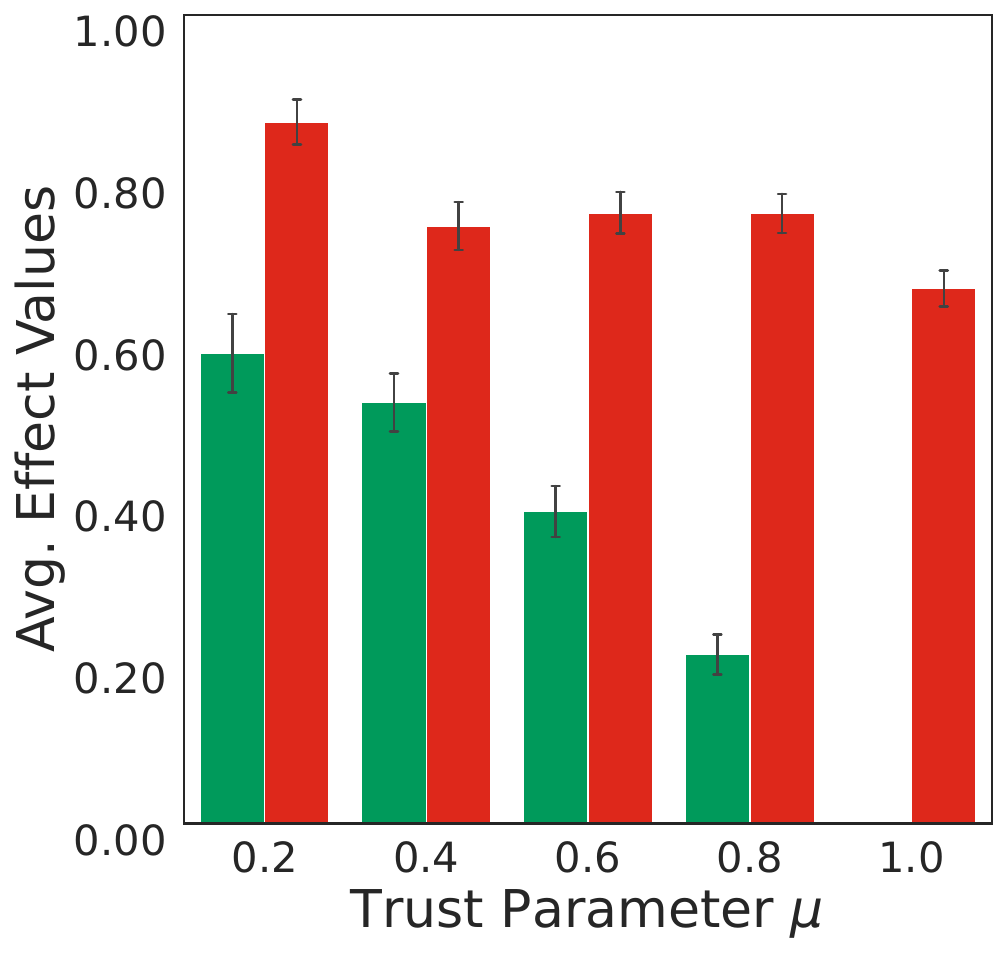}
         \caption{Sepsis: Through CL Ord4}
         \label{plot: sepsis_orderings_cl3}
     \end{subfigure}
     \hfill
     \begin{subfigure}[b]{0.24\textwidth}
         \centering
         \includegraphics[width=\textwidth]{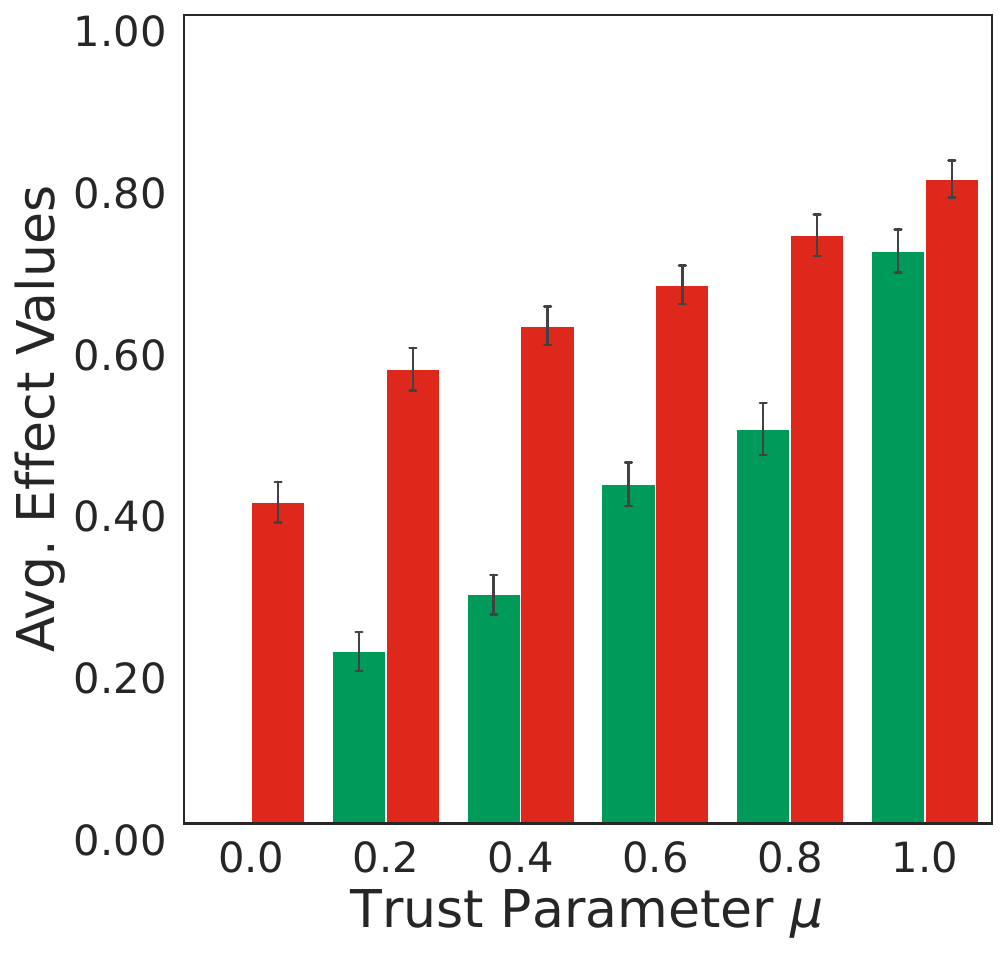}
         \caption{Sepsis: Through AI Ord4}
         \label{plot: sepsis_orderings_ai3}
     \end{subfigure}
      \\
     \begin{subfigure}[b]{0.24\textwidth}
         \centering
         \includegraphics[width=\textwidth]{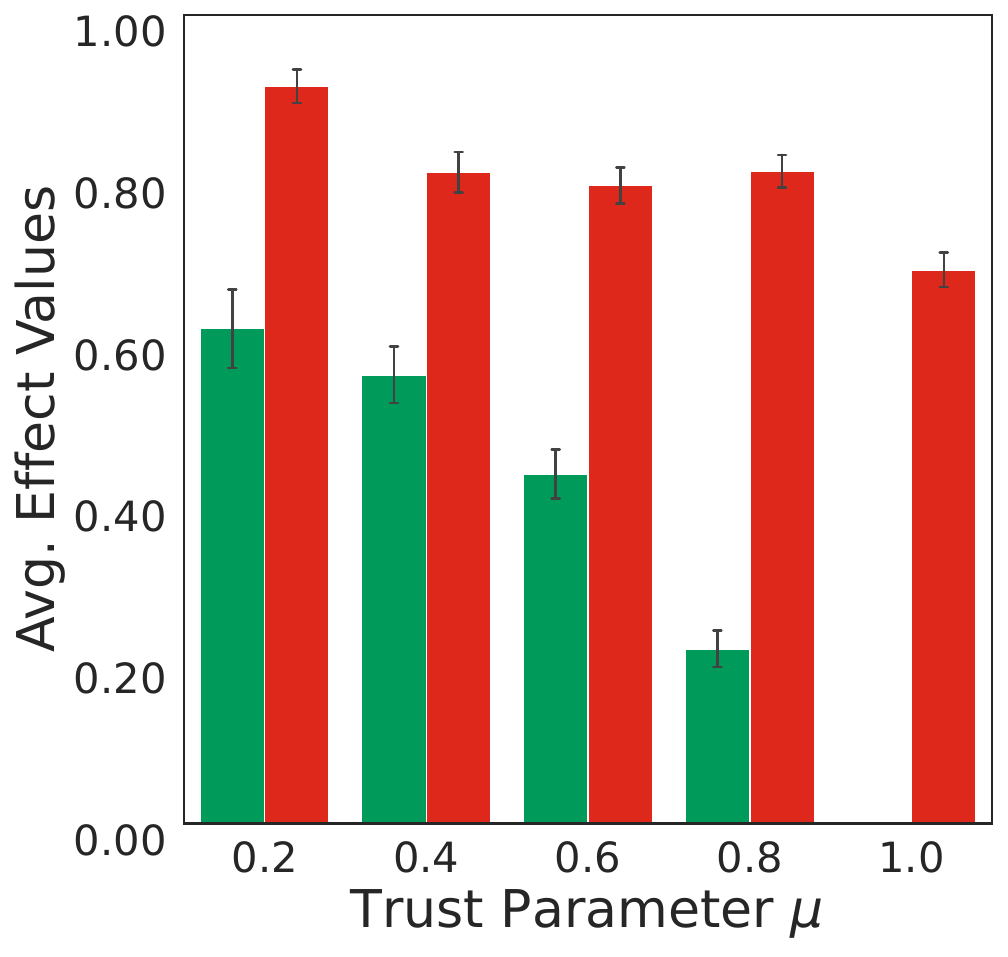}
         \caption{Sepsis: Through CL Ord5}
         \label{plot: sepsis_orderings_cl4}
     \end{subfigure}
     \begin{subfigure}[b]{0.24\textwidth}
         \centering
         \includegraphics[width=\textwidth]{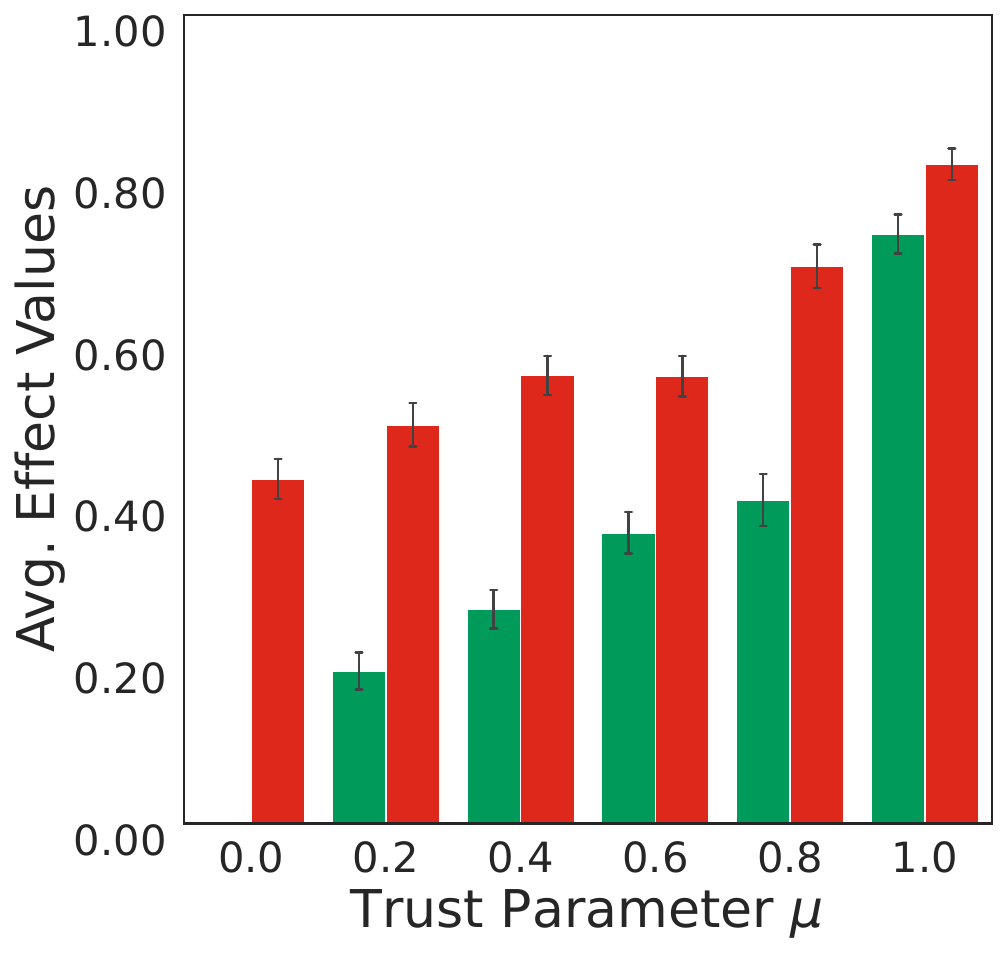}
         \caption{Sepsis: Through AI Ord5}
         \label{plot: sepsis_orderings_ai4}
     \end{subfigure}
    \caption{
    Plots in this figure show average values of cf-ASE and cf-PSE when computing effects that propagate through the clinician (CL) and AI, respectively, while varying trust parameter $\mu$. Results are shown for $5$ randomly selected total ordering sets (Ord1-Ord5).
    }
\end{figure*}
\section{Additional Related Work}\label{app.related_work} 

The results of Section \ref{sec.id} are in general related to works that study the identifiability of counterfactual effects. For instance, there are results on the identification of the \textit{effect of treatment on the treated} \cite{heckman1991randomization, shpitser2009effects} and the identification of counterfactual path-specific effects \cite{shpitser2018identification}. Moreover, \cite{shpitser2007counterfactuals} and \cite{correa2021nested} study the identification of arbitrary non-nested and nested counterfactuals, respectively, from observational and interventional data. Meanwhile, \cite{zhang2018fairness} offers insights into the identifiability of counterfactual direct, indirect and spurious effects across three canonical model settings.

In a broader sense, our work is also related to research papers that study repeated human-AI interactions involving both stationary \cite{ghosh2020towards, dimitrakakis2017multi} and non-stationary \cite{radanovic2019learning, nikolaidis2017game} human policies. We believe that investigating the implications of non-stationary policies on our approach represents an intriguing research direction.

\section{Additional Theoretical Background}\label{app.exp_background}

In this section, we provide additional background and notation needed for the proofs of our theoretical results.

\textbf{Additional background on graphs.}
With $\mathbf{E}_-^{V^i}(G)$ or $\mathbf{E}_-^i(G)$ for short, we denote the set of incoming edges of node $V^i$ in $G$.
Furthermore, we use standard graph notation for relationships such as ancestors, children and descendants, e.g., $An^i(G)$ denotes the ancestor set of node $V^i$ in $G$.
By convention, we also assume that the ancestor and descendant sets of a node are inclusive, i.e., $V^i \in An^i(G)$ and $V^i \in De^i(G)$.
We call graph $g$ a \textit{node-subgraph} of $G$ if $\mathbf{V}(g) \subseteq \mathbf{V}(G)$ and $\mathbf{E}(g)$ contains all of the edges in $\mathbf{E}(G)$ connecting pairs of vertices in $\mathbf{V}(g)$.

\textbf{Counterfactual formulas.}
We refer to realizations of a counterfactual variable, e.g., $Y_{\mathbf{w}} = y$ or simply $y_{\mathbf{w}}$, as \textit{atomic (counterfactual) formulas}. Moreover, we call \textit{counterfactual formula} any conjunction of atomic formulas, e.g., $y_{\mathbf{w}} \land y'_{\mathbf{w}'}$.
Let $\alpha$ now be a counterfactual formula defined in an SCM $M$, we denote with $\mathbf{V}(\alpha)$ the set of variables in $\mathbf{V}$ that correspond to an atomic formula in $\alpha$, e.g., $\mathbf{V}(y_{\mathbf{w}} \land y'_{\mathbf{w}'}) = \{Y\}$.

\textbf{Building tools.} Our proofs build on the do-calculus rules \cite{pearl2009causality} [Sec.~3.4], the exclusion and independence restrictions rules of SCMs \cite{pearl2009causality} [p.~232], and two axioms of structural counterfactuals: composition and consistency \cite{pearl2009causality} [Sec.~7.3.1].

\textbf{Additional SCM notation.} For some SCM $M$, let $\mathbf{V}' \subseteq \mathbf{V}$ and $\mathbf{u} \sim P(\mathbf{u})$. We denote with $\mathbf{V}'(\mathbf{u})$ the values of $\mathbf{V}'$ in $M$ given $\mathbf{u}$.

\textbf{Monotonicity \cite{pearl1999probabilities}}.
Finally, we present here for completeness the \textit{monotonicity} property for binary SCMs.
\begin{definition}[Monotonicity]
\label{def.mon}
Given a SCM $M$ consisting of two binary variables $X$ and $Y$, we say that $Y$ is monotonic relative to $X$ in $M$ if for any $x_1, x_2 \in \{0, 1\}$ s.t.~$\mathds{E}[Y | x_1] \leq \mathds{E}[Y | x_2]$ it holds that $P(Y_{x_1} = 1 \land Y_{x_2} = 0) = 0$.
\end{definition}

\section{Remark on Lemma \ref{lem.pns_nm} and Counterfactual Identification}\label{app.ctf_remark}

\begin{remark}
    Let $M$ be an SCM with causal graph $G$. A \textit{counterfactual factor} \cite{correa2021nested} is a distribution $P(\alpha)$, s.t.~$\alpha$ is a counterfactual formula in $M$ which consists only of atomic formulas of the form $z_{pa^Z}$, where $Z \in \mathbf{V}$, $z \in \mbox{dom}\{Z\}$ and $pa^Z \in \mbox{dom}\{Pa^Z(G)\}$.
    Under the exogeneity assumption, counterfactual factor $\alpha$ can be expressed as the product of probability distributions of the form $P(\land_{i \in \{1, ..., k\}} x^i_{pa^i})$, where $x^1, ..., x^k \in \mbox{dom}\{X\}$, $pa^1, ..., pa^k \in \mbox{dom}\{Pa^X(G)\}$ and $X \in \mathbf{V}(\alpha)$. The latter statement is licensed by the independence restrictions rule, and agrees with Theorem $3$ of \cite{correa2021nested}.
    Furthermore, according to Theorem $2$ (resp. Lemma $3$) of \cite{correa2021nested} any counterfactual (resp. conditional counterfactual) probability can be factorized as counterfactual factors. Therefore, it follows that Lemma \ref{lem.pns_nm} can be utilized for deriving the identification formula of any counterfactual or conditional counterfactual probability under the exogeneity and noise monotonicity assumptions.
\end{remark}

\section{Additional Information on Fixed Path-Specific Effects}\label{app.fpse}

In this section, we define  \textit{counterfactual fixed path-specific effects} (cf-FPSE), i.e., the counterfactual counterpart of FPSE from Section \ref{sec.connections}. Additionally, we show that the property of noise monotonicity suffices to render cf-FPSE identifiable.
\begin{definition}[cf-FPSE]
\label{def.cf-fpse}
Given an SCM $M$ with causal graph $G$, a realization $\mathbf{v}$ of its observable variables $\mathbf{V}$, and two edge-subgraphs of $G$, \textit{effect subgraph} $g$ and \textit{reference subgraph} $g^*$ such that $\mathbf{E}(g) \cap \mathbf{E}(g^*) = \emptyset$, the counterfactual fixed $g$-specific effect of intervention $do(X := x)$ on $Y = y$, relative to reference subgraph $g^*$, is defined as
\begin{align*}
    \mathrm{cf}\text{-}\mathrm{FPSE}^{g, g^*}_{x}(y|\mathbf{v})_M = P(y_{\mathbf{v}(X)}|\mathbf{v};M)_{M_q} - P(y|\mathbf{v})_M,
\end{align*}
where $q$ is the edge-subgraph of $G$ with 
$\mathbf{E}(q) = \mathbf{E}(G) \backslash \{\mathbf{E}(g), \mathbf{E}(g^*)\}$, and $M_q = \langle\mathbf{U}, \mathbf{V}, \mathcal{F}_q, P(\mathbf{u})\rangle$ is a modified SCM, which induces $q$ and is formed as follows.
Each function $f^i$ in $\mathcal{F}$ is replaced with a new function $f^i_q$ in $\mathcal{F}_q$, s.t.~for every $pa^i(q) \in \mathrm{dom}\{Pa^i(q)\}$ and $\mathbf{u} \sim P(\mathbf{u})$ it holds $f^i_q(pa^i(q), u^i) = f^i(pa^i(q), pa^i(g)^e, pa^i(g^*)^*, u^i)$. Here, $pa^i(g^*)^*$ (resp. $pa^i(g)^e$) denotes the value of $Pa^i(g^*)_{\mathbf{v}(X)}$ (resp. $Pa^i(g)_x$) in $M$ for noise $\mathbf{u}$.
\end{definition}


The next lemma is similar to Lemma \ref{lemma.fpse_ase} and shows how to express cf-ASE through cf-FPSE.

\begin{lemma}\label{lemma.cf_fpse_ase}
Let $M$ be an MMDP-SCM with causal graph $G$, $\mathbf{N}$ be a non-empty subset of agents in $M$, and $\tau$ be a trajectory of $M$. It holds that
\begin{align*}
    \mathrm{cf}\text{-}\mathrm{ASE}^{\mathbf{N}}_{a_{i,t}}(y|\tau)_M = \mathrm{cf}\text{-}\mathrm{FPSE}^{g, g^*}_{a_{i,t}}(y|\tau)_M,
\end{align*}
where $g$ is the edge-subgraph of $G$ with $\mathbf{E}(g) = \bigcup_{A_{i',t'}: i' \in \mathbf{N}, t' > t} \mathbf{E}^{A_{i',t'}}_+(G)$ and 
$g^*$ is the edge-subgraph of $G$ with
$\mathbf{E}(g^*) = \bigcup_{A_{i',t'}: i' \notin \mathbf{N}, t' > t} \mathbf{E}^{A_{i',t'}}_+(G)$.
\end{lemma}

\begin{proof}

By Definitions \ref{def.cf-ase} and \ref{def.cf-fpse}, it follows that it suffices to show that $P(y_{\tau(A_{i,t})}|\tau;M)_{M^{do(I)}} = P(y_{\tau(A_{i,t})}|\tau;M)_{M_q}$.

~\\
Let $t_Y$ denote the time-step of variable $Y$. First, we show that for any noise $\mathbf{u} \sim P(\mathbf{u} | \tau)$ the value of $S_{t_Y}(\mathbf{u})$ (or equivalently $Y(\mathbf{u})$) is the same in $M_q$ as in $M^{do(I)}$. To do so, we use induction in the number of time-steps $t'$.

~\\
\textbf{Base case:} 
By construction of $M_q$, the value of $S_{t+1}(\mathbf{u})$ is the same in $M_q$ as in $M$. Furthermore, note that the set of interventions $I$ does not include interventions to variable $S_{t+1}$ or to any of its ancestors. As a result, the value of $S_{t+1}(\mathbf{u})$ is the same in $M^{do(I)}$ as in $M$. We conclude then that the value of $S_{t+1}(\mathbf{u})$ is the same in $M_q$ as in $M^{do(I)}$.

~\\
\textbf{Induction hypothesis:} 
We make the hypothesis that for some $t' \in [t+1, t_Y)$ the value of $S_{t'}(\mathbf{u})$ is the same in $M_q$ as in $M^{do(I)}$.

~\\
\textbf{Induction step:} 
Under the induction hypothesis, we want to show that the value of $S_{t'+1}(\mathbf{u})$ is the same in $M_q$ as in $M^{do(I)}$. 
For convenience, we slightly abuse our notation and denote as $s^q_{t'+1}$ the value of $S_{t'+1}(\mathbf{u})$ in $M_q$, and as $s^I_{t'+1}$ the value of $S_{t'+1}(\mathbf{u})$ in $M^{do(I)}$. 
From the hypothesis, it holds that the value $s_{t'}$ of $S_{t'}(\mathbf{u})$ is the same in $M^{do(I)}$ and $M_q$.

First, we consider the value of $S_{t'+1}(\mathbf{u})$ in $M^{do(I)}$,
\begin{align}\label{eq.fpse_I}
   s^I_{t'+1} = f^{S_{t'+1}}(s_{t'}, \{A_{i',t', a_{i,t}}(\mathbf{u})\}_{i' \in \mathbf{N}}, \{\tau(A_{i',t'})\}_{i' \notin \mathbf{N}}, u^{S_{t'+1}}). 
\end{align}
 
Next, we consider the value of $S_{t'+1}(\mathbf{u})$ in $M_q$,
$$
s^q_{t'+1} = f^{S_{t'+1}}_q(pa^{S_{t'+1}}(q), u^{S_{t'+1}}),
$$ 
where $pa^{S_{t'+1}}(q)$ here denotes the value of $Pa^{S_{t'+1}}(q)$ in $M_q$ given $\mathbf{u}$. 
Note that $Pa^{S_{t'+1}}(q) = S_{t'}$, and since $S_{t'}(\mathbf{u}) = s_{t'}$ in $M_q$, we can rewrite the equation of $s^q_{t'+1}$ as
$$
s^q_{t'+1} = f^{S_{t'+1}}_q(s_{t'}, u^{S_{t'+1}}).
$$ 
Based on the definition of function $f_q^{S_{t'+1}}$ we have that
\begin{align}\label{eq.fpse_q1}
    s^q_{t'+1} = f^{S_{t'+1}}(s_{t'}, pa^{S_{t'+1}}(g)^e, pa^{S_{t'+1}}(g^*)^*, u^{S_{t'+1}}),
\end{align}
where $pa^{S_{t'+1}}(g^*)^*$ and $pa^{S_{t'+1}}(g)^e$ denote the values of $Pa^{S_{t'+1}}(g^*)_{\tau(A_{i,t})}$ and $Pa^{S_{t'+1}}(g)_{a_{i,t}}$ in $M$ given $\mathbf{u}$, respectively.
Since $\mathbf{u} \sim P(\mathbf{u}|\tau)$ it holds that $Z(\mathbf{u}) = \tau(Z)$ for every $Z 
\in \mathbf{V}$.
By the consistency axiom then
we have that for every $V^j \in Pa^{S_{t'+1}}(g^*)$ it holds $V^j_{\tau(A_{i,t})}(\mathbf{u}) = V^j(\mathbf{u}) = \tau(V^j)$. 
Moreover, from the definition of subgraph $g^*$ it follows that $Pa^{S_{t'+1}}(g^*) = \{A_{i',t'}\}_{i' \notin \mathbf{N}}$. Eq. \eqref{eq.fpse_q1} then can be rewritten as follows
\begin{align}\label{eq.fpse_q2}
    s^q_{t'+1} =  f^{S_{t'+1}}(s_{t'}, pa^{S_{t'+1}}(g)^e,  \{\tau(A_{i',t'})\}_{i' \notin \mathbf{N}}, u^{S_{t'+1}}).
\end{align}
Finally, from the definition of subgraph $g$ it follows that $Pa^{S_{t'+1}}(g) = \{A_{i',t'}\}_{i' \in \mathbf{N}}$. Eq. \eqref{eq.fpse_q2} then can be rewritten as follows
\begin{align}\label{eq.fpse_q}
    s^q_{t'+1} =  f^{S_{t'+1}}(s_{t'}, \{A_{i',t', a_{i,t}}(\mathbf{u})\}_{i' \in \mathbf{N}},  \{\tau(A_{i',t'})\}_{i' \notin \mathbf{N}}, u^{S_{t'+1}}).
\end{align}
From Equations \eqref{eq.fpse_I} and \eqref{eq.fpse_q}, it follows that $s^q_{t'+1} = s^I_{t'+1}$, and hence the induction step is concluded.

~\\
Based on the induction argument above, for every $\mathbf{u} \sim P(\mathbf{u} | \tau)$ it holds that the value of $S_{t_Y}(\mathbf{u})$ or equivalently the value of $Y(\mathbf{u})$ is the same in $M_q$ as in $M^{do(I)}$, and hence
$$
P(y|\mathbf{u})_{M^{do(I)}} = P(y|\mathbf{u})_{M_q}.
$$
Finally, the distribution of noise variables $\mathbf{U}$ is the same in $M$, $M^{do(I)}$ and $M_q$, and thus we can derive the following equations
\begin{align*}
    \int P(y|\mathbf{u})_{M^{do(I)}} \cdot P(\mathbf{u}|\tau;M)_{M^{do(I)}} \,d\mathbf{u} &= \int P(y|\mathbf{u})_{M_q} \cdot P(\mathbf{u}|\tau;M)_{M_q}\,d\mathbf{u} \Rightarrow\\
    P(y|\tau;M)_{M^{do(I)}} &= P(y|\tau;M)_{M_q}\Rightarrow\\
    P(y_{\tau(A_{i,t})}|\tau;M)_{M^{do(I)}} &= P(y_{\tau(A_{i,t})}|\tau;M)_{M_q},
\end{align*}
where the last step follows from the consistency axiom.
    
\end{proof}


The next result states that that noise monotonicity suffices to render cf-FPSE identifiable.

\begin{lemma}\label{lemma.id_cf_fpse}
Let $M$ be an SCM with causal graph $G$, $g$ and $g^*$ be two edge-subgraphs of $G$, such that $\mathbf{E}(g) \cap \mathbf{E}(g^*) = \emptyset$, and $q$ be the edge-subgraph of $G$ with $\mathbf{E}(q) = \mathbf{E}(G) \backslash \{\mathbf{E}(g), \mathbf{E}(g^*)\}$. Let also $\mathbf{v} \in \mbox{dom}\{\mathbf{V}\}$ be a realization of the observable variables in $M$.
The counterfactual fixed $g$-specific effect of intervention $do(X := x)$ on $Y = y$, relative to reference subgraph $g^*$, is identifiable, if for every variable $Z \in De^X(G) \cap An^Y(G)$ there is a total ordering $\leq_Z$ on $\mbox{dom}\{Z\}$ w.r.t. which $Z$ is noise-monotonic in $M$.
\end{lemma}

Appendix \ref{proof.pns_nm} contains the proof of Lemma \ref{lem.pns_nm}. In Appendix \ref{app.supporting_lemmas}, we show two additional results that support the proof of Lemma \ref{lemma.id_cf_fpse}. In Appendix \ref{proof.id_cf_fpse}, we prove Lemma \ref{lemma.id_cf_fpse}.

\subsection{Proof of Lemma \ref{lem.pns_nm}}\label{proof.pns_nm}

\begin{proof}
Note that domain $dom\{X\}$ is finite and discrete, and hence w.l.o.g.~we can assume that $dom\{X\} = \{x^1, ..., x^m\}$, with $m \in \mathds{N}^*$. For simplicity, we additionally assume that domain $dom\{X\}$ is sorted w.r.t. the total ordering $\leq_X$, i.e., $x^i <_X x^j$ for $i, j \in \{1, .., m\}$ such that $i < j$.

For every value $pa^i$ of $Pa^X(G)$, we define a partition of the domain of noise variable $U^X$ into $m$ sets $N^i_{x^1}, ..., N^i_{x^m}$, such that $f^X(pa^i, u^X) = x^j$ iff $u^X \in N^i_{x^j}$.

We next state a number of simple properties that are going to be useful for the rest of this proof. Let $x^j \in dom\{X\}$, $pa^i \in dom\{Pa^X(G)\}$ and $u^X_1, u^X_2 \sim P(U^X)$, then the following statements hold
\begin{align}
    & P(U^X \leq \min(u^X_1, u^X_2)) = \min(P(U^X \leq u^X_1), P(U^X \leq u^X_2)), \label{eq.4_1} \\
    & P(U^X < \max(u^X_1, u^X_2)) = \max(P(U^X < u^X_1), P(U^X < u^X_2)), \label{eq.4_2}\\
    & P(U^X \in N^i_{x^j}) = P(U^X \in \big[\min_u(N^i_{x^j}), \max_u(N^i_{x^j})\big]), 
    \label{eq.1}
\end{align}
\begin{align}\label{eq.5}
    P(U^X \in [u^X_1, u^X_2]) &= \begin{cases}
          1- P(U^X < u^X_1) - P(U^X > u^X_2) & \text{if $u^X_1 \leq u^X_2$}\\
          0 & \text{otherwise}
        \end{cases} \nonumber \\
        & = \max(0, 1- P(U^X < u^X_1) - P(U^X > u^X_2)) \nonumber \\
        & = \max(0, P(U^X \leq u^X_2) - P(U^X < u^X_1)).
\end{align}
Eq. \eqref{eq.1} follows from noise monotonicity. 
Moreover, by definition we have
\begin{align}
    & P(U^X \in N^i_{x^j}) = P(X = x^j | pa^i) = P(X_{pa^i} = x^j),\label{eq.2}\\
    & P(U^X \leq \max_u(N^i_{x^j})) = P(U^X \in N^i_{x^1}) + ... + P(U^X \in N^i_{x^j}) = P(X \leq_X x^j | pa^i), 
    \label{eq.3_1}\\
    & P(U^X < \min_u(N^i_{x^j})) = P(U^X \in N^i_{x^1}) + ... + P(U^X \in N^i_{x^{j-1}}) = P(X <_X x^j | pa^i)\label{eq.3_2}.
\end{align}
The second step in Eq. \eqref{eq.2} is licensed by the exogeneity assumption, while the first steps in Eq. \eqref{eq.3_1} and \eqref{eq.3_2} hold by the noise monotonicity assumption.

Next, we derive the equation of the lemma by using the above mentioned properties:
\begin{align*}
    P(\land_{i \in \{1, ..., k\}} x^i_{pa^i}) &= P(U^X \in \cap_{i \in \{1, ..., k\}} N^i_{x^i}) \\
    &= P\big(U^X \in \cap_{i \in \{1, ..., k\}} \big[\min_u(N^i_{x^i}), \max_u(N^i_{x^i})\big] \big) \\
    &= P\big(U^X \in \big[\max_{i \in \{1, ..., k\}}(\min_u(N^i_{x^i})), \min_{i \in \{1, ..., k\}}(\max_u(N^i_{x^i}))\big]\big) \\
    &= \max\bigg(0, P\big(U^X \leq \min_{i \in \{1, ..., k\}}(\max_u(N^i_{x^i}))\big) - P\big(U^X < \max_{i \in \{1, ..., k\}}(\min_u(N^i_{x^i}))\big)\bigg) \\
    &= \max\bigg(0, \min_{i \in \{1, ..., k\}} \big(P(U^X \leq \max_u(N^i_{x^i}))\big) - \max_{i \in \{1, ..., k\}} \big(P(U^X < \min_u(N^i_{x^i})\big)\bigg) \\
    &= \max\bigg(0, \min_{i \in \{1, ..., k\}} \big(P(X \leq_X x^i | pa^i)\big) - \max_{i \in \{1, ..., k\}}\big(P(X <_X x^i | pa^i)\big) \bigg).
\end{align*}
The first step follows from Eq. \eqref{eq.2}. The second step follows from Eq. \eqref{eq.1}. 
The fourth step is licensed by Eq. \eqref{eq.5}, and the fifth step by Eq. \eqref{eq.4_1} and \eqref{eq.4_2}. Finally, the sixth step holds by Eq. \eqref{eq.3_1} and \eqref{eq.3_2}.
    
\end{proof}

\subsection{Supporting Lemmas for the Proof of Lemma \ref{lemma.id_cf_fpse}}\label{app.supporting_lemmas}

\begin{lemma}\label{lemma.comp}
Let $M$ be an SCM with causal graph $G$, $g$ be an edge-subgraph of $G$, and $\beta$ be a counterfactual formula in $M_{g}$.
For any $X \in \mathbf{V}$, $V^i \in De^X(g) \backslash X$ and $\mathbf{W} \subseteq \mathbf{V}$ it holds that 
\begin{align}\label{eq.comp}
P(\beta \land v^i_x | \mathbf{w}; M)_{M_g} = \sum_{pa^i(g) \in dom\{Pa^i(g)_x\}} P(\beta \land v^i_{pa^i(g)} \land pa^i(g)_x | \mathbf{w}; M)_{M_g},
\end{align}
where $dom\{Pa^i(g)_x\}$ is equal to $dom\{Pa^i(g)\}$ if $X \notin Pa^i(g)$, and equal to the subset of $dom\{Pa^i(g)\}$ for which the value of $X$ is always $x$, otherwise.
\end{lemma}

\begin{proof}
\begin{align*}
    P(\beta \land v^i_x | \mathbf{w}; M)_{M_g} &= \sum_{pa^i(g) \in dom\{Pa^i(g)_x\}} P(\beta \land v^i_{x, pa^i(g)} \land pa^i(g)_x | \mathbf{w}; M)_{M_g}\\
    &= \sum_{pa^i(g) \in dom\{Pa^i(g)_x\})} P(\beta \land v^i_{pa^i(g)} \land pa^i(g)_x | \mathbf{w}; M)_{M_g}.
\end{align*}
The first step follows from the composition axiom. If $X \in Pa^i(g)$, then the second step follows from the fact that $x \land pa^i(g) = pa^i(g)$, which is implied by the definition of $dom\{Pa^i(g)_x\}$. Otherwise, the second step follows from the exclusion restrictions rule.

\end{proof}

We next introduce Algorithm \ref{alg.unroll}, which iteratively utilizes Lemma \ref{lemma.comp} to fully unroll atomic formula $y_x$. 
Lemma \ref{lemma.supp_algorithm} characterizes the output of the algorithm.

\begin{algorithm}
\caption{Iterative Process of Unrolling $y_x$}\label{alg.unroll}
\textbf{Input}: $g$, $Y$, $y$, $X$, $x$, $\pi_g$\\
\textbf{Output}: counterfactual formula $\alpha$
\begin{algorithmic}[1]
\STATE $n \gets 1$ \COMMENT{Number of iterations}
\STATE $D^1 \gets An^Y(g) \backslash \{X\}$
\STATE $\alpha^1 \gets y_x$
\WHILE{$D^n \neq \emptyset$}
    \STATE $V^n \gets \mathrm{argmax}_{V^i \in D^n}\{\pi_g(V^i)\}$
    \STATE $\alpha' \gets \hat{\alpha}$ s.t. $\alpha^n = \hat{\alpha} \land v^n_x$
    \STATE $\alpha^{n+1} \gets \alpha' \land v^n_{pa^n(g)} \land pa^n(g)_x$
    \STATE $D^{n + 1} \gets D^n \backslash \{V^n\}$
    \STATE $n \gets n + 1$
\ENDWHILE
\STATE $\alpha \gets \alpha^{|An^Y(g)|}$
\STATE \textbf{return} $\alpha$
\end{algorithmic}
\end{algorithm}

\begin{lemma}\label{lemma.supp_algorithm}
Let $M$ be an SCM with causal graph $G$. For two variables $X, Y \in \mathbf{V}$, let also $g$ be an edge-subgraph of $G$, such that it consists only of directed paths from $X$ to some $Z \in \mathbf{V}$ in $G$, and $Y$ is part of at least one of those paths.
For the input $\langle g, Y, y, X, x, \pi_g \rangle$, where $\pi_g$ is a causal ordering of $g$, it holds that Algorithm \ref{alg.unroll} terminates after $|An^Y(g)| - 1$ iterations, and that its output $\alpha$ consists only of atomic formulas of the form $v^i_{pa^i(g)}$ -- one for each variable $V^i \in An^Y(g) \backslash \{X\}$.
Furthermore, the following equation holds
\begin{align}\label{eq.algo}
    P(\beta \land y_x | \mathbf{w}; M)_{M_g} = \sum_{\{v^i | V^i \in An^Y(g) \backslash \{X, Y\}\}} P(\beta \land \alpha | \mathbf{w}; M)_{M_g},
\end{align}
where $\beta$ is a counterfactual formula in $M_g$, and $\mathbf{w}$ is a realization of $\mathbf{W} \subseteq \mathbf{V}$ in $M$.
\end{lemma}

\begin{proof}
We first show that at every iteration $n$ of the algorithm, the atomic formula $v^n_x$ is indeed part of $\alpha^n$, and subsequently that the algorithm does terminate without errors after $|An^Y(g)| - 1$ iterations. To do so, we use induction in the number of iterations $n$.

\textbf{Base case}:  For $n = 1$, we have that $V^1 = Y$ and $\alpha^1 = y_x$.
Therefore, the statement holds for the base case $n=1$.

\textbf{Induction hypothesis}: For some $k \in (1, |An^Y(g)|]$, we make the hypothesis that the statement holds for every $n < k$, i.e., we assume that the atomic formula $v^n_x$ is part of $\alpha^n$ for every $n \in [1, k)$.

\textbf{Induction step}: Under this hypothesis, we want to prove that $v^k_x$ is part of $\alpha^k$. 
Because $\pi_g$ in Step $5$ is a causal ordering, it holds that all variables in $Ch^k(g) \cap An^Y(g)$ must be selected by the algorithm before $V^k$, i.e., prior to iteration $k$.
Let, for example, $m$ be the first iteration at which a child of $V^k$ is selected by the algorithm. Because of the induction hypothesis, we know that $v^m_x$ is part of $\alpha^m$. It follows then that $v^k_x$ becomes part of $\alpha^{m+1}$ in Step $7$ of iteration $m$. Since $v^k_x$ cannot be replaced in the counterfactual formula before $V^k$ is selected, i.e., before iteration $k$, we can conclude then that $v^k_x$ is indeed part of $\alpha^k$.

We next show that at every iteration $n$ of the algorithm, the counterfactual formula $\alpha^n$ consists of (a) an atomic formula $v^i_{pa^i(g)}$ for each variable $V^i \in D^1 \backslash D^n$, and (b) an atomic formula $v^i_x$ for some of the variables $V^i \in D^n$. To do so, we use induction in the number of iterations $n$.

(\textbf{Base case}) For $n=1$, we have that $\alpha^1 = y_x$. Therefore, the statement holds for the base case $n=1$.

(\textbf{Induction hypothesis}) For some $k \in [1, |An^Y(g)|)$, we make the hypothesis that the statement holds for $n=k$.

(\textbf{Induction step}) Under this hypothesis, we want to show that the statement also holds for $n = k + 1$. From Step $7$ of the algorithm, we have that $\alpha^{k+1} = \alpha' \land v^k_{pa^k(g)} \land pa^k(g)_x$, where $\alpha'$ is such that $\alpha^k = \alpha' \land v^k_x$. From the induction hypothesis and equation $\alpha^k = \alpha' \land v^k_x$, it follows that the counterfactual formula $\alpha' \land v^k_{pa^k(g)}$ consists of (a) an atomic formula $v^i_{pa^i(g)}$ for each variable $V^i \in \{V^k\} \cup D^1 \backslash D^k$, or equivalently $V^i \in D^1 \backslash D^{k+1}$, and (b) an atomic formula $v^i_x$ for some of the variables $V^i \in D^k \backslash \{V^k\}$, or equivalently $V^i \in D^{k+1}$. 
Furthermore,  because $\pi_g$ in Step $5$ is a causal ordering, it has to hold that none of the variables in $Pa^k(g)$ can be selected by the algorithm before $V^k$, i.e., prior to iteration $k$. If $X \notin Pa^k(g)$, then this implies that $Pa^k(g) \subseteq D^{k+1}$, and hence the statement holds true for $n = k+1$. If now $X \in Pa^k(g)$, then it follows that $Pa^k(g) \backslash \{X\} \subseteq D^{k+1}$, and that $x_x$ is part of $\alpha^{k+1}$. Note, however, that $x_x$ does not influence the truth value of $\alpha^{k+1}$, since its own truth value is always $1$. Hence, $x_x$ can be removed from $\alpha^{k+1}$ without affecting its value. We conclude then that the statement holds true for $n = k+1$ also when $X \in Pa^k(g)$.

Note that $\alpha = \alpha^{|An^Y(g)|}$ (Step $11$), and that $D^{|An^Y(g)|} = \emptyset$. It follows from the last induction argument then that outcome $\alpha$ indeed consists of an atomic formula $v^i_{pa^i(g)}$ for each variable $V^i \in An^Y(g) \backslash \{X\}$, and nothing else.

Finally, Eq. \eqref{eq.algo} simply follows from repeatedly applying Lemma \ref{lemma.comp}.
    
\end{proof}

\subsection{Proof of Lemma \ref{lemma.id_cf_fpse}}\label{proof.id_cf_fpse}

\begin{proof}
Without loss of generality we assume that every edge in $\mathbf{E}(g)$ and $\mathbf{E}(g^*)$ is part of some directed path from $X$ to $Y$ in $G$.
Furthermore, we will use $x^*$ to denote the factual value $\mathbf{v}(X)$.
Thus, according to Definition \ref{def.cf-fpse} we can write the counterfactual fixed $g$-specific effect as
\begin{align}\label{eq.fpse_probs}
    \mathrm{cf}\text{-}\mathrm{FPSE}^{g, g^*}_{x}(y|\mathbf{v})_M &= P(y_{x^*}|\mathbf{v}; M)_{M_q} - P(y|\mathbf{v})_M.
\end{align}
To identify the counterfactual fixed $g$-specific effect, we need to identify $P(y_{x^*}|\mathbf{v}; M)_{M_q}$ and $P(y|\mathbf{v})_M$. Note that the term $P(y|\mathbf{v})_M$ can be trivially evaluated by comparing $y$ with $\mathbf{v}(Y)$; it is $1$ if they are equal, and $0$ otherwise.

We now focus on identifying $P(y_{x^*}|\mathbf{v}; M)_{M_q}$. Our first step is to express $P(y_{x^*}|\mathbf{v}; M)_{M_q}$ in terms of probabilities defined in $M$ instead of $M_q$. On our second step, we will show that the expression we derived in the first step is identifiable from the observational distribution of $M$.
 
\textbf{Step }$\mathbf{1}$\textbf{:} 
We denote with $q^d$ and $G^d$ the node-subgraphs of $q$ and $G$, which contain only the nodes that are parts of directed paths from $X$ to $Y$ in $q$ and $G$, respectively. Formally, $\mathbf{V}(G^d) = De^X(G) \cap An^Y(G)$ and $\mathbf{V}(q^d) = De^X(q) \cap An^Y(q)$. 

We next define the modified SCM $M_{G^d} = \langle \mathbf{U}, \mathbf{V}(G^d), \mathcal{F}_{G^d}, P(\mathbf{u})\rangle$ which induces $G^d$ as follows. For each variable $V^i \in \mathbf{V}(G^d)$, function $f^i$ in $\mathcal{F}$ is replaced with a new function $f^i_{G^d}$ in $\mathcal{F}_{G^d}$, such that for every $pa^i(G^d) \in dom\{Pa^i(G^d)\}$ and $u^i \sim P(u^i)$ it holds that $f^i_{G^d}(pa^i(G^d), u^i) = f^i(pa^i(G^d), \mathbf{v}(Pa^i(\overline{G^d})), u^i)$.
Similarly, we define the modified SCM $M_{q^d} = \langle \mathbf{U}, \mathbf{V}(G^d), \mathcal{F}_{q^d}, P(\mathbf{u})\rangle$ which induces $q^d$ as follows. For each $V^i \in \mathbf{V}(q^d)$, function $f_q^i$ in $\mathcal{F}_q$ is replaced with a new function $f^i_{q^d}$ in $\mathcal{F}_{q^d}$, such that for every $pa^i(q^d) \in dom\{Pa^i(q^d)\}$ and $u^i \sim P(u^i)$ it holds that $f^i_{q^d}(pa^i(q^d), u^i) = f_q^i(pa^i(q^d), \mathbf{v}(Pa^i(\overline{G^d})), u^i)$. Note that $Pa^i(\overline{G^d}) \cap Pa^i(g) = \emptyset$ and $Pa^i(\overline{G^d}) \cap Pa^i(g^*) = \emptyset$, which follows from the fact that we assume each edge in $\mathbf{E}(g)$ and $\mathbf{E}(g^*)$ to be part of some directed path from $X$ to $Y$ in $G$.
Furthermore, by the definitions of $M_q$ and $M_{q^d}$ it also holds that $Pa^i(q^d) \cup Pa^i(\overline{G^d}) \cup Pa^i(g) \cap Pa^i(g^*) = Pa^i(G)$.

Let set $\mathbf{S} = \cup_{V^i \in \mathbf{V}(q^d) \backslash \{X\}} Pa^i(\overline{G^d})$.
We can express $P(y_{x^*}|\mathbf{v}; M)_{M_q}$ as follows
\begin{align}\label{eq.q^d}
    P(y_{x^*}|\mathbf{v}; M)_{M_q} &=  P(y_{x^*, \mathbf{v}(\mathbf{S})} |\mathbf{v}; M)_{M_q} \nonumber \\
    &= P(y_{x^*} |\mathbf{v}; M)_{M_{q^d}},
\end{align}
where the first step is licensed by exogeneity. Therefore, from Eq. \eqref{eq.q^d} we have that if $P(y_{x^*} | \mathbf{v}; M)_{M_{q^d}}$ is identifiable in $M$ then so is $P(y_{x^*}|\mathbf{v}; M)_{M_q}$.

\textbf{Unroll $y_{x^*}$ in $M_{q^d}$.}
Let $\alpha$ be the output of Algorithm \ref{alg.unroll} for input $\langle q^d, Y, y, X, x^*, \pi_{q^d} \rangle$, where $\pi_{q^d}$ is a causal ordering of $q^d$.
According to Lemma \ref{lemma.supp_algorithm}, it holds that
\begin{align}\label{eq.fpse_unroll}
    P(y_{x^*} | \mathbf{v}; M)_{M_{q^d}} = \sum_{\{v^{i,d} | V^i \in \mathbf{V}(q^d) \backslash \{X, Y\}\}} P(\alpha | \mathbf{v}; M)_{M_{q^d}},
\end{align}
where counterfactual formula $a$ consists only of atomic formulas of the form $v^{i,d}_{pa^i(q^d)^d}$ -- one for each variable $V^i \in \mathbf{V}(q^d) \backslash \{X\}$. Here, we use $v^{i,d}$ to denote the value of $V^i_{x^*}$ in $M_{q^d}$.

\textbf{Express $\alpha$ in $M_{G^d}$.} 
Note that the counterfactual probabilities on the r.h.s. of Eq. \eqref{eq.fpse_unroll} are still defined in $M_{q^d}$ instead of $M$. 
To redefine these probabilities in $M$, we will need to modify $\alpha$ such that all implicit references to functions $f^i_{q^d}$ in $M_{q^d}$ are removed, while its truth value is preserved.
Towards that goal, we will first express $\alpha$ in the modified model $M_{G^d}$.
In particular, we begin by replacing each atomic formula $v^{i,d}_{pa^i(q^d)^d}$ in $\alpha$ with the counterfactual formula $v^{i,d}_{pa^i(q^d)^d, pa^i(g)^e, pa^i(g^*)^*} \land pa^i(g)^e_x \land pa^i(g^*)^*_{x^*}$, where $v^{i,e}$ here denotes the value of $V^i_x$ in $M_{G^d}$, and $v^{i,*}$ denotes the value of $V^i_{x^*}$ in $M_{G^d}$. Note also that under this new term, $v^{i,d}$ now denotes the value of $V^i_{x^*, pa^i(g)^e, pa^i(g^*)^*}$ in $M_{G^d}$.
    
Next, we group the terms of the resulting counterfactual formula as follows:
\begin{itemize}
    \item We denote with $\alpha_1$ the conjunction of all atomic formulas $v^{i,d}_{pa^i(q^d)^d, pa^i(g)^e, pa^i(g^*)^*}$. It holds that $\mathbf{V}(\alpha_1) = \mathbf{V}(q^d) \backslash \{X\}$.
    \item We denote with $\alpha_2$ the conjunction of all atomic formulas of the form $v^{i,e}_x$. It holds that $\mathbf{V}(\alpha_2) = \cup_{V^i \in \mathbf{V}(q^d) \backslash \{X\}} Pa^i(g) \backslash \{X\}$.\footnote{We consider $X \notin \mathbf{V}(\alpha_2)$, because the atomic formula $x_x$ has always truth value equal to $1$, and hence it can be trivially removed from any counterfactual formula.}
    \item We denote with $\alpha_3$ the conjunction of all atomic formulas of the form $v^{i,*}_{x^*}$. It holds that $\mathbf{V}(\alpha_3) = \cup_{V^i \in \mathbf{V}(q^d) \backslash \{X\}} Pa^i(g^*) \backslash \{X\}$.\footnote{We consider $X \notin \mathbf{V}(\alpha_3)$, because the atomic formula $x^*_{x^*}$  has always truth value equal to $1$, and hence it can be trivially removed from any counterfactual formula.}
\end{itemize}
By replacing $\alpha$ with $\alpha_1 \land \alpha_2 \land \alpha_3$
in Eq. \eqref{eq.fpse_unroll}, we then have that
\begin{align}\label{eq.fpse_M_G^d}
    P(y_{x^*} | \mathbf{v}; M)_{M_{q^d}} &= \sum_{\substack{
        \{v^{i,d} | V^i \in \mathbf{V}(\alpha_1) \backslash \{Y\}\},\\
        \{v^{i,e} | V^i \in \mathbf{V}(\alpha_2)\},\\
        \{v^{i,*} | V^i \in \mathbf{V}(\alpha_3)\}
        }
    } 
    P(\alpha_1 \land \alpha_2 \land \alpha_3 | \mathbf{v}; M)_{M_{G^d}}.
\end{align}

\textbf{Unroll $\alpha_2$ in $M_{G^d}$.} Let $\alpha^i_2$ be the output of Algorithm \ref{alg.unroll} for input $\langle G^d, V^i, v^{i,e}, X, x, \pi_{G^d} \rangle$, where $\pi_{G^d}$ is a causal ordering of $G^d$ and $V^i \in \mathbf{V}(\alpha_2)$. We denote with $\alpha'_2$ the conjunction of all such counterfactual formulas.
According to Lemma \ref{lemma.supp_algorithm}, we can then rewrite Eq. \eqref{eq.fpse_M_G^d} as follows
\begin{align}\label{eq.fpse_unroll_a2}
    P(y_{x^*} | \mathbf{v}; M)_{M_{q^d}} &= \sum_{\substack{
        \{v^{i,d} | V^i \in \mathbf{V}(\alpha_1)\backslash \{Y\}\},\\
        \{v^{i,e} | V^i \in \mathbf{V}(\alpha'_2)\},\\
        \{v^{i,*} | V^i \in \mathbf{V}(\alpha_3)\}
        }
    } 
    P(\alpha_1 \land \alpha'_2 \land \alpha_3 | \mathbf{v}; M)_{M_{G^d}},
\end{align}
where counterfactual formula $a'_2$ consists only of atomic formulas of the form $v^{i,d}_{pa^i(G^d)^d}$ -- one for each variable $V^i \in \mathbf{V}(\alpha'_2)$, where $\mathbf{V}(\alpha'_2) = \cup_{V^i \in \mathbf{V}(\alpha_2)} An^i(G^d) \backslash \{X\}$.

\textbf{Express $\alpha_1 \land \alpha'_2 \land \alpha_3$ in $M$.} To remove from $\alpha_1$, $\alpha'_2$ and $\alpha_3$ all implicit references to functions $f^i_{G^d}$ in $M_{G^d}$, while preserving their truth values
\begin{itemize}
    \item We replace each atomic formula $v^{i,d}_{pa^i(q^d)^d, pa^i(g)^e, pa^i(g^*)^*}$ in $\alpha_1$ with the new atomic formula $v^{i,d}_{pa^i(q^d)^d, pa^i(g)^e, pa^i(g^*)^*, \mathbf{v}(Pa^i(\overline{G^d}))}$. We denote with $\beta_1$ the resulting counterfactual formula. Note that it holds $\mathbf{V}(\beta_1) = \mathbf{V}(\alpha_1)$.
    \item We replace each atomic formula $v^{i,e}_{pa^i(G^d)^d}$ in $\alpha'_2$ with the new atomic formula $v^{i,e}_{pa^i(G^d)^d, \mathbf{v}(Pa^i(\overline{G^d}))}$. We denote with $\beta_2$ the resulting counterfactual formula. Note that it holds $\mathbf{V}(\beta_2) = \mathbf{V}(\alpha'_2)$.
    \item We replace each atomic formula $v^{i,*}_{x^*}$ in $\alpha_3$ with the new atomic formula $v^{i,*}_{x^*, \mathbf{v}(Pa^i(\overline{G^d}))}$. We denote with $\beta_3$ the resulting counterfactual formula. Note that it holds $\mathbf{V}(\beta_3) = \mathbf{V}(\alpha_3)$.
\end{itemize}

Note that counterfactual formulas $\beta_1$, $\beta_2$ and $\beta_3$ are all defined under the original model $M$. Therefore, by replacing $\alpha_1 \land \alpha'_2 \land \alpha_3$ with $\beta_1 \land \beta_2 \land \beta_3$ in Eq. \eqref{eq.fpse_unroll_a2}, we have that
\begin{align}\label{eq.fpse_M}
    P(y_{x^*} | \mathbf{v}; M)_{M_{q^d}} &= \sum_{\substack{
        \{v^{i,d} | V^i \in \mathbf{V}(\beta_1)\backslash \{Y\}\},\\
        \{v^{i,e} | V^i \in \mathbf{V}(\beta_2)\},\\
        \{v^{i,*} | V^i \in \mathbf{V}(\beta_3)\}
        }
    } 
    P(\beta_1 \land \beta_2 \land \beta_3 | \mathbf{v})_M.
\end{align}

This concludes Step $1$.

\textbf{Step }$\mathbf{2}$\textbf{:}
By Bayes rule, we can rewrite Eq. \eqref{eq.fpse_M} as
\begin{align}\label{eq.fpse_Bayes}
    P(y_{x^*} | \mathbf{v}; M)_{M_{q^d}} &= \frac{1}{P(\mathbf{v})_M} \cdot
    \sum_{\substack{
        \{v^{i,d} | V^i \in \mathbf{V}(\beta_1)\backslash \{Y\}\},\\
        \{v^{i,e} | V^i \in \mathbf{V}(\beta_2)\},\\
        \{v^{i,*} | V^i \in \mathbf{V}(\beta_3)\}
        }
    } 
    P(\beta_1 \land \beta_2 \land \beta_3 \land \mathbf{v})_M.
\end{align}

By consistency axiom, it holds $X = \mathbf{v}(X) \Rightarrow \mathbf{V}_{\mathbf{v}(X)}  = \mathbf{V}$, and since $\mathbf{v}(X) = x^*$, we have that $\mathbf{V}_{x^*} = \mathbf{V}$. Therefore, we can rewrite Eq. \eqref{eq.fpse_Bayes} as follows
\begin{align}\label{eq.fpse_cons}
    P(y_{x^*} | \mathbf{v}; M)_{M_{q^d}} &= \frac{1}{P(\mathbf{v})_M} \cdot
    \sum_{\substack{
        \{v^{i,d} | V^i \in \mathbf{V}(\beta_1)\backslash \{Y\}\},\\
        \{v^{i,e} | V^i \in \mathbf{V}(\beta_2)\},\\
        \{v^{i,*} | V^i \in \mathbf{V}(\beta_3)\}
        }
    } 
    P(\beta_1 \land \beta_2 \land \beta_3 \land \mathbf{v}_{x^*})_M \nonumber \\
    &= \frac{1}{P(\mathbf{v})_M} \cdot
    \sum_{\substack{
        \{v^{i,d} | V^i \in \mathbf{V}(\beta_1)\backslash \{Y\}\},\\
        \{v^{i,e} | V^i \in \mathbf{V}(\beta_2)\}
        }
    } 
    P(\beta_1 \land \beta_2 \land \mathbf{v}_{x^*})_M.
\end{align}
The second step follows from the fact that if $v^{i,*} \neq \mathbf{v}(V^i)$, for some variable $V^i \in \mathbf{V}(\beta_3)$, then $v^{i,*}_{x^*} \land \mathbf{v}(V^i)_{x^*}$ always evaluates to $0$.

By composition axiom, for some variable $V^i \in \mathbf{V}$ it holds that $Pa^i(G)_{x^*} = \mathbf{v}(Pa^i(G)) \Rightarrow V^i_{\mathbf{v}(Pa^i(G)), x^*} = V^i_{x^*}$. Therefore, we can rewrite Eq. \eqref{eq.fpse_cons} as follows
\begin{align}\label{eq.fpse_comp}
    P(y_{x^*} | \mathbf{v}; M)_{M_{q^d}} &= \frac{1}{P(\mathbf{v})_M} \cdot
    \sum_{\substack{
        \{v^{i,d} | V^i \in \mathbf{V}(\beta_1)\backslash \{Y\}\},\\
        \{v^{i,e} | V^i \in \mathbf{V}(\beta_2)\}
        }
    } 
    P(\beta_1 \land \beta_2 \land \gamma)_M,
\end{align}
where $\gamma$ is the conjunction of all atomic formulas $\mathbf{v}(V^i)_{\mathbf{v}(Pa^i(G)), x^*}$.\footnote{The order in which the composition axiom is applied here can follow any causal ordering on $G$.}

By exclusions restrictions rule, if $X \notin Pa^i(G)$ it holds that $V^i_{\mathbf{v}(Pa^i(G)), x^*} = V^i_{\mathbf{v}(Pa^i(G))}$. Furthermore, if $X \in Pa^i(G)$ it holds $V^i_{\mathbf{v}(Pa^i(G)), x^*} = V^i_{\mathbf{v}(Pa^i(G))}$, since $\mathbf{v}(X) = x^*$. Therefore, we can rewrite Eq. \eqref{eq.fpse_comp} as follows
\begin{align}\label{eq.fpse_excl}
    P(y_{x^*} | \mathbf{v}; M)_{M_{q^d}} &= \frac{1}{P(\mathbf{v})_M} \cdot
    \sum_{\substack{
        \{v^{i,d} | V^i \in \mathbf{V}(\beta_1)\backslash \{Y\}\},\\
        \{v^{i,e} | V^i \in \mathbf{V}(\beta_2)\}
        }
    } 
    P(\beta_1 \land \beta_2 \land \gamma')_M,
\end{align}
where $\gamma'$ is the conjunction of all atomic formulas $\mathbf{v}(V^i)_{\mathbf{v}(Pa^i(G))}$.

Finally, the following equation holds by the independence restrictions rule and the exogeneity assumption.
\begin{align}\label{eq.fpse_ind}
    P(y_{x^*} | \mathbf{v}; M)_{M_{q^d}} &= \frac{1}{P(\mathbf{v})_M} \cdot
    \sum_{\substack{
        \{v^{i,d} | V^i \in \mathbf{V}(\beta_1)\backslash \{Y\}\},\\
        \{v^{i,e} | V^i \in \mathbf{V}(\beta_2)\}
        }
    } \nonumber \\
    &\quad\quad\quad\quad\quad~~
    \prod_{\{i | V^i \in \mathbf{V}(\beta_1) \cup \mathbf{V}(\beta_2)\}} P(v^{i,d}_{pa^i(q^d)^d, pa^i(g)^e, pa^i(g^*)^*, \mathbf{v}(Pa^i(\overline{G^d}))} \land
    \nonumber \\
    &\quad\quad\quad\quad\quad\quad\quad\quad\quad\quad\quad\quad\quad\quad\quad~~
    v^{i,e}_{pa^i(G^d)^d, \mathbf{v}(Pa^i(\overline{G^d}))} \land
    \mathbf{v}(V^i)_{\mathbf{v}(Pa^i(G))})_M \cdot \nonumber \\
    &\quad\quad\quad\quad\quad~~
    \prod_{\{i | V^i \in \mathbf{V}(\beta_1) \backslash \mathbf{V}(\beta_2)\}} P(v^{i,d}_{pa^i(q^d)^d, pa^i(g)^e, pa^i(g^*)^*, \mathbf{v}(Pa^i(\overline{G^d}))} \land
    \nonumber \\
    &\quad\quad\quad\quad\quad\quad\quad\quad\quad\quad\quad\quad\quad\quad\quad~~~
    \mathbf{v}(V^i)_{\mathbf{v}(Pa^i(G))})_M \cdot \nonumber \\
    &\quad\quad\quad\quad\quad~~
    \prod_{\{i | V^i \in \mathbf{V}(\beta_2) \backslash \mathbf{V}(\beta_1)\}} P(v^{i,e}_{pa^i(G^d)^d, \mathbf{v}(Pa^i(\overline{G^d}))} \land \mathbf{v}(V^i)_{\mathbf{v}(Pa^i(G))})_M \cdot \nonumber \\
    &\quad\quad\quad\quad~~
    \prod_{\{i | V^i \in \mathbf{V} \backslash \{\mathbf{V}(\beta_1) \cup \mathbf{V}(\beta_2)\}\}} P(\mathbf{v}(V^i) | \mathbf{v}(Pa^i(G)))_M.
\end{align}

$P(\mathbf{v})_M$ can be computed from the observational distribution of $M$. The same also holds for the conditional probabilities $P(\mathbf{v}(V^i) | \mathbf{v}(Pa^i(G)))_M$ in the last product of the equation.
Regarding the counterfactual probabilities shown on the r.h.s.~of Eq. \eqref{eq.fpse_ind}, according to Lemma \ref{lem.pns_nm}, they are also identifiable under the noise monotonicity assumption. We conclude then that $P(y_{x^*} | \mathbf{v}; M)_{M_{q^d}}$ is indeed identifiable from the observational distribution of $M$.

\end{proof}

\section{Proofs of Theorems \ref{thrm.ase_nonid} and \ref{thrm.id_ase_cf}}\label{proof.ase_id}

First, we prove Theorem \ref{thrm.ase_nonid}.

\begin{proof}

Let $G$ be the causal graph of $M$. First, we are going to prove the result for $Y = S_{t+2}$, i.e., $t_Y = t+2$.
According to Definition \ref{def.ase}, we have that 
\begin{align*}
    \mathrm{ASE}^{\mathbf{N}}_{a_{i,t}, a^*_{i,t}}(y)_M = P(y_{a^*_{i,t}})_{M^{do(I)}} - P(y_{a^*_{i,t}})_M,
\end{align*}
where $I = \{A_{i',t+1} := A_{i',t+1, a^*_{i,t}}\}_{i' \notin \mathbf{N}} \cup \{A_{i',t+1} := A_{i',t+1, a_{i,t}}\}_{i' \in \mathbf{N}}$.

It trivially holds for $Y = S_{t+2}$ that $P(y_{a^*_{i,t}})_{M^{do(I)}} = P(y_{a^*_{i,t}})_{M^{do(I')}}$, where $I' = \{A_{i',t+1} := A_{i',t+1, a_{i,t}}\}_{i' \in \mathbf{N}}$, and hence it follows
\begin{align}\label{eq.proof_ase}
    \mathrm{ASE}^{\mathbf{N}}_{a_{i,t}, a^*_{i,t}}(y)_M &= P(y_{a^*_{i,t}})_{M^{do(I')}} - P(y_{a^*_{i,t}})_M.
\end{align}

Note that the modified SCM $M^{do(I')} = \langle\mathbf{U}, \mathbf{V}, \mathcal{F}_{I'}, P(\mathbf{u})\rangle$ is identical to $M$ apart from the functions corresponding to variables in $\{A_{i',t+1}\}_{i' \in \mathbf{N}}$.
In particular, the function $f_{I'}^{A_{j,t+1}}$ which corresponds to a variable $A_{j, t+1} \in \{A_{i',t+1}\}_{i' \in \mathbf{N}}$ in $M^{do(I')}$ can be defined as
\begin{align}\label{eq.proof_doI}
f_{I'}^{A_{j,t+1}}(Pa^{A_{j,t+1}}(G), U^{A_{j,t+1}}) &= A_{j,t+1, a_{i,t}} \nonumber\\
&= f^{A_{j,t+1}}(Pa^{A_{j,t+1}}(G)_{a_{i,t}}, U^{A_{j,t+1}})\nonumber\\
&= f^{A_{j,t+1}}(S_{t+1, a_{i,t}}, U^{A_{j,t+1}}).
\end{align}

The second step follows from the exclusion restrictions rule. Consider now the path-specific effect
\begin{align*}
    \mathrm{PSE}^g_{a^*_{i,t}, a_{i,t}}(y)_M &= P(y_{a^*_{i,t}})_{M_g} - P(y_{a_{i,t}})_{M_g}\\
    &= P(y_{a^*_{i,t}})_{M_g} - P(y_{a_{i,t}})_M,
\end{align*}
where $g$ is the edge-subgraph of $G$ that does not include the edges from $S_{t+1}$ to nodes in $\{A_{i',t+1}\}_{i' \in \mathbf{N}}$. The second step holds because the quantity $P(y_{a_{i,t}})$ in the modified model $M_g$ corresponds to $P(y_{a_{i,t}})$ in the original model $M$. 

Based on Definition \ref{def.pse}, we have that the modified SCM $M_g = \langle\mathbf{U}, \mathbf{V}, \mathcal{F}_g, P(\mathbf{u})\rangle$ is identical to $M$ apart from the functions corresponding to variables in $\{A_{i',t+1}\}_{i' \in \mathbf{N}}$. In particular, the function $f_g^{A_{j,t+1}}$  which corresponds to a variable $A_{j, t+1} \in \{A_{i',t+1}\}_{i' \in \mathbf{N}}$ in $M_g$ can be defined as
\begin{align}\label{eq.proof_pse}
    f_g^{A_{j,t+1}}(Pa^{A_{j,t+1}}(g), U^{A_{j,t+1}}) &= f^{A_{j,t+1}}(Pa^{A_{j,t+1}}(g), Pa^{A_{j,t+1}}(\overline{g})_{a_{i,t}}, U^{A_{j,t+1}}) \nonumber\\
    &= f^{A_{j,t+1}}(S_{t+1, a_{i,t}}, U^{A_{j,t+1}}).
\end{align}

From Eq. \eqref{eq.proof_doI} and \eqref{eq.proof_pse}, we can infer that the SCMs $M^{do(I')}$ and $M_g$ are equivalent. Therefore, we can rewrite Eq. \eqref{eq.proof_ase} as follows
\begin{align*}
    \mathrm{ASE}^{\mathbf{N}}_{a_{i,t}, a^*_{i,t}}(y)_M &= P(y_{a^*_{i,t}})_{M_g} - P(y_{a^*_{i,t}})_M.
\end{align*}

Note that quantity $P(y_{a^*_{i,t}})_M$ is trivially identifiable (the same holds for $P(y_{a_{i,t}})_M$). Thus, we can restrict our attention to the identifiability of $P(y_{a^*_{i,t}})_{M_g}$, which we can determine using the \textit{recanting witness criterion} \cite{avin2005identifiability}.
 
The subgraph $g$ contains the directed path $A_{i,t} \rightarrow S_{t+1} \rightarrow S_{t+2}$. As such, there is a directed path from $A_{i,t}$ to $S_{t+1}$ in $g$, as well as a directed path from $S_{t+1}$ to $Y$ in $g$.
Furthermore, the directed path $A_{i,t} \rightarrow S_{t+1} \rightarrow A_{j,t+1} \rightarrow S_{t+2}$, where $j \in \mathbf{N}$, does not belong to $g$.
We can conclude then $A_{i,t}$, $Y$ and $g$ satisfy the recanting witness criterion, and hence $P(y_{a^*_{i,t}})_{M_g}$ is non-identifiable. 
Subsequently, $\mathrm{ASE}^{\mathbf{N}}_{a_{i,t}, a^*_{i,t}}(y)_M$ is also non-identifiable.

Since we cannot identify the agent-specific effect on $S_{t+2}$, it follows that we cannot also identify it on any state variable $S_{t'}$ that causally depends on $S_{t+2}$, i.e., for any $t' > t+2$.
    
\end{proof}

Next, we prove Theorem \ref{thrm.id_ase_cf}.

\begin{proof}
    This result follows directly from Lemma \ref{lemma.cf_fpse_ase} and Lemma \ref{lemma.id_cf_fpse}.
    
\end{proof}



\section{Proof of Lemma \ref{lemma.nm_mmdp}}\label{proof.nm_mmdp}

\begin{proof}
We define the domain of each variable $V^i \in \mathbf{V}$ in $M$ as $\mbox{dom}\{V^i\} = \mbox{dom}\{X^i\}$.
Furthermore, we slightly abuse our notation to denote the set of variables in $\mathbf{X}$ that correspond to the variables $V^j$ in $Pa^i(G)$, 
$\mathbf{X}(Pa^i(G)) = \{X^j \in \mathbf{X}| j : V^j \in Pa^i(G)\}$.

Observe that function $f^i(pa^i(G), U^i)$ is equivalent to the \textit{conditional quantile} function $Q_{X^i | \mathbf{X}(Pa^i(G)) = pa^i(G)}(U^i)$, i.e., the \textit{quantile} function of random variable $X^i$ when variables in $\mathbf{X}(Pa^i(G))$ are fixed to their corresponding values in $pa^i(G)$. We consider the following equations
\begin{align}\label{eq.quant}
P(V^i \leq_i v^i | Pa^i(G) = pa^i(G)) &= P(f^i(Pa^i(G), U^i) \leq_i v^i | Pa^i(G) = pa^i(G)) \nonumber \\
&= P(f^i(pa^i(G), U^i) \leq_i v^i) \nonumber \\
&= P(Q_{X^i | \mathbf{X}(Pa^i(G)) = pa^i(G)}(U^i) \leq_i v^i) \nonumber \\
&= P(U^i \leq P(X^i \leq_i v^i | \mathbf{X}(Pa^i(G)) = pa^i(G))) \nonumber \\
&= P(X^i \leq_i v^i | \mathbf{X}(Pa^i(G)) = pa^i(G)).
\end{align}

The fifth step simply follows from the fact that $U^i \sim \mbox{Uniform}[0,1]$. The fourth step holds by the \textit{reproduction} property of quantile functions \cite{kampke2015income} [Theorem. 2.1], which we restate and prove for our setting next.

\begin{lemma}\label{lemma.quant_prop}
$Q_{X^i | \mathbf{X}(Pa^i(G)) = pa^i(G)}(u^i) \leq_i v^i$ iff $P(X^i \leq_i v^i | \mathbf{X}(Pa^i(G)) = pa^i(G)) \geq u^i$.
\end{lemma}

\begin{proof}
~
\begin{itemize}
    \item $P(X^i \leq_i v^i | \mathbf{X}(Pa^i(G)) = pa^i(G)) \geq u^i$ implies that 
    $$
    Q_{X^i | \mathbf{X}(Pa^i(G)) = pa^i(G)}(u^i) = \inf_{x^i \in \mbox{dom}\{X^i\}} \{P(X^i \leq_i x^i | \mathbf{X}(Pa^i(G)) = pa^i(G)) \geq u^i\} \leq_i v^i,
    $$
    because $v^i$ belongs to the set over which the infimum is formed.
    \item Let 
    $
    x^i_{inf} = \inf_{x^i \in \mbox{dom}\{X^i\}} \{P(X^i \leq_i x^i | \mathbf{X}(Pa^i(G)) = pa^i(G)) \geq u^i\}.
    $
    Then, $Q_{X^i | \mathbf{X}(Pa^i(G)) = pa^i(G)}(u^i) \leq_i v^i$ implies that 
    $$
    v^i \geq_i x^i_{inf}.
    $$
    Since the set $S = \{x^i \in \mbox{dom}\{X^i\} : P(X^i \leq_i x^i | \mathbf{X}(Pa^i(G)) = pa^i(G)) \geq u^i\}$ is finite, discrete and non-empty, the value $x^i_{inf}$ has to belong to $S$. This means that $P(X^i \leq_i x^i_{inf} | \mathbf{X}(Pa^i(G)) = pa^i(G)) \geq u^i$. By monotonicity of the cumulative distribution function it has to also hold that $P(X^i \leq_i v^i | \mathbf{X}(Pa^i(G)) = pa^i(G)) \geq u^i$.
\end{itemize}
\end{proof}

From Eq. \eqref{eq.quant}, it follows that the MMDP-SCM $M$ induces the transition probabilities and the initial state distribution of the MMDP, as well as the agents' joint policy $\pi$.
Therefore, joint distributions $P(\mathbf{X})$ and $P(\mathbf{V})$ are indeed equal.

Next, consider two noise values $u^i_1, u^i_2 \sim P(U^i)$ such that $u^i_1 < u^i_2$. By monotonicity of the cumulative distribution function, for any $pa^i(G) \in \mbox{dom}\{Pa^i(G)\}$, it holds that
\begin{align*}
    & \inf_{v^i \in \mbox{dom}\{V^i\}} \{P(X^i \leq_i v^i | pa^i(G)) \geq u^i_1\} \leq_i \inf_{v^i \in \mbox{dom}\{V^i\}} \{P(X^i \leq_i v^i | pa^i(G)) \geq u^i_2\} \\
    & \Rightarrow f^i(pa^i(G), u^i_1) \leq_i f^i(pa^i(G), u^i_2).
\end{align*}
Hence, variable $V^i \in \mathbf{V}$ is indeed noise-monotonic in $M$ w.r.t. $\leq_i$.

\end{proof}



\section{Proof of Proposition \ref{prop.nm_m}}\label{proof.nm_m}

\begin{proof}
We prove the first part of this lemma by contradiction. We first assume that $Y$ is noise-monotonic in $M$ w.r.t.~the numerical total ordering $\leq$. Let there now be two values $x_1, x_2 \in \{0, 1\}$ such that $\mathds{E}[Y | x_1] \leq \mathds{E}[Y | x_2]$ and $P(Y_{x_1} = 1 \land Y_{x_2} = 0) > 0$ (i.e., the monotonicity assumption is violated). This means that there has to be at least one noise value $u^Y_k$ such that $f^Y(x_1, u^Y_k) = 1$ and $f^Y(x_2, u^Y_k) = 0$. Because $Y$ is assumed to be noise-monotonic, for any noise value $u^Y \geq u^Y_k$ it must hold that $f^Y(x_1, u^Y) = 1$, and for any noise value $u^Y < u^Y_k$ it must hold that $f^Y(x_2, u^Y) = 0$. Consequently, the following inequality holds
\begin{align*}
\mathds{E}[Y | x_1] &= \int_{u^Y \sim P(U^Y)} P(u^Y) \cdot f^Y(x_1, u^Y) \,du^Y \\
&= \int_{u^Y \sim P(U^Y) | u^Y < u^Y_k} P(u^Y) \cdot f^Y(x_1, u^Y) \,du^Y + 1 + \int_{u^Y \sim P(U^Y) | u^Y > u^Y_k} P(u^Y) \cdot 1 \,du^Y \\
&> \int_{u^Y \sim P(U^Y) | u^Y < u^Y_k} P(u^Y) \cdot 0 \,du^Y + 0 + \int_{u^Y \sim P(U^Y) | u^Y > u^Y_k} P(u^Y) \cdot f^Y(x_2, u^Y) \,du^Y \\
&= \int_{u^Y \sim P(U^Y)} P(u^Y) \cdot f^Y(x_2, u^Y) \,du^Y \\
&= \mathds{E}[Y | x_2].
\end{align*}
Therefore, we have reached a contradiction.

We prove the second part of this lemma by counterexample.
Assume that noise variable $U^Y$ can only take the three values $\{1, 2, 3\}$, and that function $f^Y$ is defined as follows
\begin{align*}
    & f^Y(X = 0, U^Y = 1) = 1, 
    f^Y(X = 0, U^Y = 2) = 0,
    f^Y(X = 0, U^Y = 3) = 0, \\
    & f^Y(X = 1, U^Y = 1) = 1, 
    f^Y(X = 1, U^Y = 2) = 0,
    f^Y(X = 1, U^Y = 3) = 1.
\end{align*}
It can be easily inferred that in this SCM, variable $Y$ is monotonic relative to $X$, while it is not noise-monotonic w.r.t.~the numerical total ordering or any other total ordering $\leq_Y$ on $\{0,1\}$.

\end{proof}



\section{Proof of Theorem \ref{thrm.algo_unbias}}\label{proof.algo_unbias}

\begin{proof}
According to Theorem \ref{thrm.id_ase_cf} we have that $\mathrm{cf\text{-}ASE}^{\mathbf{N}}_{a_{i,t}}(y|\tau)_M = \mathrm{cf\text{-}ASE}^{\mathbf{N}}_{a_{i,t}}(y|\tau)_{\hat{M}}$. 
This means that the output of Algorithm \ref{alg.cf-ase} is the same when it takes as input either model $M$ or model $\hat{M}$.
Therefore, to prove Theorem \ref{thrm.algo_unbias} it would suffice to show that the output of Algorithm \ref{alg.cf-ase}, when it takes as input MMDP-SCM $M$, is an unbiased estimator of $\mathrm{cf\text{-}ASE}^{\mathbf{N}}_{a_{i,t}}(y|\tau)_M$.

We begin by formally expressing the output of Algorithm \ref{alg.cf-ase}
\begin{align}\label{eq.outcome}
    O(M, \tau, \mathbf{N}, A_{i,t}, a_{i,t}, Y, y, H) &= \frac{\sum_{\mathbf{u} \in \{\mathbf{u}_1, ..., \mathbf{u}_H\}} P(y | \mathbf{u})_{M^{do(I(\mathbf{u}))}}}{H} - P(y | \tau)_M \nonumber \\
    &= \frac{\sum_{\mathbf{u} \in \{\mathbf{u}_1, ..., \mathbf{u}_H\}} P(y | \mathbf{u})_{M^{do(I(\mathbf{u}))}}  - P(y | \mathbf{u})_M}{H},
\end{align}
where $I(\mathbf{u})$ is defined as in line $7$ of the algorithm. The second step of the equation holds because $\{\mathbf{u}_1, ..., \mathbf{u}_H\} \sim P(\mathbf{u} | \tau)$. 

By Definition \ref{def.cf-ase}, we have that
\begin{align*}
    \mathrm{cf\text{-}ASE}^{\mathbf{N}}_{a_{i,t}}(y|\tau)_M &= P(y_{\tau(A_{i,t})}|\tau;M)_{M^{do(I)}} - P(y|\tau)_M,
\end{align*}
where $I = \{A_{i',t'} := \tau(A_{i',t'})\}_{i' \notin \mathbf{N}, t' > t} \cup \{A_{i',t'} := A_{i',t',a_{i,t}}\}_{i' \in \mathbf{N}, t' > t}$.

Note that for any given noise $\mathbf{u} \sim P(\mathbf{u} | \tau)$, it holds that $A_{i',t', a_{i,t}}(\mathbf{u}) = \tau^h(A_{i',t'})$, where $\tau^h \sim P(\mathbf{V}|\mathbf{u})_{M^{do(A_{i,t} := a_{i,t})}}$. Hence, it follows that $P(y | \mathbf{u})_{M^{do(I)}} = P(y | \mathbf{u})_{M^{do(I(\mathbf{u}))}}$.
We consider next the following equations
\begin{align*}
    \mathds{E}_{\mathbf{u} \sim P(\mathbf{u} | \tau)}[P(y | \mathbf{u})_{M^{do(I(\mathbf{u}))}} - P(y | \mathbf{u})_M] &= \int
    P(y | \mathbf{u})_{M^{do(I(\mathbf{u}))}} \cdot P(\mathbf{u} | \tau) - P(y | \mathbf{u})_M \cdot P(\mathbf{u} | \tau) \,d\mathbf{u} \\
    &= \int P(y | \mathbf{u})_{M^{do(I)}} \cdot P(\mathbf{u} | \tau)  - P(y | \mathbf{u})_M \cdot P(\mathbf{u} | \tau) \,d\mathbf{u} \\
    &= \int P(y | \mathbf{u})_{M^{do(I)}} \cdot P(\mathbf{u} | \tau; M)_{M^{do(I)}}  - P(y | \mathbf{u})_M \cdot P(\mathbf{u} | \tau)_M \,d\mathbf{u} \\
    &= P(y | \tau; M)_{M^{do(I)}} - P(y | \tau)_M \\
    &= P(y_{\tau(A_{i,t})} | \tau; M)_{M^{do(I)}} - P(y | \tau)_M \\
    &= \mathrm{cf\text{-}ASE}^{\mathbf{N}}_{a_{i,t}}(y|\tau)_M,
\end{align*}
where the third step follows from the fact that the distribution of noise variables $\mathbf{U}$ remains the same between $M$ and $M^{do(I)}$. The fifth step follows from the consistency axiom.
We conclude then that $O(M, \tau, \mathbf{N}, A_{i,t}, a_{i,t}, Y, y, H)$ is indeed an unbiased estimator of $\mathrm{cf\text{-}ASE}^{\mathbf{N}}_{a_{i,t}}(y|\tau)_M$.

\end{proof}



\section{Proof of Lemma \ref{lemma.fpse_ase}}\label{proof.fpse_ase}

\begin{proof}

By Definitions \ref{def.ase} and \ref{def.fpse}, it follows that it suffices to show that $P(y_{a^*_{i,t}})_{M^{do(I)}} = P(y_{a^*_{i,t}})_{M_q}$.

Let $t_Y$ denote the time-step of variable $Y$. First, we show that for any noise $\mathbf{u} \sim P(\mathbf{u})$ the value of $S_{t_Y}(\mathbf{u})$ (or equivalently $Y(\mathbf{u})$) is the same in $M^{do(A_{i, t} := a^*_{i,t})}_q$ as in $M^{do(I \cup A_{i, t} := a^*_{i,t})}$. To do so, we use induction in the number of time-steps $t'$.

~\\
\textbf{Base case:} 
By construction of $M_q$, the value of $S_{t+1}(\mathbf{u})$ is the same in $M^{do(A_{i, t} := a^*_{i,t})}_q$ as in $M^{do(A_{i, t} := a^*_{i,t})}$. Furthermore, note that the set of interventions $I$ does not include interventions to variable $S_{t+1}$ or to any of its ancestors. As a result, the value of $S_{t+1}(\mathbf{u})$ is the same in $M^{do(I \cup A_{i, t} := a^*_{i,t})}$ as in $M^{do(A_{i, t} := a^*_{i,t})}$. We conclude then that the value of $S_{t+1}(\mathbf{u})$ is the same in $M^{do(A_{i, t} := a^*_{i,t})}_q$ as in $M^{do(I \cup A_{i, t} := a^*_{i,t})}$.

~\\
\textbf{Induction hypothesis:} 
We make the hypothesis that for some $t' \in [t+1, t_Y)$ the value of $S_{t'}(\mathbf{u})$ is the same in $M^{do(A_{i, t} := a^*_{i,t})}_q$ as in $M^{do(I \cup A_{i, t} := a^*_{i,t})}$.

~\\
\textbf{Induction step:} 
Under the induction hypothesis, we want to show that the value of $S_{t'+1}(\mathbf{u})$ is the same in $M^{do(A_{i, t} := a^*_{i,t})}_q$ as in $M^{do(I \cup A_{i, t} := a^*_{i,t})}$.
For convenience, we slightly abuse our notation and denote as $s^q_{t'+1}$ the value of $S_{t'+1}(\mathbf{u})$ in $M^{do(A_{i, t} := a^*_{i,t})}_q$, and as $s^I_{t'+1}$ the value of $S_{t'+1}(\mathbf{u})$ in $M^{do(I \cup A_{i, t} := a^*_{i,t})}$. 
From the hypothesis, it holds that the value $s_{t'}$ of $S_{t'}(\mathbf{u})$ is the same in $M^{do(A_{i, t} := a^*_{i,t})}_q$ as in $M^{do(I \cup A_{i, t} := a^*_{i,t})}$.

First, we consider the value of $S_{t'+1}(\mathbf{u})$ in $M^{do(I \cup A_{i, t} := a^*_{i,t})}$,
\begin{align}\label{eq.proof_I}
   s^I_{t'+1} = f^{S_{t'+1}}(s_{t'}, 
   \{A_{i',t', a_{i,t}}(\mathbf{u})\}_{i' \in \mathbf{N}},
   \{A_{i',t', a^*_{i,t}}(\mathbf{u})\}_{i' \notin \mathbf{N}},
   u^{S_{t'+1}}). 
\end{align}
 
Next, we consider the value of $S_{t'+1}(\mathbf{u})$ in $M^{do(A_{i, t} := a^*_{i,t})}_q$,
$$
s^q_{t'+1} = f^{S_{t'+1}}_q(pa^{S_{t'+1}}(q), u^{S_{t'+1}}),
$$ 
where $pa^{S_{t'+1}}(q)$ here denotes the value of $Pa^{S_{t'+1}}(q)$ in $M^{do(A_{i, t} := a^*_{i,t})}_q$ given $\mathbf{u}$. 
Note that $Pa^{S_{t'+1}}(q) = S_{t'}$, and since $S_{t'}(\mathbf{u}) = s_{t'}$ in $M^{do(A_{i, t} := a^*_{i,t})}_q$,
we can rewrite the equation of $s^q_{t'+1}$ as
$$
s^q_{t'+1} = f^{S_{t'+1}}_q(s_{t'}, u^{S_{t'+1}}).
$$ 
Based on the definition of function $f_q^{S_{t'+1}}$ we have that
\begin{align}\label{eq.proof_q1}
    s^q_{t'+1} = f^{S_{t'+1}}(s_{t'}, pa^{S_{t'+1}}(g)^e, pa^{S_{t'+1}}(g^*)^*, u^{S_{t'+1}}),
\end{align}
where $pa^{S_{t'+1}}(g^*)^*$ and $pa^{S_{t'+1}}(g)^e$ denote the values of $Pa^{S_{t'+1}}(g^*)_{a^*_{i,t}}$ and $Pa^{S_{t'+1}}(g)_{a_{i,t}}$ in $M$ given $\mathbf{u}$, respectively.
From the definition of subgraph $g^*$ it follows that $Pa^{S_{t'+1}}(g^*) = \{A_{i',t'}\}_{i' \notin \mathbf{N}}$. Eq. \eqref{eq.proof_q1} then can be rewritten as follows
\begin{align}\label{eq.proof_q2}
    s^q_{t'+1} =  f^{S_{t'+1}}(s_{t'}, pa^{S_{t'+1}}(g)^e, \{A_{i',t', a^*_{i,t}}(\mathbf{u})\}_{i' \notin \mathbf{N}}, u^{S_{t'+1}}).
\end{align}
Finally, from the definition of subgraph $g$ it follows that $Pa^{S_{t'+1}}(g) = \{A_{i',t'}\}_{i' \in \mathbf{N}}$. Eq. \eqref{eq.proof_q2} then can be rewritten as follows
\begin{align}\label{eq.proof_q}
    s^q_{t'+1} =  f^{S_{t'+1}}(s_{t'}, \{A_{i',t', a_{i,t}}(\mathbf{u})\}_{i' \in \mathbf{N}}, \{A_{i',t', a^*_{i,t}}(\mathbf{u})\}_{i' \notin \mathbf{N}}, u^{S_{t'+1}}).
\end{align}
From Equations \eqref{eq.proof_I} and \eqref{eq.proof_q}, it follows that $s^q_{t'+1} = s^I_{t'+1}$, and hence the induction step is concluded.

~\\
Based on the induction argument above, for every $\mathbf{u} \sim P(\mathbf{u})$ it holds that the value of $S_{t_Y}(\mathbf{u})$ or equivalently the value of $Y(\mathbf{u})$ is the same in $M^{do(A_{i, t} := a^*_{i,t})}_q$ as in $M^{do(I \cup A_{i, t} := a^*_{i,t})}$, and hence
$$
P(y_{a^*_{i,t}}|\mathbf{u})_{M^{do(I)}} = P(y_{a^*_{i,t}}|\mathbf{u})_{M_q}.
$$
Finally, the distribution of noise variables $\mathbf{U}$ is the same in $M^{do(I)}$ as in $M_q$, and thus we can derive the following equations
\begin{align*}
    \int P(y_{a^*_{i,t}}|\mathbf{u})_{M^{do(I)}} \cdot P(\mathbf{u})_{M^{do(I)}} \,d\mathbf{u} &= \int P(y_{a^*_{i,t}}|\mathbf{u})_{M_q} \cdot P(\mathbf{u})_{M_q}\,d\mathbf{u} \Rightarrow\\
    P(y_{a^*_{i,t}})_{M^{do(I)}} &= P(y_{a^*_{i,t}})_{M_q}.
\end{align*}
    
\end{proof}



\section{Proof of Proposition \ref{prop.fpse_pse}}\label{proof.fpse_pse}

\begin{proof}
By Definition \ref{def.pse}, we have that
\begin{align}\label{eq.pse-pse}
    PSE^g_{x, x^*}(y)_M &= P(y_x)_{M_g} - P(y_{x^*})_{M_g} \nonumber\\
    &= P(y_x)_{M_g} - P(y_{x^*})_M.
\end{align}
The second step holds because the quantity $P(y_{a_{i,t}})$ in the modified model $M_g$ corresponds to $P(y_{a_{i,t}})$ in the original model $M$. By Definition \ref{def.fpse}, we have that
\begin{align}\label{eq.fpse-pse}
    FPSE^{\overline{g}, g^\emptyset}_{x^*, x}(y)_M &= P(y_{x})_{M_q} - P(y_x)_M \nonumber\\
    &= P(y_{x})_{M_g} - P(y_x)_M.
\end{align}
The second step follows from $\mathbf{E}(g^\emptyset) = \emptyset$.
From Section \ref{sec.framework}, we have that
\begin{align}\label{eq.tce-pse}
    TCE_{x, x^*}(y)_M = P(y_x)_M - P(y_{x^*})_M.
\end{align}
By combining Equations \eqref{eq.pse-pse}, \eqref{eq.fpse-pse} and \eqref{eq.tce-pse} we get
\begin{align*}
    PSE^g_{x, x^*}(y)_M = FPSE^{\overline{g}, g^\emptyset}_{x^*, x}(y)_M + TCE_{x, x^*}(y)_M.
\end{align*}

\end{proof}



\end{document}